\documentclass[journal,onecolumn,draftcls]{IEEEtran}

\usepackage[utf8]{inputenc}
\usepackage{amsmath,amsthm, amssymb, amsfonts}
\usepackage{xcite}
\usepackage{todonotes}
\usepackage{kantlipsum}
\usepackage{tikz}
\usetikzlibrary{arrows}
\usepackage{wrapfig}
\usepackage{paralist}
\usepackage{verbatim}
\usepackage{fullpage}
\usepackage{relsize}
\usepackage{hyperref}
\usepackage{graphicx}
\usepackage{epstopdf}
\usepackage{stmaryrd}
\usepackage{wrapfig}
\usepackage[]{algorithm2e}
\usepackage{nicefrac}

\allowdisplaybreaks[1]
\usepackage{sistyle} 
\newcommand\bSI[1]{{\small[\SI{}{#1}]}}

\usepackage{pgfplots}
\usepgfplotslibrary{groupplots}
\usetikzlibrary{arrows.meta}
\usepackage{caption}

\makeatletter
\newlength\unitwdth
\newlength\numwdth
\newlength\tdima
\newcommand\SIdescr[2]{%
    \setlength\tdima{\linewidth}%
    \addtolength\tdima{\@totalleftmargin}%
    \addtolength\tdima{-\dimen\@curtab}%
    \addtolength\tdima{-\unitwdth}%
    \addtolength\tdima{-\numwdth}%
    \parbox[t]{\tdima}{%
        #1
        \leaders\hbox{$\m@th\mkern \@dotsep mu\hbox{\tiny.}\mkern \@dotsep mu$}%
        \hfill
        \ifhmode\strut\fi
        \makebox[0pt][l]{%
            \makebox[\unitwdth][l]{}%
            \makebox[\numwdth][r]{#2}}}}
\makeatother

\newcommand{\Wcal}{\mathcal{W}}

\newcommand{\Z}{\mathbb{Z}}
\newcommand{\N}{\mathbb{N}}

\newcommand{\R}{\mathbb{R}}

\newcommand{\C}{\mathbb{C}}

\newcommand{\cC}{\mathcal{C}}

\newcommand{\supp}{\mathrm{supp}}

\def\diag{{\text{\rm diag}}}

\newtheorem{theorem}{Theorem}[section]
\newtheorem{remark}[theorem]{Remark}
\newtheorem{definition}[theorem]{Definition}
\newtheorem{proposition}[theorem]{Proposition}
\newtheorem{lemma}[theorem]{Lemma}
\newtheorem{corollary}[theorem]{Corollary}

\DeclareMathOperator{\spann}{span \,}

\usepackage{footnote}
\makesavenoteenv{tabular}
\makesavenoteenv{table}

\usepackage{cleveref}
\renewcommand{\epsilon}{\varepsilon}
\newcommand{\eps}{\epsilon}
\newcommand{\Ss}{\mathcal{S}}
\newcommand{\F}{\mathcal{F}}
\newcommand{\G}{\mathcal{G}}
\newcommand{\RE}{\mathrm{Re}}
\newcommand{\IM}{\mathrm{Im}}
\newcommand{\M}{\mathcal{M}}
\renewcommand{\L}{\mathcal{L}}
\newcommand{\cN}{\mathcal{N}}
\newcommand{\cL}{\mathcal{L}}
\newcommand{\cM}{\mathcal{M}}
\newcommand{\cW}{\mathcal{W}}
\newcommand{\cB}{\mathcal{B}}

\def\tablename{Table}
\renewcommand{\thetable}{\arabic{table}}

\title{Deep Neural Network Approximation Theory}
\pgfplotsset{compat=1.16}

\begin{document}

\author{Dennis Elbr\"achter, Dmytro Perekrestenko, Philipp Grohs, and Helmut B\"olcskei%
\thanks{D. Elbr\"achter is with the Department of Mathematics, University of Vienna, Austria (e-mail: dennis.elbraechter@univie.ac.at).}%
\thanks{D. Perekrestenko and H. B\"olcskei are with the Chair for Mathematical Information Science, ETH Zurich, Switzerland (e-mail: pdmytro@mins.ee.ethz.ch,\, hboelcskei@ethz.ch).}%
\thanks{P. Grohs is with the Department of Mathematics and the Research Platform DataScience@UniVienna, University of Vienna, Austria (e-mail: philipp.grohs@univie.ac.at).}
\thanks{D. Elbr\"achter was supported through the FWF projects P 30148 and I 3403 as well as the WWTF project ICT19-041.}}

\maketitle
\thispagestyle{empty}

\begin{abstract}

This paper develops fundamental limits of deep neural network learning by characterizing what is possible if no constraints are imposed on the learning algorithm and on the amount of training data. Concretely, we consider
Kolmogorov-optimal approximation through deep neural networks with
the guiding theme being a relation between the complexity of the function (class) to be approximated and the complexity of the approximating network in terms of connectivity and memory requirements for storing the network topology and the associated quantized weights. 
The theory we develop establishes that deep networks
are Kolmogorov-optimal approximants for markedly different function classes, such as unit balls in Besov spaces and modulation spaces.
In addition, deep networks
provide exponential approximation accuracy---i.e., the approximation error decays exponentially in the number of nonzero weights in the network---of the
multiplication operation, polynomials, sinusoidal functions, and certain smooth functions. Moreover, this holds true even for one-dimensional oscillatory textures and the Weierstrass function---a fractal function, neither of which has previously known methods achieving exponential approximation accuracy.
We also show that in the approximation of sufficiently smooth functions finite-width deep networks require strictly smaller connectivity than finite-depth wide networks.

\end{abstract}

\section{Introduction}

Triggered by the availability of vast amounts of training data and drastic improvements in computing power, deep neural networks have become state-of-the-art technology for a wide range of practical machine learning tasks such as
image classification \cite{Krizhevsky2012Imagenet}, handwritten digit recognition \cite{lecun:1995MNIST}, speech recognition \cite{Hint2012acoustic}, or game intelligence \cite{David2016Go}. For an in-depth overview, we refer to the survey paper \cite{LeCun2015DeepLearning} and the recent book \cite{Goodfellow-et-al-2016}.

A neural network effectively implements a mapping 
approximating a function that is learned based on a given set of input-output value pairs,
typically through the backpropagation algorithm \cite{Rumelhart1988Backpropagation}. 
Characterizing the fundamental limits of approximation through neural networks shows what is possible if no constraints are imposed on the learning algorithm and on the amount of training data~\cite{Anthony1998}.

The theory of function approximation through neural networks has a long history dating back to the work by McCulloch and Pitts \cite{MP43} and the seminal paper by Kolmogorov \cite{Kolmogorov1957}, who showed, when interpreted in neural network parlance, that any continuous function of $n$ variables can be represented exactly through a $2$-layer neural network of width $2n+1$. However, the nonlinearities in Kolmogorov's neural network are highly nonsmooth and the outer nonlinearities, i.e., those in the output layer, depend on the function to be represented. In modern neural network theory, one is usually interested in networks with nonlinearities that are independent of the function to be realized and exhibit, in addition, certain smoothness properties. Significant progress in understanding the approximation capabilities of such networks has been made in \cite{Cybenko1989, Hornik1991251}, where it was shown that
single-hidden-layer neural networks can approximate continuous functions on bounded domains arbitrarily well, provided that the activation function satisfies
certain (mild) conditions and the number of nodes is allowed to grow arbitrarily large. In practice one is, however, often interested in approximating functions from a given function class 
$\mathcal{C}$ determined by the application at hand. It is therefore natural to ask how the complexity of a neural network approximating every function in $\mathcal{C}$ to within a prescribed accuracy depends on the complexity of $\mathcal{C}$ (and on the desired approximation accuracy). The recently developed Kolmogorov-Donoho rate-distortion theory for neural networks \cite{boelcskei:2017DNN} formalizes this question by relating the complexity of 
$\mathcal{C}$---in terms of the number of bits needed to describe any element in $\mathcal{C}$ to within prescribed accuracy---to network complexity in terms of connectivity and memory requirements for storing the network topology and the associated quantized weights.
The theory is based on a framework for quantifying the fundamental limits of nonlinear approximation through dictionaries as introduced by Donoho~\cite{DONOHO1993100,Donoho1996}.

The purpose of this paper is to provide a comprehensive, principled, and self-contained introduction to Kolmogorov-Donoho rate-distortion optimal approximation through deep neural networks.
The idea is to equip the reader with a 
working knowledge of the mathematical tools underlying the theory at a level that is sufficiently deep to enable further research in the field. 
Part of this paper is based on \cite{boelcskei:2017DNN}, but extends the theory therein to the rectified linear unit (ReLU) activation function and to networks with depth scaling in the approximation error. 

The theory we develop educes remarkable universality properties of finite-width deep networks. Specifically, deep networks
are Kolmogorov-Donoho optimal approximants for vastly different function classes such as unit balls in Besov spaces \cite{mallat_wavelet_tour} and modulation spaces \cite{grochenig2013foundations}. 
This universality is afforded by a concurrent invariance property of deep networks to time-shifts, scalings, and frequency-shifts. In addition, deep networks provide exponential approximation accuracy---i.e., the approximation error decays exponentially in the number of parameters employed in the approximant, namely the number of nonzero weights in the network---for vastly different functions such as the squaring operation, multiplication, 
polynomials, sinusoidal functions, general smooth functions, and even one-dimensional oscillatory textures \cite{demanet2007wave} and the Weierstrass function---a fractal function, neither of which has known methods achieving exponential approximation accuracy. 

While we consider networks based on the ReLU\footnote{ReLU stands for the Rectified Linear Unit nonlinearity defined as $x\mapsto\max\{0,x\}$.} activation function throughout, certain parts of our theory carry over to strongly sigmoidal activation functions of order $k\ge2$ as defined in \cite{boelcskei:2017DNN}. For the sake of conciseness, we refrain from providing these extensions.

\vspace{0.3cm}
\textit{Outline of the paper.} In Section~\ref{sec:setup}, we introduce notation, formally define neural networks, and record basic elements needed in the neural network constructions throughout the paper. Section~\ref{func-mult} presents an algebra of function approximation by neural networks.
In Section~\ref{sec:approximation_theory}, we develop the Kolmogorov-Donoho rate-distortion framework that will allow us to characterize the fundamental limits of deep neural network learning of function classes. This theory is based on the concept of metric entropy, which is introduced and reviewed starting from first principles. Section~\ref{sec:approx-rep-systems} then puts the Kolmogorov-Donoho framework to work in the context of nonlinear function approximation with dictionaries. This discussion serves as a basis for the development of the concept of best $M$-weight approximation in neural networks presented in Section~\ref{subsec:NNapproxintro}. We proceed, in Section~\ref{sec:bestapprox}, with the development of a method---termed the transference principle---for transferring results on function approximation through dictionaries to results on approximation by neural networks. 
The purpose of Section~\ref{sec:optimalapprox} is to demonstrate that function classes that are optimally approximated by affine dictionaries (e.g., wavelets), are optimally approximated by neural networks as well. 
In Section~\ref{sec:gabor}, we show that this optimality transfer extends to function classes that are optimally approximated by Weyl-Heisenberg dictionaries.
Section~\ref{sec:weierstrass} demonstrates that neural networks can improve the best-known approximation rates for two example functions, namely oscillatory textures and the Weierstrass function, from polynomial to exponential. 
The final Section~\ref{sec:depth-width} makes a formal case for depth in neural network approximation by establishing a provable benefit of deep networks over shallow networks in the approximation of sufficiently smooth functions. The Appendices collect ancillary technical results.

\vspace{0.3cm}
\textit{Notation.} For a function $f(x)\colon\R^d\to\R$ and a set $\Omega\subseteq\R^d$, we define $\| f \|_{L^{\infty}(\Omega)}:= \sup \{|f(x)|: x \in \Omega \}$. $L^p(\R^d)$ and $L^p(\R^d,\C)$ denote the space of real-valued, respectively complex-valued, $L^p$-functions.
When dealing with the approximation error for simple functions such as, e.g., $(x,y)\mapsto xy$, we will for brevity of exposition and with slight abuse of notation, make the arguments inside the norm explicit according to $\|f(x,y)-xy\|_{L^p(\Omega)}$.
For a vector $b\in\R^d$, we let $\|b\|_{\infty}:=\max_{i=1,\dots,d}|b_i|$, similarly we write $\|A\|_{\infty}:=\max_{i,j}|A_{i,j}|$ for the matrix $A\in\R^{m\times n}$. We denote the identity matrix of size $n \times n$ by $\mathbb{I}_n$. $\log$ stands for the logarithm to base $2$. For a set $X\in\R^d$, we write $|X|$ for its Lebesgue measure. Constants like $C$ are understood to be allowed to take on different values in different uses.

\newpage
\section{Setup and basic ReLU calculus}\label{sec:setup}

This section defines neural networks, introduces the basic setup as well as further notation, and lists basic elements needed in the neural network constructions considered throughout, namely compositions and linear combinations of neural networks.
There is a plethora of neural network architectures and activation functions in the literature. Here, we restrict ourselves to the ReLU activation function and
consider the following general network architecture.

\begin{definition}\label{def:NN}
Let $L\in \N$ and $N_0, N_1, \ldots, N_{L}\in \N$. A \textnormal{ReLU neural network} $\Phi$ is a map $\Phi: \R^{N_0} \to \R^{N_L}$ \mbox{given by}
\begin{equation}\label{eq:NNdef}
\Phi = \begin{cases}
\begin{array}{lc} W_1, & L=1\\
W_2\circ\rho\circ W_1, & L=2\\
W_L\circ\rho \circ W_{L-1} \circ \rho \circ \dots \circ \rho \circ W_{1}, & L\ge3 
\end{array} \end{cases},
\end{equation}
where, for $\ell\in\{1,2,\dots,L\}$, $W_{\ell}\colon \R^{N_{\ell-1}} \to \R^{N_\ell},W_\ell(x):=A_\ell x + b_\ell$
are the associated affine transformations with matrices $A_{\ell}\in \mathbb{R}^{N_{\ell}\times N_{\ell-1}}$ and (bias) vectors $b_\ell\in \mathbb{R}^{N_\ell}$, and the ReLU activation function $\rho\colon\R\to\R,\ \rho(x) := \max(0,x)$ acts component-wise, i.e., $\rho(x_1,\dots,x_N):=(\rho(x_1),\dots,\rho(x_N))$.
We denote by $\cN_{d,d'}$ the set of all ReLU networks with input dimension $N_0=d$ and output dimension $N_L=d'$. Moreover, we define the following quantities related to the notion of size of the ReLU network $\Phi$\emph{:}
\begin{itemize}
    \item the \emph{connectivity} $\M(\Phi)$ is the total number of nonzero entries in the matrices $A_\ell$, $\ell\in\{1,2,\dots,L\}$, and the vectors $b_\ell$, $\ell\in\{1,2,\dots,L\}$,
    \item \emph{depth} $\L(\Phi):=L$, 
    \item \emph{width} $\mathcal{W}(\Phi):=\max_{\ell=0,\dots,L} N_\ell$,
    \item \emph{weight magnitude} $\mathcal{B}(\Phi):=\max_{\ell=1,\dots,L} \max\{\|A_\ell\|_{\infty},\|b_\ell\|_{\infty}\}$.
\end{itemize}
\end{definition}

\begin{remark}\label{remark:non-degeneracy}
Note that for a given function $f: \R^{N_0} \to \R^{N_L}$, which can be expressed according to \eqref{eq:NNdef}, the underlying affine transformations $W_\ell$ are highly nonunique in general~\cite{Fefferman94,NIPS2019_Degenerate}. The question of uniqueness in this context is of independent interest and was addressed recently in \cite{Vlacic-Boel19,Vlacic-Boel20}.
Whenever we talk about a given ReLU network $\Phi$, we will either explicitly or implicitly associate $\Phi$ with a given set of affine {transformations $W_\ell$}.

$N_0$ is the \emph{dimension of the input layer} indexed as the $0$-th layer,
$N_1, \ldots, N_{L-1}$ are the \emph{dimensions of the $L-1$ hidden layers}, and $N_L$ is the \emph{dimension of the output layer}. Our definition of depth $\L(\Phi)$ counts the number of affine transformations involved in the representation (\ref{eq:NNdef}). Single-hidden-layer neural networks hence have depth $2$ in this terminology. Finally, we consider standard affine transformations as neural networks of depth $1$ for technical purposes.

The matrix entry $(A_\ell)_{i,j}$ represents the \emph{weight associated with the edge between
the $j$-th node in the $(\ell-1)$-th layer and the $i$-th node in the $\ell$-th layer}, $(b_\ell)_i$ is \emph{the weight associated with the $i$-th node in the $\ell$-th layer}.
These assignments are schematized in Figure \ref{fig:Weights}.  The real numbers $(A_\ell)_{i,j}$ and $(b_\ell)_i$ are referred to as the network's edge weights and node weights, respectively.

Throughout the paper, we assume that every node in the input layer and in layers $1,\dots,L-1$ has at least one outgoing edge and every node in the output layer $L$ has at least one incoming edge. These nondegeneracy assumptions are basic as nodes that do not satisfy them can be removed without changing the functional relationship realized by the network. 

Finally, we note that the connectivity satisfies 
$$\M(\Phi) \leq \L(\Phi) \mathcal{W}(\Phi)(\mathcal{W}(\Phi)+1).$$
\end{remark}

The term ``network'' stems from the interpretation of the mapping $\Phi$ as a weighted acyclic directed graph with nodes arranged in hierarchical layers and
edges only between adjacent layers.

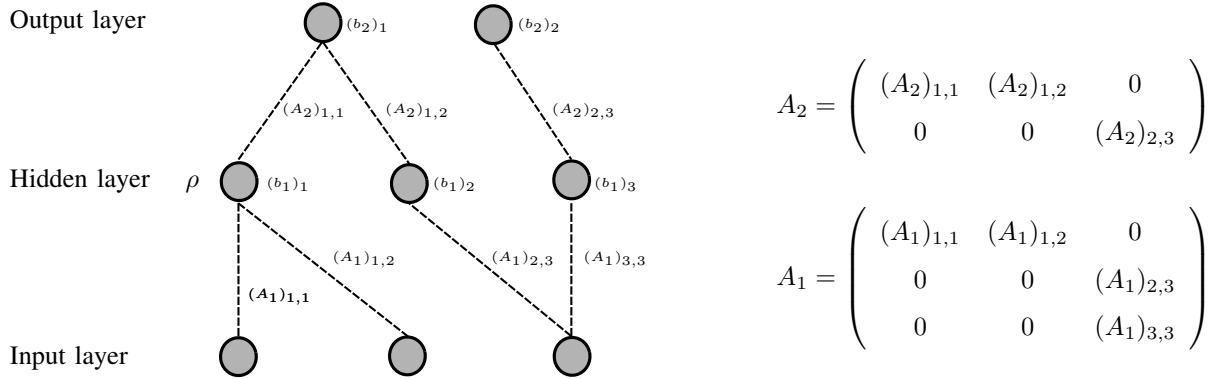
\begin{figure}[htb]
\flushleft
\hspace{3cm}
  \includegraphics[width = 0.3\textwidth]{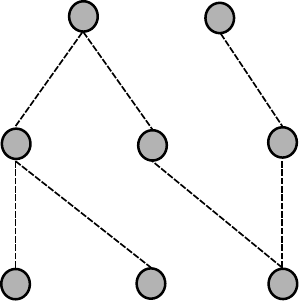}
 \tiny
 \put(-93,132){$(b_2)_1$}
 \put(-28,132){$(b_2)_2$}
 \put(-123,72){$(b_1)_1$}
 \put(-60,72){$(b_1)_2$}
 \put(1,72){$(b_1)_3$}
 \put(-117,100){$(A_2)_{1,1}$}
 \put(-78,100){$(A_2)_{1,2}$}
 \put(-15,100){$(A_2)_{2,3}$}
 \put(-130,30){$(A_1)_{1,1}$}
 \put(-3,44){$(A_1)_{3,3}$}
 \put(-38,44){$(A_1)_{2,3}$}
 \put(-98,44){$(A_1)_{1,2}$}
 \put(-130,30){$(A_1)_{1,1}$}
 \normalsize
  \put(70, 100){$A_2 = \left(\begin{array}{c c c}
                     (A_2)_{1,1} & (A_2)_{1,2} & 0\\
                     0 & 0 & (A_2)_{2,3}\\
                    \end{array}\right)
$}
 \put(70, 35){$A_1 = \left(\begin{array}{c c c}
                     (A_1)_{1,1} & (A_1)_{1,2} & 0\\
                     0 & 0 & (A_1)_{2,3}\\
                     0 & 0 & (A_1)_{3,3}\\
                    \end{array}\right)
$}
\put(-220, 132){Output layer}
\put(-220, 72){Hidden layer \quad $\rho$}
\put(-220, 5){Input layer}

 \caption{Assignment of the weights $(A_\ell)_{i,j}$ and $(b_\ell)_{i}$ of a two-layer network to the edges and nodes, respectively. }
 \label{fig:Weights}
\end{figure}

We mostly consider the case $\Phi: \R^d \to \R$, i.e., $N_L = 1$, but emphasize that our results readily generalize to $N_L >1$.

The neural network constructions provided in the paper frequently make use of basic elements introduced next, namely compositions and linear combinations of networks \cite{PetersenVoigtlaender}.

\begin{lemma}\label{network_conc}
Let $d_1,d_2,d_3\in\N$, $\Phi_1\in\cN_{d_1,d_2}$, and $\Phi_2\in\cN_{d_2,d_3}$. Then, there exists a network $\Psi\in\cN_{d_1,d_3}$ with
$\cL(\Psi)=\cL(\Phi_1)+\cL(\Phi_2)$,
$\cM(\Psi)\leq 2\cM(\Phi_1)+2\cM(\Phi_2)$,
$\cW(\Psi)\leq \max\{2d_2,\Wcal(\Phi_1), \Wcal(\Phi_2)\}$,
$\cB(\Psi)=\max\{\mathcal{B}(\Phi_1),\mathcal{B}(\Phi_2)\}$, and satisfying
$$\Psi(x)=(\Phi_2\circ\Phi_1)(x)=\Phi_2(\Phi_1(x)), \quad \text{for all } x\in\R^{d_1}.$$ 
\end{lemma}

\pagebreak
\begin{proof}
The proof is based on the identity $x = \rho(x) - \rho(-x)$. First, note that by Definition \ref{def:NN}, we can write
\begin{equation*}
\Phi_1 = W^1_{L_1}\circ \rho \circ W^1_{L_1-1}\circ  \dots  \circ\rho \circ W^1_1 \quad \mbox{ and } \quad \Phi_2 = W^2_{L_2}\circ \rho \circ \dots \circ W^2_2 \circ \rho \circ W^2_1. 
\end{equation*}
Next, let $N^{1}_{L_{1}-1}$ denote the width of layer $L_{1}-1$ in $\Phi_{1}$ and let $N^{2}_{1}$ denote the width of layer $1$ in $\Phi_{2}$. We define the affine transformations $\widetilde{W}^1_{L_1}\colon\R^{N^{1}_{L_{1}-1}}\mapsto\R^{2d_2}$ and $\widetilde{W}^2_1\colon\R^{2d_2}\mapsto\R^{N^2_1}$ according to
\begin{align*}
     \widetilde{W}^1_{L_1}(x):=\begin{pmatrix}
    \mathbb{I}_{d_2} \\
    -\mathbb{I}_{d_2}
 \end{pmatrix}   W^1_{L_1}(x)
 \quad \mbox{ and }  
    \widetilde{W}^2_1(y):= W^2_1\left(\begin{pmatrix}\mathbb{I}_{d_2} & -\mathbb{I}_{d_2}\end{pmatrix}y
    \right).
\end{align*}
The proof is finalized by noting that the network
\begin{align*}
    \Psi:=W^2_{L_2}\circ \rho \circ \dots \circ W^2_2 \circ \rho \circ \widetilde{W}^2_1 \circ \rho \circ \widetilde{W}^1_{L_1}   \circ \rho \circ W^1_{L_1-1} \circ \dots \circ \rho \circ W^1_1
\end{align*}
satisfies the claimed properties.
\end{proof}

Unless explicitly stated otherwise, the composition of two neural networks will be understood in the sense of Lemma \ref{network_conc}. 

In order to formalize the concept of a linear combination of networks with possibly different depths, we need the following two technical lemmas which show how to augment network depth while retaining the network's input-output relation and how to parallelize networks.
\begin{lemma}
\label{network_extension}
Let $d_1,d_2, K\in\N$, and $\Phi\in\cN_{d_1,d_2}$ with $\cL(\Phi)<K$. 
Then, there exists a network $\Psi\in\cN_{d_1,d_2}$ with 
$\cL(\Psi)=K$,
$\cM(\Psi)\leq \cM(\Phi)+d_2\cW(\Phi)+2d_2(K-\cL(\Phi))$,
$\cW(\Psi) = \max\{2d_2,\Wcal(\Phi)\}$,
$\cB(\Psi)=\max\{1,\mathcal{B}(\Phi)\}$, and satisfying
$\Psi(x)=\Phi(x)$ for all $x \in \mathbb{R}^{d_1}$.
\end{lemma}

\begin{proof}
Let 
$\widetilde{W}_j(x):=\diag\big(\mathbb{I}_{d_2},\mathbb{I}_{d_2}\big)\,x$, for $j\in\{\cL(\Phi)+1,\dots,K-1\}$, 
$\widetilde{W}_{K}(x):=\begin{pmatrix}\mathbb{I}_{d_2} & -\mathbb{I}_{d_2}\end{pmatrix}x$, and note that with
\begin{align*}
\Phi = W_{\cL(\Phi)}\circ\rho \circ W_{\cL(\Phi)-1} \circ \rho \circ \dots \circ \rho \circ W_{1},
\end{align*}
the network
\begin{align*}
    \Psi:= \widetilde{W}_{K} \circ\rho\circ\widetilde{W}_{K-1}\circ\rho\circ\dots\circ\rho\circ\widetilde{W}_{\cL(\Phi)+1}\circ\rho\circ\begin{pmatrix}W_{\cL(\Phi)} \\ -W_{\cL(\Phi)}\end{pmatrix}\circ\rho \circ W_{\cL(\Phi)-1} \circ\rho\circ\dots\circ\rho\circ W_1
\end{align*}
satisfies the claimed properties.
\end{proof}

For the sake of simplicity of exposition, we state the following two lemmas only for networks of the same depth, the extension to the general case follows by straightforward application of Lemma~\ref{network_extension}. The first of these two lemmas formalizes the notion of neural network parallelization, 
concretely of combining neural networks implementing the functions $f$ and $g$ into a neural network realizing the mapping $x\mapsto(f(x),g(x))$.

\begin{lemma}
\label{network_parallelization}
Let $n,L\in\N$ and, for $i\in\{1,2,\dots,n\}$, let $d_i,d'_i\in\N$ and $\Phi_i\in\cN_{d_i,d'_i}$ with $\cL(\Phi_i)=L$.
Then, there exists a network $\Psi \in \cN_{\sum_{i=1}^n d_i,\sum_{i=1}^n d'_i}$ with $\cL(\Psi) = L$, $\cM(\Psi) = \sum_{i=1}^n\cM(\Phi_i)$, $\mathcal{W}(\Psi) = \sum_{i=1}^n \mathcal{W}(\Phi_i)$, $\cB(\Psi)=\max_i\cB(\Phi_i)$, and satisfying
\begin{align*}
    \Psi(x)&=(\Phi_1(x_1),\Phi_2(x_2),\dots,\Phi_n(x_n))\in\R^{\sum_{i=1}^n d'_i},
\end{align*}
for $x=(x_1,x_2,\dots,x_n)\in\R^{\sum_{i=1}^n d_i}$ with $x_i\in\R^{d_i}$, $i\in\N$.
\end{lemma}

\begin{proof}
    We write the networks $\Phi_i$ as
    \begin{align*}
        \Phi_i = W^i_L\circ\rho \circ W^i_{L-1} \circ \rho \circ \dots \circ \rho \circ W^i_{1},
    \end{align*}
    with $W^i_\ell(x)=A^i_\ell x+b^i_\ell$. Furthermore, we denote the layer dimensions of $\Phi_i$ by $N^i_0,\dots,N^i_L$ and set $N_\ell:=\sum_{i=1}^n N^i_\ell$, for $\ell\in\{0,1,\dots,L\}$.
    Next, define, for $\ell\in\{1,2,\dots,L\}$, the block-diagonal matrices $A_\ell:=\diag(A^1_\ell,A^2_\ell,\dots,A^n_\ell)$,
    the vectors $b_\ell=(b^1_\ell,b^2_\ell,\dots,b^n_\ell)$,
    and the affine transformations $W_\ell(x):=A_\ell x + b_\ell$. 
    The proof is concluded by noting that
    \begin{align*}
        \Psi:= W_L\circ\rho \circ W_{L-1} \circ \rho \circ \dots \circ \rho \circ W_{1}
    \end{align*}
    satisfies the claimed properties.
\end{proof}

We are now ready to formalize the concept of a linear combination of neural networks.

\begin{lemma}
\label{network_linearcombination}
Let $n,L,d'\in\N$ and, for $i\in\{1,2,\dots,n\}$, let $d_i\in\N$, $a_i\in\R$, and $\Phi_i\in\cN_{d_i,d'}$ with $\cL(\Phi_i)=L$.
Then, there exists a network $\Psi \in \cN_{\sum_{i=1}^n d_i,d'}$ with $\cL(\Psi) = L$, $\cM(\Psi)\leq\sum_{i=1}^n \cM(\Phi_i)$, $\mathcal{W}(\Psi) \leq \sum_{i=1}^n \mathcal{W}(\Phi_i)$, $\cB(\Psi)=\max_i\{|a_i|\cB(\Phi_i)\}$, and satisfying
\begin{align*}
    \Psi(x)&=\sum_{i=1}^n a_i\Phi_i(x_i)\in\R^{d'},
\end{align*}
for $x=(x_1,x_2,\dots,x_n)\in\R^{\sum_{i=1}^n d_i}$ with $x_i\in\R^{d_i}$, $i\in \{1,2,\dots,n\}$.
\end{lemma}
\begin{proof}
    The proof follows by taking the construction in Lemma~\ref{network_parallelization}, replacing $A_L$ by $(a_1A^1_L,a_2A^2_L,\dots,a_nA^n_L)$, $b_L$ by $\sum_{i=1}^n a_i b^i_L$, and noting that the resulting network satisfies the claimed properties.
\end{proof}

\section{Approximation of multiplication, polynomials, smooth functions, and sinusoidals}
\label{func-mult}

This section constitutes the first part of the paper dealing with the approximation of basic function ``templates" through neural networks. Specifically, we shall develop an algebra of neural network approximation by starting with the squaring function, building thereon to approximate the multiplication function, proceeding to polynomials and general smooth functions, and ending with sinusoidal functions.

The basic element of the neural network algebra we develop is based on an approach by Yarotsky~\cite{yarotsky:2016ReLU} and by Schmidt-Hieber \cite{Schmidt_Hieber}, both of whom, in turn, employed the ``sawtooth'' construction from \cite{telgarsky:2015Sawtooth}.

We start by reviewing the sawtooth construction underlying our program.
Consider the hat function $g: \R \rightarrow [0,1]$,
\begin{equation*}
g(x) = 2\rho(x)-4\rho(x-\tfrac{1}{2})+2\rho(x-1)=\begin{cases} 
2x, &\mbox{if } 0\leq x < \frac{1}{2}\\
2(1-x), &\mbox{if } \frac{1}{2}\leq x \leq 1 \\ 
0, &\mathrm{else}
\end{cases},
\end{equation*}
let $g_{0}(x)=x,g_{1}(x)=g(x)$, and define the $s$-th order sawtooth function $g_s$ as the $s$-fold composition of $g$ with itself, i.e.,
\begin{equation}
\label{g_s_def}
g_s := \underbrace{g \circ g \circ \dots \circ g}_{s},
\quad s \geq 2.
\end{equation}
We note that $g$ can be realized by a $2$-layer network $\Phi_g \in \cN_{1,1}$ according to 
$\Phi_g: = W_2 \circ \rho \circ W_1 = g$
with
\[W_1(x) =  \hspace{-0.1cm} 
 \begin{pmatrix}
  1\\
  1\\
  1
 \end{pmatrix}  
 \hspace{-0.1cm}x \
 -  
 \begin{pmatrix}
  0 \\
  1/2 \\
  1
 \end{pmatrix} \hspace{-0.1cm}, \hspace{1cm}
 W_2(x) = \begin{pmatrix}
  2 & -4 & 2
 \end{pmatrix}  \hspace{-0.1cm} \begin{pmatrix}
 x_1\\
 x_2\\
 x_3
 \end{pmatrix}\hspace{-0.1cm}.
\]
The $s$-th order sawtooth function $g_s$ can hence be realized by a network $\Phi^s_g \in \cN_{1,1}$ according to 
\begin{align}\label{eq:Phisg}
   \Phi^s_g := W_2 \circ \rho \circ \underbrace{W_g \circ \rho \circ \dots \circ W_g \circ \rho}_{s-1}\circ \, W_1 = g_s 
\end{align}
with
\[W_g(x) = \begin{pmatrix}
  2 & -4 & 2\\
  2 & -4 & 2\\
  2 & -4 & 2
 \end{pmatrix} \begin{pmatrix}
 x_1\\
 x_2\\
 x_3
 \end{pmatrix}
 -
 \begin{pmatrix}
  0 \\
  1/2 \\
  1
 \end{pmatrix}.
\]
The following restatement of \cite[Lemma 2.4]{telgarsky:2015Sawtooth} summarizes the self-similarity and symmetry properties of $g_{s}(x)$ we will frequently make use of.
\begin{lemma}\label{lem:gsproperties}
For $s\in\N$, $k\in\{0,1,\dots,2^{s-1}-1\}$, it holds that $g(2^{s-1}\cdot-k)$ is supported in $\left[\frac{k}{2^{s-1}}, \frac{k+1}{2^{s-1}}\right]$,
\begin{align*}
   g_s(x) = \sum^{2^{s-1}-1}_{k=0} g(2^{s-1}x-k),\quad \mathrm{for}\ x\in[0,1],
\end{align*} 
and
\begin{align*}
    g_s\left(\tfrac{k}{2^{s-1}}+x\right)=g_s\left(\tfrac{k+1}{2^{s-1}}-x\right),\quad \mathrm{for}\ x\in\left[0, \tfrac{1}{2^{s-1}}\right].
\end{align*}
\end{lemma}

We are now ready to proceed with the statement of the basic building block of our neural network algebra, namely the approximation of the squaring function through deep ReLU networks.

\begin{figure}
    \centering
\newcommand{\Figscale}{0.75}
\newcommand{\FigLineT}{thick}
\begin{tikzpicture}[scale=\Figscale,
  declare function={
    func(\x)= 2*max(0,\x)-4*max(0,\x-0.5);
  }
]
\begin{axis}[
  ymin=0, ymax=0.25, 
  xmin=0, xmax=1,
  domain=0:1,   
  xtick={0,0.25,0.5,0.75,1},
  xticklabels={0,$\frac{1}{4}$,$\frac{2}{4}$,$\frac{3}{4}$,1},
  ytick={0,1/16,2/16,3/16,4/16},
  yticklabels={0,$\frac{1}{16}\ $,$\frac{2}{16}\ $,$\frac{3}{16}\ $,$\frac{4}{16}\ $},
  width=5cm,
  height=5cm,
  scale only axis=true,
  legend style={at={(1,1)},anchor=north}
];
  \addplot[black, \FigLineT, samples=200]{x-x*x};
  \addplot[blue, \FigLineT, samples at={0,0.5,1}]{func(x)/4};
  \addplot[red!80!black, \FigLineT, samples=200]{x-x*x-func(x)/4};
  \legend{$F$,$I_1$,$F-I_1$}
\end{axis}
\end{tikzpicture}\hspace{0.5cm}
\begin{tikzpicture}[scale=\Figscale,
  declare function={
    func(\x)= 2*max(0,\x)-4*max(0,\x-0.5);
  }
]
\begin{axis}[
  ymin=0, ymax=1/16, 
  xmin=0, xmax=1, 
  domain=0:1,
  scaled ticks=false,
  xtick={0,1/8,2/8,3/8,4/8,5/8,6/8,7/8,1},
  xticklabels={0,$\frac{1}{8}$,$\frac{2}{8}$,$\frac{3}{8}$,$\frac{4}{8}$,$\frac{5}{8}$,$\frac{6}{8}$,$\frac{7}{8}$,1},
  ytick={0,1/64,2/64,3/64,4/64},
  yticklabels={0,$\frac{1}{64}$,$\frac{2}{64}$,$\frac{3}{64}$,$\frac{4}{64}$},
  width=10cm,
  height=5cm,
  scale only axis=true,
  legend style={at={(1,1)},anchor=north}
];
  \addplot[red!80!black, \FigLineT, samples=257]{x-x*x-func(x)/4};
  \addplot[blue, \FigLineT, samples at={0,0.25,0.5,0.75,1}]{func(func(x))/16};
  \addplot[green!50!black, \FigLineT, samples=257]{x-x*x-func(x)/4-func(func(x))/16};
  \legend{$F-I_1$,$I_2-I_1$,$F-I_2$}
\end{axis}
\end{tikzpicture}
\begin{tikzpicture}[scale=\Figscale,
  declare function={
    func(\x)= 2*max(0,\x)-4*max(0,\x-0.5);
  }
]
\begin{axis}[
  ymin=0, ymax=1/64, 
  xmin=0, xmax=1, 
  domain=0:1,
  scaled ticks=false,
  xtick={0,1/16,2/16,3/16,4/16,5/16,6/16,7/16,8/16,9/16,10/16,11/16,12/16,13/16,14/16,15/16,1},
  xticklabels={0,$\frac{1}{16}$,$\frac{2}{16}$,$\frac{3}{16}$,$\frac{4}{16}$,$\frac{5}{16}$,$\frac{6}{16}$,$\frac{7}{16}$,$\frac{8}{16}$,$\frac{9}{16}$,$\frac{10}{16}$,$\frac{11}{16}$,$\frac{12}{16}$,$\frac{13}{16}$,$\frac{14}{16}$,$\frac{15}{16}$,1},
  ytick={0,1/256,2/256,3/256,4/256},
  yticklabels={0,$\frac{1}{256}$,$\frac{2}{256}$,$\frac{3}{256}$,$\frac{4}{256}$},
  width=20cm,
  height=5cm,
  scale only axis=true,
  legend style={at={(1,1)},anchor=north}
];
  \addplot[green!50!black, \FigLineT, samples=500]{x-x*x-func(x)/4-func(func(x))/16};
  \addplot[blue, \FigLineT, samples at={0,0.125,0.25,0.375,0.5,0.625,0.75,0.875,1}]{func(func(func(x)))/64};
  \addplot[orange, \FigLineT, samples=500]{x-x*x-func(x)/4-func(func(x))/16-func(func(func(x)))/64};
  \legend{$F-I_2$,$I_3-I_2$,$F-I_3$}
\end{axis}
\end{tikzpicture}
    \caption{First three steps of approximating $F(x)=x-x^2$ by an equispaced linear interpolation $I_m$ at $2^m+1$ points.}
    \label{fig:my_label}
\end{figure}
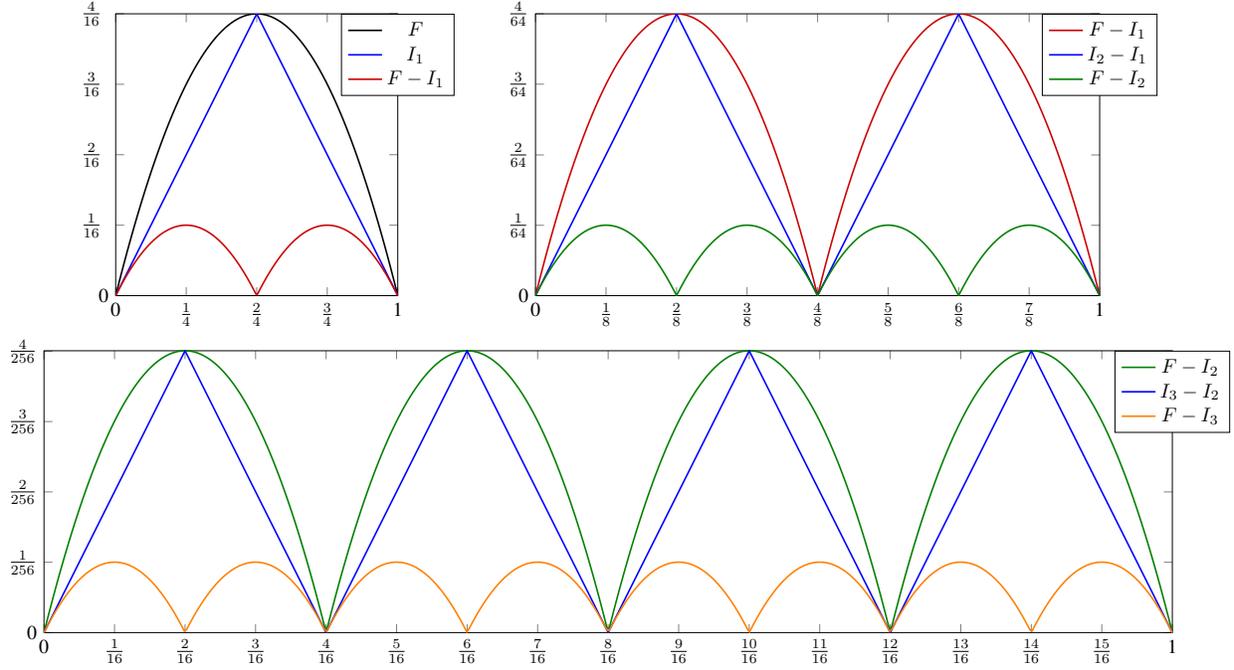

\begin{proposition}
\label{relu_square}
There exists a constant $C>0$ such that for all $\eps\in(0,1/2)$, there is a network $\Phi_{\epsilon}\in\cN_{1,1}$ with 
$\mathcal{L}(\Phi_{\epsilon}) \leq C\log(\epsilon^{-1})$,
$\mathcal{W}(\Phi_{\epsilon}) = 3$, 
$\mathcal{B}(\Phi_\eps) = 1$,
$\Phi_{\epsilon}(0)=0$, satisfying
\begin{equation*}
\|\Phi_{\epsilon}(x) - x^2 \|_{L^{\infty}([0,1])} \leq \epsilon.
\end{equation*}
\end{proposition}

\begin{proof}
    The proof builds on two rather elementary observations. The first one concerns the linear interpolation $I_m\colon[0,1]\to\R$, $m \in \N$, of the function $F(x):=x-x^2$ at the points $\tfrac{j}{2^m}$, $j\in\{0,1,\dots,2^m\}$, and in particular the self-similarity of the refinement step $I_m\to I_{m+1}$. For every $m\in\N$, the residual $F-I_m$ is identical on each interval between two points of interpolation (see Figure~\ref{fig:my_label}).
    Concretely, let $f_m\colon[0,2^{-m}]\to[0,2^{-2m-2}]$ be defined as $f_m(x)=2^{-m}x-x^2$ and consider its linear interpolation $h_m\colon[0,2^{-m}]\to[0,2^{-2m-2}]$ at the midpoint and the endpoints of the interval $[0,2^{-m}]$ given by
    \begin{align*}
        h_m(x):=\begin{cases}2^{-m-1}x, & x\in[0,2^{-m-1}] \\ -2^{-m-1}x + 2^{-2m-1}, & x\in[2^{-m-1} ,2^{-m}]  \end{cases}.
    \end{align*}
    Direct calculation shows that
    \begin{align*}
        f_m(x)-h_m(x) = \begin{cases}f_{m+1}(x), & x\in[0,2^{-m-1}] \\ f_{m+1}(x-2^{-m-1}), & x\in[2^{-m-1} ,2^{-m}]  \end{cases}.
    \end{align*}
    As $F=f_0$ and $I_1=h_0$ this implies that, for all $m\in\N$, 
    \begin{align*}
        F(x)-I_m(x) = f_m(x-\tfrac{j}{2^{m}}), \,\,\, \mathrm{for} \,\,\, x\in[\tfrac{j}{2^m},\tfrac{j+1}{2^m}], \,\,\, j\in\{0,1,\dots,2^m-1\}
    \end{align*}
    and $I_m=\sum_{k=0}^{m-1} H_k$, where $H_k\colon[0,1]\to\R$ is given by 
    \begin{align*}
        H_k(x) = h_k(x-\tfrac{j}{2^k}),\,\,\, \mathrm{for} \,\,\, x\in[\tfrac{j}{2^k},\tfrac{j+1}{2^k}], \,\,\, j\in\{0,1,\dots,2^k-1\}.
    \end{align*}
    Thus, we have
    \begin{align}\label{eq:square_inter_err}
        \sup_{x\in[0,1]}|x^2 - (x - I_m(x))| = \sup_{x\in[0,1]}|F(x)-I_m(x)| =
        \sup_{x\in[0,2^{-m}]}|f_m(x)| = 2^{-2m-2}.
    \end{align}
    
    The second observation we build on is a manifestation of the sawtooth construction described above and leads to economic realizations of the $H_k$ through $k$-layer networks with two neurons in each layer; a third neuron is used to realize the approximation  $x-I_m(x)$ to $x^2$. Concretely, let $s_k(x):=2^{-1}\rho(x)-\rho(x-2^{-2k-1})$, and note that, for $x\,\in\,[0,1]$, $H_0=s_0$, we get $H_k=s_k\circ H_{k-1}$. We can thus construct a network realizing $x-I_m(x)$, for $x\in[0,1]$, as follows. Let $A_1:=(1,1,1)^T\in\R^{3\times 1}$, $b_1:=(0,-2^{-1},0)^T\in\R^{3}$,
    \begin{align*}
        A_\ell:=\begin{pmatrix} 2^{-1} & -1 & 0 \\ 2^{-1} & -1 & 0 \\ -2^{-1} & 1 & 1\end{pmatrix}\in\R^{3\times 3}, \quad b_\ell:=\begin{pmatrix} 0 \\ -2^{-2\ell+1} \\ 0 \end{pmatrix}\in\R^{3},\quad \text{for } \ell\in\{2,\dots,m\},
    \end{align*}
    and $A_{m+1}:=(-2^{-1},1,1)\in\R^{1 \times 3}$, $b_{m+1}=0$. Setting $W_\ell(x):=A_\ell x + b_\ell$, $\ell\in\{1,2,\dots,m+1\}$, and
    \begin{align*}
        \tilde{\Phi}_m:= W_{m+1}\circ\rho \circ W_{m} \circ \rho \circ \dots \circ \rho \circ W_{1},
    \end{align*}
    a direct calculation yields $\tilde{\Phi}_m(x)= x - \sum_{k=0}^{m-1} H_k(x)$, for $x\in[0,1]$. The proof is completed upon noting that the networks $\Phi_\eps:=\tilde{\Phi}_{\lceil \log(\eps^{-1})/2\rceil}$ satisfy the claimed properties.
\end{proof}

The symmetry properties of $g_s(x)$
according to Lemma \ref{lem:gsproperties} lead to the interpolation error in the proof of Proposition~\ref{relu_square} being identical in each interval, with the maximum error taken on at the centers of the respective intervals. More importantly, however, the approximating neural networks realize linear interpolation at a number of points that grows exponentially in network depth. This is a manifestation of the fact that the number of linear regions in the sawtooth construction (\ref{eq:Phisg}) grows exponentially with depth, which, owing to Lemma~\ref{defsawtooth}, is optimal. 
We emphasize that the theory developed in this paper hinges critically on this optimality property, which, however, is brittle in the sense that networks with weights obtained through training will, as observed in \cite{NIPS2019_8328}, in general, not exhibit exponential growth of the number of linear regions with network depth. 
An interesting approach to neural network training which manages to partially circumvent this problem was proposed recently in~\cite{fokina2019growing}.
Understanding how the number of linear regions grows in general trained networks and quantifying the impact of this---possibly subexponential---growth behavior on the approximation-theoretic fundamental limits of neural networks constitutes a major open problem.

We proceed to the construction of networks that approximate the multiplication function over the interval $[-D,D]$. This will be effected by using the result on the approximation of $x^{2}$ just established combined with the polarization identity $xy = \tfrac{1}{4}((x+y)^2-(x-y)^2)$, the fact that $\rho(x) + \rho(-x)=|x|$, and a scaling argument exploiting that the ReLU function is positive homogeneous, i.e., $\rho(\lambda x)=\lambda\rho(x)$, for all $\lambda \geq 0$, $x\in\R$.

\begin{proposition}
\label{relu_mult}
There exists a constant $C>0$ such that, for all $D \in \mathbb{R}_+$ and $\eps\in(0,1/2)$, there is a network $\Phi_{D,\epsilon}\in\cN_{2,1}$ with $\mathcal{L}(\Phi_{D,\epsilon}) \leq C (\log(\lceil D \rceil) +\log( \epsilon^{-1}))$, $\mathcal{W}(\Phi_{D,\epsilon}) \leq 5$, $\mathcal{B}(\Phi_{D,\eps}) = 1$, satisfying
$\Phi_{D,\eps}(0,x)=\Phi_{D,\eps}(x,0)=0$, for all $x\in\R$, and
\begin{equation}
\label{relu-mult-statement}
\|\Phi_{D,\epsilon}(x,y) - xy \|_{L^{\infty}([-D,D]^2)} \leq \epsilon.
\end{equation}
\end{proposition}

\begin{proof}
We first note that, w.l.o.g., we can assume $D\geq 1$ in the following, as for $D<1$, we can simply employ the network constructed for $D=1$ to
guarantee the claimed properties. The proof builds on the polarization identity and essentially constructs two squaring networks according to Proposition~\ref{relu_square} which share the neuron responsible for summing up the $H_k$, preceded by a layer mapping $(x,y)$ to $(|x+y|/(2D),|x-y|/(2D))$ and followed by layers realizing the multiplication by $D^2$ through weights bounded by $1$. Specifically, consider the network $\tilde{\Psi}_m$ with associated matrices $A_\ell$ and vectors $b_\ell$ given by 
    \begin{align*}
        A_1&:=\frac{1}{2D}\begin{pmatrix}1 & 1 \\ -1 & -1 \\ 1 & -1 \\ -1 & 1\end{pmatrix}
        \in\R^{4\times 2}, \quad b_1:=0\in\R^4,\quad
        A_2:=\begin{pmatrix}  1 & 1 & 0 & 0 \\ 1 & 1 & 0 & 0 \\ 1 & 1 & -1 & -1 \\ 0 & 0 & 1 & 1 \\ 0 & 0 & 1 & 1 \end{pmatrix}\in\R^{5 \times 4}, \quad b_2:=\begin{pmatrix} 0 \\ -2^{-1} \\ 0 \\ 0 \\ -2^{-1}
        \end{pmatrix}\\
        A_\ell&:=\begin{pmatrix}  2^{-1} & -1 & 0 & 0 & 0 \\ 2^{-1} & -1 & 0 & 0 & 0 \\ -2^{-1} & 1 & 1 & 2^{-1} & -1 \\ 0 & 0 & 0 & 2^{-1} & -1 \\ 0 & 0 & 0 & 2^{-1} & -1 \end{pmatrix}\in\R^{5 \times 5}, \quad b_\ell:=\begin{pmatrix} 0 \\ -2^{-2\ell+3} \\ 0 \\ 0 \\ -2^{-2\ell+3} \end{pmatrix}, \quad \text{for } \ell\in \{3,\dots,m+1\},
    \end{align*}
    and $A_{m+2}:=(-2^{-1}, 1, 1, 2^{-1},-1)\in\R^{1\times 5}$, $b_{m+2}:=0$. A direct calculation yields 
    \begin{align}\begin{split}\label{eq:polarization}
    \tilde{\Psi}_m(x,y)&=\left(\tfrac{|x+y|}{2D}-\sum_{k=0}^{m-1}H_k\big(\tfrac{|x+y|}{2D}\big)\right)-\left(\tfrac{|x-y|}{2D}-\sum_{k=0}^{m-1}H_k\big(\tfrac{|x-y|}{2D}\big)\right)\\
    &=\tilde{\Phi}_m\left(\tfrac{|x+y|}{2D}\right)-\tilde{\Phi}_m\left(\tfrac{|x-y|}{2D}\right),
    \end{split}\end{align}
    with $H_k$ and $\tilde{\Phi}_m$ as defined in the proof of Proposition~\ref{relu_square}. With \eqref{eq:square_inter_err} this implies 
    \begin{align}\begin{split}\label{eq:mult_helper_est}
        \sup_{(x,y)\in[-D,D]^2}\left|\tilde{\Psi}_m(x,y) - \tfrac{xy}{D^2}\right|&=\sup_{(x,y)\in[-D,D]^2}\left|\left(\tilde{\Phi}_m\left(\tfrac{|x+y|}{2D}\right)-\tilde{\Phi}_m\left(\tfrac{|x-y|}{2D}\right)
        \right) - \left(\left(\tfrac{|x+y|}{2D}\right)^2-\left(\tfrac{|x-y|}{2D}\right)^2\right)\right|\\
        &\leq2\sup_{z\in[0,1]}|\tilde{\Phi}_m(z)-z^2|\leq 2^{-2m-1}. 
    \end{split}
    \end{align}
    Next, let $\Psi_D(x)=D^2 x$ be the scalar multiplication network according to Lemma~\ref{lem:scalar_mult} and take
    $\Phi_{D,\eps}:= \Psi_D\circ\tilde{\Psi}_{m(D,\eps)}$, where $m(D,\eps):=\lceil2^{-1}(1+\log(D^2\eps^{-1}))\rceil$. Then, the error estimate~\eqref{relu-mult-statement} follows directly from~\eqref{eq:mult_helper_est} and Lemma~\ref{network_conc} establishes the desired bounds on depth, width, and weight magnitude. Finally, $\Phi_{D,\eps}(0,x)=\Phi_{D,\eps}(x,0)=0$, for all $x\in\R$, follows directly from \eqref{eq:polarization}.
\end{proof}

\begin{remark} 
Note that the multiplication network just constructed has weights bounded by $1$ irrespectively of the size $D$ of the domain. This is
accomplished by trading network depth for weight magnitude according to Lemma~\ref{lem:scalar_mult}.
\end{remark}
We proceed to the approximation of polynomials, effected by networks that realize linear combinations of monomials, which, in turn, are built 
by composing multiplication networks. Before presenting the specifics of this construction, we hasten to add that a similar approach was considered previously in~\cite{yarotsky:2016ReLU} and~\cite{Schmidt_Hieber}. While there are slight differences in formulation, the main distinction between our construction and those in~\cite{yarotsky:2016ReLU} and~\cite{Schmidt_Hieber} resides in their purpose. Specifically, the goal in \cite{yarotsky:2016ReLU} and~\cite{Schmidt_Hieber} is to establish, by way of local Taylor-series approximation, that $d$-variate, $k$-times (weakly) differentiable functions can be approximated in $L^\infty$-norm to within error $\eps$ with networks of connectivity scaling according to $\eps^{-d/k}\log(\eps^{-1})$. Here, on the other hand, we will be interested in functions that allow approximation with networks of connectivity scaling polylogarithmically in $\epsilon^{-1}$ (i.e., as a polynomial in $\log(\eps^{-1})$). Moreover, for ease of exposition, we will employ finite-width networks.
Polylogarithmic connectivity scaling will turn out to be crucial (see Sections \ref{subsec:NNapproxintro}-\ref{sec:gabor})
in establishing  Kolmogorov-Donoho rate-distortion optimality of neural networks in the approximation of a variety of prominent function classes.
Finally, we would like to mention related recent work \cite{schwab-zech-2018, OP_Pet_Schwa}, \cite{2019arXiv190207896G} on the approximation of Sobolev-class functions in certain Sobolev norms enabled by neural network approximations of the multiplication operation and of polynomials.

\begin{proposition}
\label{relu_poly}
There exists a constant $C>0$ such that for all $m \in \mathbb{N}$, $a=(a_i)_{i=0}^m\in\R^{m+1}$, $D \in \mathbb{R}_+$, and $\eps\in(0,1/2)$, there is a network $\Phi_{a,D,\epsilon} \in\cN_{1,1}$ with $\mathcal{L}(\Phi_{a,D,\epsilon}) \leq Cm(\log(\eps^{-1})+m\log(\lceil D \rceil) + \log(m) + \log(\lceil \|a\|_\infty\rceil)) $, $\mathcal{W}(\Phi_{a,D,\epsilon}) \le 9$, $\mathcal{B}(\Phi_{a,D,\epsilon})\leq 1$, and satisfying
\begin{equation*}
\|\Phi_{a,D,\epsilon}(x) -  \sum_{i=0}^m a_i x^i\|_{L^{\infty}([-D,D])} \leq \epsilon.
\end{equation*}
\end{proposition}
\begin{proof}
As in the proof of Proposition \ref{relu_mult} and for the same reason, it suffices to consider the case $D\ge 1$. For $m=1$, we simply have an affine transformation and the statement follows directly from Corollary~\ref{cor:matrix_mult}.
The proof for $m \geq 2$ will be effected by realizing the monomials
$x^k, k\geq 2$, through iterative composition of multiplication networks
and combining this with a construction that uses the network realizing $x^k$ not only as a building block in the network implementing $x^{k+1}$ but also to approximate the partial sum $\sum_{i=0}^k a_ix^i$ in parallel.

We start by setting $B_k=B_k(D,\eta):=\lceil D\rceil^k+\eta\sum_{s=0}^{k-2} \lceil D \rceil^s$, $k\in\N,\eta \in \mathbb{R}_+$ and take $\Phi_{B_k,\eta}$ to be the multiplication network from Proposition~\ref{relu_mult}. Next, we recursively define the functions
\begin{align*}
    f_{k,D,\eta}(x)=\Phi_{B_{k-1},\eta}(x,f_{k-1,D,\eta}(x)),\quad k\geq 2,
\end{align*}
with $f_{0,D,\eta}(x)= 1$ and $f_{1,D,\eta}(x)= x$.
For notational simplicity, we use the abbreviation $f_k=f_{k,D,\eta}$ in the following.
First, we verify that the $f_{k,D,\eta}$ approximate monomials sufficiently well.
Specifically, we prove by induction that
\begin{equation}
\label{psi_precise}
\|f_k(x) - x^k \|_{ L^{\infty}([-D,D])} \leq \eta \sum_{s=0}^{k-2} \lceil D \rceil^s,
\end{equation}
for all $k\geq 2$.
The base case $k=2$, i.e.,
\begin{align*}
\|f_2(x) - x^2 \|_{ L^{\infty}([-D,D])} = \| \Phi_{B_1,\eta}(x,x) - x^2 \|_{ L^{\infty}([-D,D])} \leq \eta,
\end{align*}
follows directly from Proposition~\ref{relu_mult} upon noting that $D\le B_1=\lceil D \rceil$ (we take the sum in the definition of $B_k$ to equal zero when the upper limit of summation is negative).
We proceed to establish the induction step $(k-1)\to k$ with the induction assumption given by
\begin{equation*}
\|f_{k-1}(x) - x^{k-1} \|_{ L^{\infty}([-D,D])} \leq \eta \sum_{s=0}^{k-3} \lceil D \rceil^s.
 \end{equation*}
As
\begin{align*}
\| f_{k-1} \|_{L^{\infty}([-D,D])}\leq 
\| x^{k-1} \|_{L^{\infty}([-D,D])}+\|f_{k-1}(x) - x^{k-1} \|_{L^{\infty}([-D,D])}\leq B_{k-1},
\end{align*}
application of Proposition~\ref{relu_mult} yields
\begin{align*}
\|f_k(x) - x^k \|_{ L^{\infty}([-D,D])}
&\leq \|f_k(x) - xf_{k-1}(x) \|_{ L^{\infty}([-D,D])} + \|xf_{k-1}(x) - x^k \|_{ L^{\infty}([-D,D])}\\
&\leq \| \Phi_{B_{k-1},\eta}(x,f_{k-1}(x)) - xf_{k-1}(x) \|_{ L^{\infty}([-D,D])}+ D\| f_{k-1}(x) - x^{k-1}\|_{ L^{\infty}([-D,D])}\\
&\leq \eta + \lceil D \rceil \eta \sum_{s=0}^{k-3} \lceil D \rceil^s = \eta \sum_{s=0}^{k-2} \lceil D \rceil^s,
\end{align*}
which completes the induction.

We now construct the network $\Phi_{a,D,\epsilon}$ approximating the polynomial $\sum_{i=0}^m a_i x^i$.
To this end, note that there exists a constant $C'$ such that for all $m\geq 2$, $a=(a_i)_{i=0}^m\in\R^{m+1}$, and $i\in\{1,\dots,m-1\}$, there is a network $\Psi^i_{a,D,\eta}\in\cN_{3,3}$ with
$\cL(\Psi^i_{a,D,\eta})\leq C' (\log(\eta^{-1})+\log(\lceil B_{i} \rceil)+\log(\|a\|_\infty))$, $\cW(\Psi^i_{a,D,\eta})\leq 9$, 
$\cB(\Psi^i_{a,D,\eta})\leq 1$, and satisfying
\begin{align*}\label{eq:polynomial_levels}
    \Psi^i_{a,D,\eta}(x,s,y)=(x,s+a_i y,\Phi_{B_i,\eta}(x,y)).
\end{align*}
To see that this is, indeed, the case, consider the following chain of mappings
\begin{align*}
    (x,s,y)
    \xrightarrow{ (I) }(x,s,y,y)
    \xrightarrow{ (II) }(x,s+a_i y,y)
    \xrightarrow{ (III) }(x,s+a_i y,x,y)
    \xrightarrow{ (IV) }(x,s+a_i y,\Phi_{B_i,\eta}(x,y)).
\end{align*}
Observe that the mapping (I) is an affine transformation with coefficients in $\{0,1\}$, which we can simply consider to be a depth-$1$ network. The mapping (II) is obtained by using Corollary~\ref{cor:matrix_mult} in order to implement the affine transformation $(s,y)\mapsto{s+a_i y}$ with weights bounded by 1, followed by application of Lemmas~\ref{network_extension} and~\ref{network_parallelization} to put this network in parallel with two networks realizing the identity mapping according to $x=\rho(x)-\rho(-x)$. Mapping (III) is obtained along the same lines by putting the result of mapping (II) in parallel with another network realizing the identity mapping. Finally, mapping (IV) is realized by putting the network $\Phi_{B_i,\eta}$ in parallel with two identity networks. Composing these four networks according to Lemma~\ref{network_conc} yields, for $i\in\{1,\dots,m-1\}$, a network $\Psi^i_{a,D,\eta}$ with the claimed properties. 
Next, we employ Corollary~\ref{cor:matrix_mult} to get networks $\Psi^0_{a,D,\eta}$ which implement $x\mapsto(x,a_0,x)$ as well as networks $\Psi^m_{a,D,\eta}$ realizing $(x,s,y)\mapsto s + a_m y$.
Let now $\eta=\eta(a,D,\eps):=(\|a\|_\infty(m-1)^2\lceil D\rceil^{m-2})^{-1}\eps$ and define
\begin{align*}
  \Phi_{a,D,\eps} := \Psi^m_{a,D,\eta}\circ \Psi^{m-1}_{a,D,\eta}\circ\dots\circ\Psi^1_{a,D,\eta}\circ\Psi^0_{a,D,\eta}.
\end{align*}
A direct calculation yields
\begin{align*}
    \Phi_{a,D,\eps} = \sum_{i=0}^m a_i f_{i,D,\eta}.
\end{align*}
Hence \eqref{psi_precise} implies
\begin{align*}
\Big \|\Phi_{a,D,\epsilon}(x) - \sum_{i=0}^m a_i x^i \Big \|_{L^{\infty}([-D,D])} &
\leq\sum_{i=0}^{m} |a_i| \|f_{i,D,\eta}(x)-x^i\|_{L^\infty([-D,D])}
\leq \sum_{i=2}^m |a_i| \Big( \eta \sum_{s=0}^{i-2} \lceil D \rceil^{s} \Big )\\
&\leq\|a\|_\infty\eta\sum_{k=0}^{m-2}(m-1-k)\lceil D\rceil^k\leq \|a\|_\infty(m-1)^2\lceil D\rceil^{m-2}\eta = \eps.
\end{align*}
Lemma~\ref{network_conc} now establishes that $\cW(\Phi_{a,D,\eps})\leq 9$, $\cB(\Phi_{a,D,\eps})\leq 1$, and
\begin{align*}
\mathcal{L}(\Phi_{a,D,\epsilon}) &\leq\sum_{i=0}^{m}\L(\Psi_{a,D,\eta}^i)\\
&\leq 2(\log(\lceil \|a\|_\infty \rceil) + 5) + \sum_{i=1}^{m-1} C' (\log(\eta^{-1})+\log(\lceil B_{i-1} \rceil)+\log(\lceil \|a\|_\infty \rceil)) \\
&\leq Cm(\log(\eps^{-1})+m\log(\lceil D \rceil) + \log(m) + \log(\lceil \|a\|_\infty\rceil))
\end{align*}
for a suitably chosen absolute constant $C$. This completes the proof.
\end{proof}

Next, we recall that the Weierstrass approximation theorem states that every continuous function on a closed interval can be approximated to within arbitrary accuracy by a polynomial.

\begin{theorem}[\cite{stone1948theorem}]\label{weierstrass}
Let $[a,b]\subseteq\R$ and $f \in C([a,b])$. Then, for every $\epsilon > 0$, there exists a polynomial $\pi$ such that
\[ \| f - \pi \|_{L^{\infty}([a,b])} \leq \epsilon. \]
\end{theorem}

Proposition \ref{relu_poly} hence allows us to conclude that every continuous function on a closed interval can be approximated to within arbitrary accuracy by a deep ReLU network of width no more than $9$.
This amounts to a variant of the universal approximation theorem \cite{Cybenko1989,Hornik1991251} for finite-width deep ReLU networks. 
A quantitative statement in terms of making the approximating network's width, depth, and weight bounds explicit can be obtained for (very) smooth functions by applying Proposition~\ref{relu_poly} to Lagrangian interpolation with Chebyshev points. 

\pagebreak

\begin{lemma}\label{Sfunctions}
  Consider the set
 \begin{align*}
   \Ss_{[-1,1]}:=\left\{f\in C^{\infty}([-1,1],\R)\colon \|f^{(n)}(x)\|_{ L^{\infty}([-1,1])} \leq n!,\, \text{\emph{ for all }} n\in\N_0\right\}.
 \end{align*}
  There exists a constant $C>0$ such that for all $f\in\Ss_{[-1,1]}$ and $\eps\in(0,1/2)$, there is a network $\Psi_{f,\eps}\in\cN_{1,1}$ with $\mathcal{L}(\Psi_{f,\eps}) \le C(\log(\eps^{-1}))^2$, $\mathcal{W}(\Psi_{f,\eps}) \le 9$,
  $\mathcal{B}(\Psi_{f,\eps})\leq 1$, and satisfying
   \begin{align*}
    \|\Psi_{f,\eps}-f\|_{ L^{\infty}([-1,1])}\leq\eps.
  \end{align*}
\end{lemma}

\begin{proof}
  A fundamental result on Lagrangian interpolation with Chebyshev points (see e.g. \cite[Lemma 3]{srikant:2017Poly}) guarantees, for all $f\in \Ss_{[-1,1]}$, $m\in\N$, the existence of a polynomial $P_{f,m}$ of degree $m$ such that
  \begin{align*}
    \|f-P_{f,m}\|_{L^{\infty}([-1,1])}\leq \tfrac{1}{(m+1)!2^m}\|f^{(m+1)}\|_{L^{\infty}([-1,1])}\leq\tfrac{1}{2^m}.
  \end{align*}
  Note that $P_{f,m}$ can be expressed in the Chebyshev basis (see e.g. \cite[Section 3.4.1]{GA_SJ_TN}) according to $P_{f,m}=\sum_{j=0}^m c_{f,m,j}T_j(x)$ with $|c_{f,m,j}|\leq 2$ and the Chebyshev polynomials defined through the two-term recursion $T_k(x)=2xT_{k-1}(x)-T_{k-2}(x)$, $k\geq 2$, with $T_0(x)=1$ and $T_1(x)=x$. We can moreover use this recursion to conclude that the coefficients of the $T_k$ in the monomial basis are upper-bounded by $3^k$. Consequently, we can express $P_{f,m}$ according to $P_{f,m}=\sum_{j=0}^m a_{f,m,j}x^j$ with
  \begin{align*}
      A_{f,m}:=\max_{j=0,\dots,m}|a_{f,m,j}|\leq 2(m+1)3^m.
  \end{align*}
  Application of Proposition \ref{relu_poly} to $P_{f,m}$ in the monomial basis, with $m=\lceil\log(2/\eps)\rceil$ and approximation error $\eps/2$, completes the proof upon noting that
  \begin{align*}
  C'm(\log(2/\eps) + \log(m) + \log(|A_{f,m}|))\leq C(\log(\eps^{-1}))^2
  \end{align*}
  for some absolute constant $C$.
\end{proof}

An extension of Lemma~\ref{Sfunctions} to approximation over general intervals is provided in Lemma~\ref{lem:Sfunctions_general}.
While Lemma\\ \ref{Sfunctions} shows that a specific class of $C^{\infty}$-functions, namely those whose derivatives are suitably bounded, can be approximated by neural networks with connectivity growing polylogarithmically in $\eps^{-1}$, it turns out that this is not possible for general (Sobolev-class) $k$-times differentiable functions \cite[Thm.~4]{yarotsky:2016ReLU}.

We are now ready to proceed to the approximation of sinusoidal functions.
Before stating the corresponding result, we comment on the basic idea enabling the approximation of oscillatory functions through deep neural networks. In essence, we exploit the optimality of the sawtooth construction (\ref{eq:Phisg}) in terms of achieving exponential---in network depth---growth in the number of linear regions. As indicated in Figure~\ref{cos_visualized}, the composition of the cosine function (realized according to Lemma~\ref{Sfunctions}) with the sawtooth function, combined with the symmetry properties of the cosine function and the sawtooth function, yields oscillatory behavior that increases exponentially with network depth.

\begin{theorem}\label{sin}
There exists a constant $C>0$ such that for every $a,D\in \mathbb{R}_+$, $\eps\in(0,1/2)$, there is a network $\Psi_{a,D,\epsilon} \in \cN_{1,1}$ 
with $\L(\Psi_{a,D,\epsilon}) \le C((\log(\eps^{-1}))^2 + \log(\lceil a D \rceil))$, $\mathcal{W}(\Psi_{a,D,\epsilon}) \le 9$, $\mathcal{B}(\Psi_{a,D,\epsilon})\leq 1$, and satisfying
\begin{align*}
\|\Psi_{a,D,\epsilon}(x) - \cos(ax)\|_{L^\infty([-D,D])} \leq \epsilon.
\end{align*}
\end{theorem}
\begin{proof}
Note that $f(x):=(6/\pi^3)\cos(\pi x)$ is in $\Ss_{[-1,1]}$. Thus, by Lemma~\ref{Sfunctions}, there exists a constant $C>0$ such that for every $\eps\in(0,1/2)$, there is a network $\Phi_\eps\in\cN_{1,1}$ with $\cL(\Phi_\eps)\leq C(\log(\eps^{-1}))^2$, $\cW(\Phi_\eps) \le 9$, $\cB(\Phi_\eps) \le 1$, and satisfying
\begin{align}\label{cos_approx} 
\|\Phi_\eps-f\|_{L^\infty([-1,1])}\leq \tfrac{6}{\pi^3}\eps.    
\end{align}

We now extend this result to the approximation of $x\mapsto\cos(ax)$ on the interval $[-1,1]$ for arbitrary $a\in\R_+$. This will be accomplished by exploiting that $x\mapsto\cos(\pi x)$ is 2-periodic and even.
Let $g_s\colon [0,1]\to[0,1]$, $s\in\N$, be the s-th order sawtooth functions as defined in \eqref{g_s_def} and note that, due to the periodicity and the symmetry of the cosine function (see Figure~\ref{cos_visualized} for illustration), we have for all $s\in\N_0$, $x \in [-1,1]$,
\begin{align*}
    \cos(\pi 2^{s}x)=\cos(\pi g_s(|x|)).
\end{align*}
For $a>\pi$, we define $s=s(a):=\lceil \log(a) - \log(\pi) \rceil$ and
$\alpha=\alpha(a):=(\pi2^s)^{-1}a\in(1/2,1]$, and note that 
\begin{align*}
    \cos(a x) 
    = \cos(\pi 2^s\alpha x) 
    = \cos(\pi g_s(\alpha|x|)), \quad x \in [-1,1].
\end{align*}
As $g_s(\alpha |x|)\in[0,1]$, it follows from \eqref{cos_approx} that 
\begin{align}\label{fpsieps}
    \|\tfrac{\pi^3}{6}\Phi_\eps(g_s(\alpha |x|))-\cos(ax)\|_{L^{\infty}([-1,1])}
    =\tfrac{\pi^3}{6}\|\Phi_\eps(g_s(\alpha|x|))-f(g_s(\alpha|x|))\|_{L^{\infty}([-1,1])} \leq \epsilon.
\end{align}
In order to realize $\Phi_\eps(g_s(\alpha |x|))$ as a neural network, we
start from the networks $\Phi^s_g$ defined in \eqref{eq:Phisg} and apply Proposition~\ref{WDtradeoff} to convert them into networks $\Psi^s_g(x)=g_s(x)$, for $x\in[0,1]$, with $\cB(\Psi^s_g) \le 1$, $\cL(\Psi^s_g)=7(s+1)$, and $\cW(\Psi^s_g)=3$.
Furthermore, let $\Psi(x):=\alpha\rho(x)-\alpha\rho(-x)=\alpha|x|$ and take $\Phi^{\text{mult}}_{\pi^3\!/6}$ to be the scalar multiplication network from Lemma~\ref{lem:scalar_mult}.
Noting that $\Psi_{a,\eps}:=\Phi^{\text{mult}}_{\pi^3\!/6}\circ\Phi_\eps\circ\Psi^s_g\circ\Psi=\Phi_\eps(g_s(\alpha |x|))$ and concluding from Lemma~\ref{network_conc} that
$\cL(\Psi_{a,\eps})\leq C((\log(\eps^{-1}))^2+\log(\lceil a \rceil))$, $\cW(\Psi_{a,\eps})\le 9$, and $\cB(\Psi_{a,\eps}) \le 1$, together with (\ref{fpsieps}), establishes the desired result for $a>\pi$ and for approximation over the interval $[-1,1]$. For $a\in(0,\pi)$, we can simply take $\Psi_{a,\eps}:=\Phi^{\text{mult}}_{\pi^3\!/6}\circ\Phi_\eps$ as $x\mapsto (6/\pi^3)\cos(ax)$ is in $\Ss_{[-1,1]}$ in this case.

Finally, we consider the approximation of $x\mapsto\cos(ax)$ on intervals $[-D,D]$, for arbitrary $D \geq 1$. 
To this end, we define the networks $\Psi_{a,D,\eps}(x):=\Psi_{aD,\eps}(\tfrac{x}{D})$ and observe that
\begin{align}\begin{split}\label{psi-transformed}
    \sup_{x\in[-D,D]}|\Psi_{a,D,\eps}(x)-\cos(ax)|&=\sup_{y\in[-1,1]}|\Psi_{a,D,\eps}(Dy)-\cos(aDy)|\\
    &=\sup_{y\in[-1,1]}|\Psi_{aD,\eps}(y)-\cos(aDy)|\leq\eps.
\end{split}
\end{align}
This concludes the proof.
\end{proof}

\begin{center}
\newcommand{\CosFigScale}{0.88}
\begin{tabular}{ c c }
 
 \begin{tikzpicture}[scale=\CosFigScale]
	\begin{axis}[
		xlabel=$x$,ylabel=$g(x)$]
	\addplot[color=blue,mark=*] coordinates {
		(0,0)
		(0.15,0.3)
		(0.3,0.6)
		(0.5,1)
		(0.7,0.6)
		(0.85,0.3)
		(1,0)
	};
	\end{axis}%
\end{tikzpicture} & \begin{tikzpicture}[scale=\CosFigScale]
	\begin{axis}[
		xlabel=$x$,ylabel=$g(g(x))$]

	\addplot[color=blue,mark=*] coordinates {
		(0,0)
		(0.125,0.5)
		(0.25,1)
		(0.375,0.5)
		(0.5,0)
		(0.625,0.5)
		(0.75,1)
		(0.875,0.5)
		(1,0)
	};
	\end{axis}%
\end{tikzpicture}\\

\begin{tikzpicture}[scale=\CosFigScale]
	\begin{axis}[
		xlabel=$x$,
		ylabel=$\cos(2 \pi x)$
	]
	
		\addplot[color=blue,mark=*] coordinates {
(  0.00,  1.00)
(  0.05,  0.95)
(  0.10,  0.81)
(  0.15,  0.59)
(  0.20,  0.31)
(  0.25,  0.00)
(  0.30, -0.31)
(  0.35, -0.59)
(  0.40, -0.81)
(  0.45, -0.95)
(  0.50, -1.00)
(  0.55, -0.95)
(  0.60, -0.81)
(  0.65, -0.59)
(  0.70, -0.31)
(  0.75, -0.00)
(  0.80,  0.31)
(  0.85,  0.59)
(  0.90,  0.81)
(  0.95,  0.95)
(  1.00,  1.00)
};
	\end{axis}
    \draw[] (0, -1.5) node{0};
	\draw[-{Latex[length=3mm]}] (0.3,-1.5)--(6.5,-1.5); 
	\draw[] (6.8, -1.75) node{1/2};
	\draw[-{Latex[length=3mm]}] (6.5,-2)--(0.3,-2); 
	\draw[] (0, -2) node{1};	
	\draw[opacity=0,-{Latex[length=3mm]}] (6.5,-2.5)--(0.3,-2.5);
	\draw[opacity=0] (0, -2.5) node{1};
\end{tikzpicture} & \begin{tikzpicture}[scale=\CosFigScale]
	\begin{axis}[
		xlabel=$x$,
		ylabel=$\cos(2 \pi x)$
	]

		\addplot[color=blue,mark=*] coordinates {
(  0.00,  1.00)
(  0.05,  0.95)
(  0.10,  0.81)
(  0.15,  0.59)
(  0.20,  0.31)
(  0.25,  0.00)
(  0.30, -0.31)
(  0.35, -0.59)
(  0.40, -0.81)
(  0.45, -0.95)
(  0.50, -1.00)
(  0.55, -0.95)
(  0.60, -0.81)
(  0.65, -0.59)
(  0.70, -0.31)
(  0.75, -0.00)
(  0.80,  0.31)
(  0.85,  0.59)
(  0.90,  0.81)
(  0.95,  0.95)
(  1.00,  1.00)
};
	\end{axis}

    \draw[] (0, -1) node{0};
	\draw[-{Latex[length=3mm]}] (0.3,-1)--(6.5,-1);
	\draw[] (6.8,-1.25) node{1/4};
    \draw[] (0, -1.75) node{1/2};
	\draw[-{Latex[length=3mm]}] (6.5,-1.5)--(0.3,-1.5);
	\draw[] (6.8,-2.25) node{3/4};
	\draw[-{Latex[length=3mm]}] (0.3,-2)--(6.5,-2); 
	\draw[-{Latex[length=3mm]}] (6.5,-2.5)--(0.3,-2.5); 
	\draw[] (0, -2.5) node{1};
\end{tikzpicture} \\ 

\begin{tikzpicture}[scale=\CosFigScale]
	\begin{axis}[
		xlabel=$x$,
		ylabel=${\cos(2 \pi 2x) = \cos(2 \pi g(x))}$
	]
		\addplot[color=blue,mark=*] coordinates {
(  0.00,  1.00)
(  0.02,  0.97)
(  0.04,  0.88)
(  0.06,  0.73)
(  0.08,  0.54)
(  0.10,  0.31)
(  0.12,  0.06)
(  0.14, -0.19)
(  0.16, -0.43)
(  0.18, -0.64)
(  0.20, -0.81)
(  0.22, -0.93)
(  0.24, -0.99)
(  0.26, -0.99)
(  0.28, -0.93)
(  0.30, -0.81)
(  0.32, -0.64)
(  0.34, -0.43)
(  0.36, -0.19)
(  0.38,  0.06)
(  0.40,  0.31)
(  0.42,  0.54)
(  0.44,  0.73)
(  0.46,  0.88)
(  0.48,  0.97)
(  0.50,  1.00)
(  0.52,  0.97)
(  0.54,  0.88)
(  0.56,  0.73)
(  0.58,  0.54)
(  0.60,  0.31)
(  0.62,  0.06)
(  0.64, -0.19)
(  0.66, -0.43)
(  0.68, -0.64)
(  0.70, -0.81)
(  0.72, -0.93)
(  0.74, -0.99)
(  0.76, -0.99)
(  0.78, -0.93)
(  0.80, -0.81)
(  0.82, -0.64)
(  0.84, -0.43)
(  0.86, -0.19)
(  0.88,  0.06)
(  0.90,  0.31)
(  0.92,  0.54)
(  0.94,  0.73)
(  0.96,  0.88)
(  0.98,  0.97)
(  1.00,  1.00)
};
	\end{axis}
\end{tikzpicture} & \begin{tikzpicture}[scale=\CosFigScale]
	\begin{axis}[
		xlabel=$x$,
		ylabel=${\cos(2 \pi 4x) = \cos(2 \pi g(g(x)))}$
	]
		\addplot[color=blue,mark=*] coordinates {
(  0.00,  1.00)
(  0.02,  0.88)
(  0.04,  0.54)
(  0.06,  0.06)
(  0.08, -0.43)
(  0.10, -0.81)
(  0.12, -0.99)
(  0.14, -0.93)
(  0.16, -0.64)
(  0.18, -0.19)
(  0.20,  0.31)
(  0.22,  0.73)
(  0.24,  0.97)
(  0.26,  0.97)
(  0.28,  0.73)
(  0.30,  0.31)
(  0.32, -0.19)
(  0.34, -0.64)
(  0.36, -0.93)
(  0.38, -0.99)
(  0.40, -0.81)
(  0.42, -0.43)
(  0.44,  0.06)
(  0.46,  0.54)
(  0.48,  0.88)
(  0.50,  1.00)
(  0.52,  0.88)
(  0.54,  0.54)
(  0.56,  0.06)
(  0.58, -0.43)
(  0.60, -0.81)
(  0.62, -0.99)
(  0.64, -0.93)
(  0.66, -0.64)
(  0.68, -0.19)
(  0.70,  0.31)
(  0.72,  0.73)
(  0.74,  0.97)
(  0.76,  0.97)
(  0.78,  0.73)
(  0.80,  0.31)
(  0.82, -0.19)
(  0.84, -0.64)
(  0.86, -0.93)
(  0.88, -0.99)
(  0.90, -0.81)
(  0.92, -0.43)
(  0.94,  0.06)
(  0.96,  0.54)
(  0.98,  0.88)
(  1.00,  1.00)
};
	\end{axis}
\end{tikzpicture}
\end{tabular}

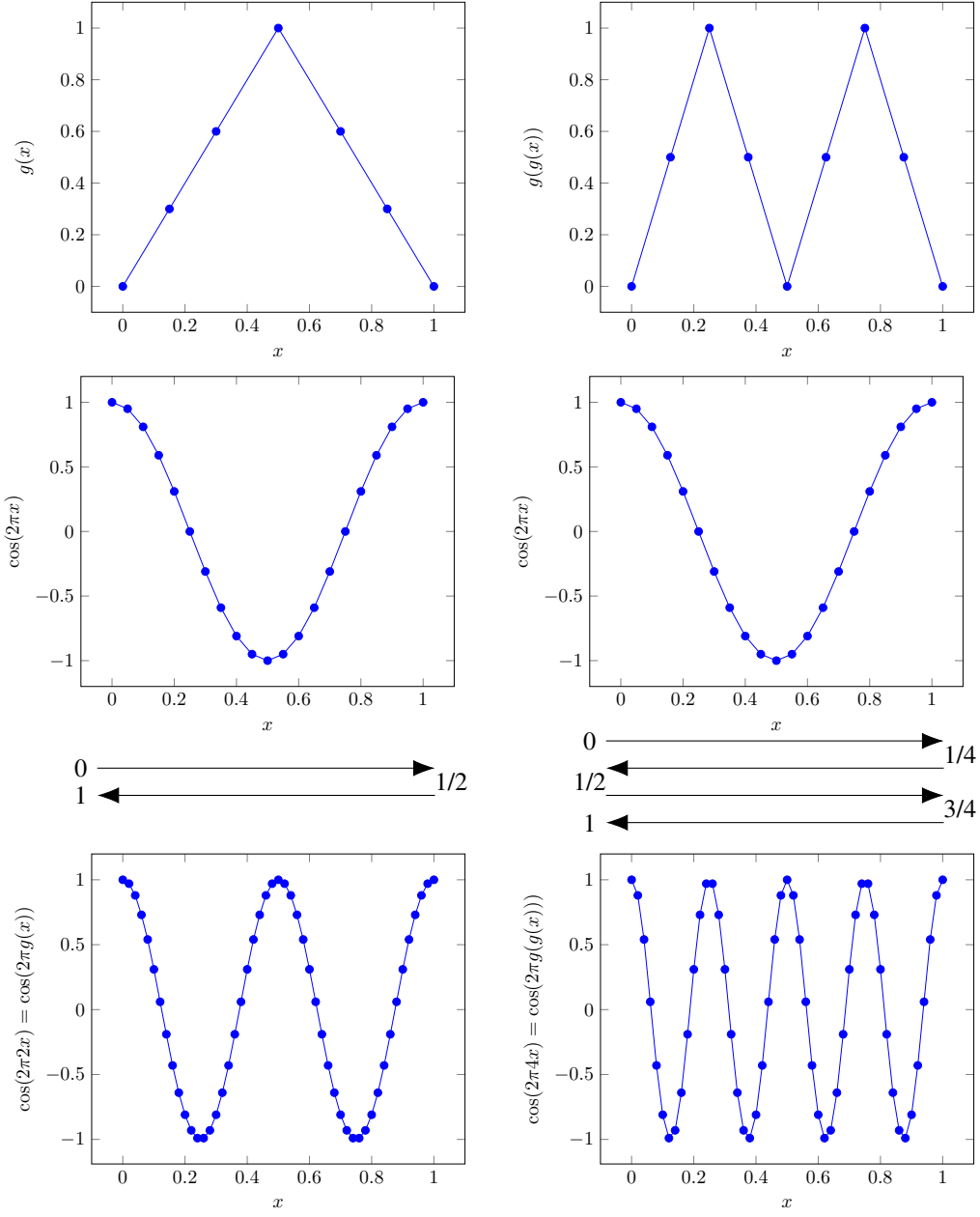
\captionof{figure}{Approximation of the function $\cos(2 \pi a x)$ according to Theorem \ref{sin} using ``sawtooth'' functions $g_s(x)$ as per (\ref{g_s_def}), left $a=2$, right $a=4$.}
\label{cos_visualized}
\end{center}

The result just obtained extends to the approximation of $x\mapsto\sin(ax)$, formalized next, simply by noting that $\sin(x)=\cos(x-\pi/2)$.

\begin{corollary}
\label{sin_shifted}
There exists a constant $C>0$ such that for every $a,D\in\R_+$, $b \in \mathbb{R}$, $\eps\in(0,1/2)$, there is a network $\Psi_{a,b,D,\epsilon} \in\cN_{1,1}$ with $\L(\Psi_{a,b,D,\epsilon}) \le  C( (\log(\eps^{-1}))^2 + \log(\lceil a D + |b|\rceil))$, $\mathcal{W}(\Psi_{a,b,D,\epsilon}) \le 9$, $\mathcal{B}(\Psi_{a,b,D,\epsilon})\leq 1$, and satisfying
\begin{equation*}
\|\Psi_{a,b,D,\epsilon}(x)-\cos(ax - b) \|_{L^\infty([-D,D])} \leq \epsilon.
\end{equation*}
\end{corollary}
\begin{proof}
For given $a,D\in\R_+$, $b \in \mathbb{R}$, $\eps\in(0,1/2)$, consider the network $\Psi_{a,b,D,\eps}(x):=\Psi_{a,D+\frac{|b|}{a},\eps}\left(x-\tfrac{b}{a}\right)$ with $\Psi_{a,D,\eps}$ as defined in the proof of Theorem \ref{sin}, and observe that, owing to (\ref{psi-transformed}),
\begin{align*}
    \sup_{x\in[-D,D]}|\Psi_{a,b,D,\eps}(x)-\cos(ax-b)| \le \sup_{y \in \left[-(D+\frac{|b|}{a}),D+\frac{|b|}{a}\right]}|\Psi_{a,D+\frac{|b|}{a},\eps}(y)-\cos(ay)|\leq\eps.
\end{align*}
\end{proof}

\begin{remark}
The results in this section all have approximating networks of finite width and depth scaling polylogarithmically in $\eps^{-1}$. Owing to
$$
\M(\Phi) \leq \L(\Phi) \mathcal{W}(\Phi)(\mathcal{W}(\Phi)+1)
$$
this implies that the connectivity scales no faster than polylogarithmic in $\eps^{-1}$. It therefore follows that the approximation error $\eps$ decays (at least) exponentially fast in the connectivity or equivalently in the number of parameters
the approximant (i.e., the neural network) employs. We say that the network provides exponential approximation accuracy.
\end{remark}

\section{Approximation of Function Classes and Metric Entropy}
\label{sec:approximation_theory}

So far we considered the explicit construction of deep neural networks for the approximation of a wide range of functions, namely polynomials, smooth functions, and sinusoidal functions, in all cases with exponential accuracy, i.e., with an approximation error that decays exponentially in network connectivity. We now proceed to lay the foundation for the development of a framework that allows us to characterize the fundamental limits of deep neural network approximation of entire function classes. But first, we provide a review of relevant literature.

The best-known results on approximation by neural networks are the universal approximation theorems of Hornik \cite{Hornik1991251} and Cybenko \cite{Cybenko1989}, stating that continuous functions on bounded domains can be approximated arbitrarily well by a single-hidden-layer ($L=2$ in our terminology) neural network with sigmoidal activation function. 
The literature on approximation-theoretic properties of networks with a single hidden layer continuing this line of work is abundant. Without any claim to completeness, we mention work on approximation error bounds in terms of the number of neurons for functions with Fourier transforms of bounded first moments
\cite{barron:1993}, \cite{Barron1994}, the nonexistence of localized approximations
\cite{ChuXM1994networksforlocApprox}, a fundamental lower bound on approximation rates \cite{DeVore1997approxfeedforward, candes:1998Ridgelets}, and the approximation of smooth or analytic functions \cite{Mhaskar1996NNapprox,Mhaskar1995151}.

Approximation-theoretic results for networks with multiple hidden layers were obtained in \cite{Hornik1989universalApprox, Mhaskar1993} for general functions, in \cite{Funahashi1989183} for continuous functions, and for functions together with their derivatives in \cite{NguyenThien1999687}.
In \cite{ChuXM1994networksforlocApprox} it was shown that for certain approximation tasks deep networks can perform fundamentally better than single-hidden-layer networks.
We also highlight two recent papers, which investigate the benefit---from
an approximation-theoretic perspective---of multiple hidden layers.
Specifically, in \cite{Eldan2016PowerofDepth} it was shown that there exists a function which, although expressible through a small three-layer network,
can only be represented through a very large two-layer network; here size is measured in terms of the total number of neurons in the network. 

In the setting
of deep convolutional neural networks first results of a nature similar to those in \cite{Eldan2016PowerofDepth} were reported in \cite{Mhaskar2016DeepVSShallow}.
Linking the expressivity properties of neural networks to tensor decompositions, \cite{cohen2016expressive, cohen2016convolutional} established the existence of functions that can be realized by relatively small deep convolutional networks but require exponentially larger shallow convolutional networks.

We conclude by mentioning recent results bearing witness to the approximation power of deep ReLU networks in the
context of PDEs. Specifically, it was shown in \cite{schwab-zech-2018} that
deep ReLU networks can approximate very effectively certain solution families of parametric PDEs depending on a large (possibly infinite) number of parameters. The series of papers \cite{GrohsPDE2018,BernerPDE2018,BeckPDE2018,elbrachter2018dnn} constructs and analyzes a deep-learning-based numerical solver for Black-Scholes PDEs.

For survey articles on approximation-theoretic aspects of neural networks, we refer the interested reader to \cite{ellacott1994aspects} and \cite{pinkus1999approximation} as well as the very recent \cite{devore2020neural}.
Most closely related to the framework we develop here is the paper by Shaham, Cloninger, and Coifman \cite{ShaCC2015provableAppDNN},
which shows that for functions that are sparse in specific wavelet frames, the best $M$-weight approximation rate (see Definition \ref{def:optimalApproximationRateNN} below) of three-layer neural networks is at least as large as the best $M$-term approximation rate in piecewise linear wavelet frames.

We begin the development of our framework with a review of a widely used theoretical foundation for deterministic lossy data compression \cite{DL93,DeVore1998nonlinear}. Our presentation essentially follows \cite{DONOHO1993100,grohs2015optimally}.

\subsection{Kolmogorov-Donoho Rate Distortion Theory} \label{subsec:ratedistorsion}
Let $d\in \N$, $\Omega \subseteq \R^d$, and consider a set of functions $\cC \subseteq L^2(\Omega)$, which we will frequently refer to as \emph{function class}. Then, for each $\ell\in \N$, we denote by
\begin{align*}
    \mathfrak{E}^\ell:= \left\{E: \cC \to \{0,1\}^{\ell}\right\}
\end{align*}
the set of \emph{binary encoders of $\cC$ of length $\ell$}, and we let
\begin{align*}
    \mathfrak{D}^\ell:= \left\{D:\{0,1\}^{\ell} \to  L^2(\Omega)\right\}
\end{align*}
be the set of \emph{binary decoders of length $\ell$}. An encoder-decoder pair  $(E, D) \in \mathfrak{E}^\ell \times \mathfrak{D}^\ell$
is said to {\em achieve uniform error $\varepsilon$ over the function class $\cC$}, if
\begin{align*}
    \sup_{f\in \cC} \|D(E(f)) - f \|_{L^2(\Omega)} \leq \varepsilon.
\end{align*}
Note that here we quantified the approximation error in $L^2(\Omega)$-norm, whereas in the previous section we used the $L^\infty(\Omega)$-norm. 
While results in terms of $L^\infty(\Omega)$-norm are stronger, we shall employ the $L^2(\Omega)$-norm in order to parallel 
the Kolmogorov-Donoho framework for nonlinear approximation through dictionaries \cite{DONOHO1993100,Donoho1996}. We furthermore note that for sets $\Omega$ of finite Lebesgue measure $|\Omega|$, the two norms are related through $\|f\|_{L^2(\Omega)} \leq |\Omega|^{1/2}\|f\|_{L^{\infty}(\Omega)}$. Finally, whenever we talk about compactness and related topological notions, we shall always mean w.r.t.\@ the topology induced by the $L^2(\Omega)$-norm.

A quantity of central interest is the minimal length $\ell\in \N$ for which there exists an encoder-decoder pair $(E, D) \in \mathfrak{E}^\ell \times \mathfrak{D}^\ell$ that achieves uniform error $\varepsilon$ over the function class $\cC$, along with its asymptotic behavior as made precise in the following definition.

\begin{definition}\label{def:optexp}
Let $d\in \N$, $\Omega \subseteq \R^d$, and let $\cC\subseteq L^2(\Omega)$ be compact. Then, for $\varepsilon >0$, the \emph{minimax code length} $L(\varepsilon, \cC)$ is
\begin{align}\label{minimaxDef}
L(\varepsilon, \cC) := \min\left\{\ell\in \N: \exists (E,D) \in  \mathfrak{E}^\ell \times \mathfrak{D}^\ell:  \sup_{f\in \cC} \|D(E(f)) - f \|_{L^2(\Omega)} \leq \varepsilon\right\}.
\end{align}
Moreover, the \emph{optimal exponent} $\gamma^*(\cC)$ is defined as
\[
\gamma^*(\cC): = \sup \left \{\gamma \in \R: L(\varepsilon, \cC) \in \mathcal{O} \! \left(\varepsilon ^{-1/\gamma}\right), \, \varepsilon \rightarrow 0 \right\}.
\]
\end{definition}

The optimal exponent $\gamma^*(\cC)$ determines the minimum growth rate of $L(\varepsilon, \cC)$ as the error $\varepsilon$ tends to zero and
can hence be seen as quantifying the ``description complexity'' of the function class $\cC$. Larger $\gamma^*(\cC)$ results in smaller growth 
rate and hence smaller memory requirements for storing functions $f\in \cC$ such that reconstruction with uniformly bounded error is possible. 

\begin{remark}\label{rem:optexpalt}
The optimal exponent $\gamma^*(\cC)$ can equivalently be thought of as quantifying the asymptotic behavior of the minimal achievable error for the function class $\cC$ with a given code length. Specifically, we have
\begin{align}\label{eq:gamma_star_flipped}
    \gamma^*(\cC) = \sup \big\{\gamma \in \R: L(\varepsilon, \cC) \in \mathcal{O} \big(\varepsilon ^{-1/\gamma}\big), \, \varepsilon \rightarrow 0 \big\} = \sup \big\{\gamma \in \R: \eps(L) \in \mathcal{O} \big(L^{-\gamma}\big), \, L \rightarrow \infty \big\},
\end{align}
where
\begin{align*}
    \eps(L):=\inf_{(E,D)\in\mathfrak{E}^L \times \mathfrak{D}^L}\sup_{f\in\cC}\|D(E(f)) - f \|_{L^2(\Omega)}.
\end{align*}
\end{remark}

The quantity $\gamma^*(\cC)$ is closely related to the concept of Kolmogorov-Tikhomirov epsilon entropy a.k.a. metric entropy \cite{Ott2002MetricEntropy}. 
We next make this connection explicit.

\subsection{Metric entropy}

Most of the discussion in this subsection, which is almost exclusively of review nature, follows very closely \cite[Chapter~5]{Wainwright2019}.
Consider the metric space $({\cal X},\rho)$ with ${\cal X}$ a nonempty set and $\rho: \cal X \times \cal X \rightarrow \R$ a distance function.
A natural measure for the size of a compact subset ${\cal C}$ of ${\cal X}$ is given by the number of balls of a fixed radius $\eps$ 
required to cover ${\cal C}$, a quantity known as the covering number (for covering radius $\eps$).

\begin{definition}\cite{Wainwright2019}
Let $({\cal X},\rho)$ be a metric space. An $\eps$-covering of a compact set ${\cal C} \subseteq {\cal X}$ with respect to the metric $\rho$ is a set $\{x_1,\dots,x_N\} \subseteq \cal C$ such that for each $x \in \cal C$, there exists an $i \in \{1,\dots,N\}$ so that
$\rho(x,x_i) \le \eps$. The $\eps$-covering number $N(\eps; {\cal C},\rho)$ is the cardinality of the smallest $\eps$-covering.
\end{definition}

An $\eps$-covering is a collection of balls of radius $\eps$ that cover the set $\cal C$, i.e.,
$$
{\cal C} \subseteq \bigcup_{i=1}^{N(\eps; \, {\cal C},\rho)}B(x_i,\eps),
$$
where $B(x_i,\eps)$ is a ball---in the metric $\rho$---of radius $\eps$ centered at $x_{i}$. The covering number is nonincreasing in $\eps$, i.e., $N(\eps; {\cal C},\rho) \ge N(\eps'; {\cal C},\rho)$, for all $\eps \le \eps'$. When the set ${\cal C}$ is not finite, the covering number goes to infinity as $\eps$ goes to zero. We shall be interested in the corresponding rate of growth, more specifically in the quantity $\log N(\eps; {\cal C}, \rho)$ known as the metric entropy of ${\cal C}$ with respect to $\rho$. Recall that $\log$ is to the base $2$, hence the unit of metric entropy is ``bits". The operational significance of metric entropy follows from the question: What is the minimum number of bits needed to represent any element $x \in \cal C$ with error---quantified in terms of the distance measure $\rho$---of at most $\epsilon$?
By what was just developed, the answer to this question is $\lceil \log N(\eps; {\cal C}, \rho)\rceil$. Specifically, for a given $x\in{\cal X}$, the corresponding encoder $E(x)$ simply identifies the closest ball center $x_i$ and encodes the index $i$ using $\lceil\log N(\eps; {\cal C}, \rho)\rceil$ bits. The corresponding decoder $D$ delivers the ball center $x_i$, which guarantees that the resulting error satisfies $\|D(E(x))-x\| \le \eps$.

We proceed with a simple example (\cite[Example~5.2]{Wainwright2019}) computing an upper bound on the metric entropy of the interval ${\cal C}=[-1,1]$ in $\R$ with respect to the metric $\rho(x,x')=|x-x'|$. To this end, 
we divide ${\cal C}$ into intervals of length $2\eps$ by setting $x_i=-1+2(i-1)\eps$, for $i\in[1,L]$, where $L=\lfloor \frac{1}{\eps} \rfloor +1$. This guarantees that, 
for every point $x \in [-1,1]$, there is an $i \in [1,L]$ such that $|x-x_i| \le \eps$, which, in turn, establishes
$$
N(\eps; {\cal C}, \rho) \le \Big\lfloor \frac{1}{\eps} \Big \rfloor+1 \le \frac{1}{\eps}+1 
$$
and hence yields an upper bound on metric entropy according to\footnote{The notation $f(\eps)\asymp g(\eps)$, as $\eps \rightarrow 0$, means that there are constants $c,C,\eps_0>0$ such that $cf(\eps)\leq g(\eps)\leq Cf(\eps)$, for all $\eps\leq\eps_0$. For ease of exposition, we shall usually omit the qualifier $\eps \rightarrow 0$.} 
\begin{equation}
\log N(\eps; {\cal C}, \rho) \le \log \left(\frac{1}{\eps}+1 \right) \asymp \log (\eps^{-1}), \quad \text{as}\,\, \eps \rightarrow 0.   \label{upper-unit-cube}
\end{equation}
This result can be generalized to the $d$-dimensional unit cube to yield $\log(N(\eps; {\cal C}, \rho)) \le d\log(1/\eps +1) \asymp d\log(\eps^{-1})$. In order to show that the 
upper bound (\ref{upper-unit-cube}) correctly reflects metric entropy scaling for ${\cal C}=[-1,1]$ with respect to $\rho(x,x')=|x-x'|$, we would need a lower bound on $N(\eps; {\cal C}, \rho)$ that exhibits the same scaling (in $\eps$) behavior. A systematic approach to establishing lower bounds on metric entropy is through the concept of packing, which will be introduced next.

We start with the definition of the packing number of a compact set ${\cal C}$ in a metric space $({\cal X},\rho)$.

\begin{definition}\cite[Definition~5.4]{Wainwright2019}
Let $({\cal X},\rho)$ be a metric space.  An $\eps$-packing of a compact set ${\cal C} \subseteq {\cal X}$ with respect to the metric $\rho$ is a set $\{x_1,\dots,x_N\} \subseteq \cal C$ such that 
$\rho(x_i,x_j) > \eps$, for all distinct $i,j$. The $\eps$-packing number $M(\eps; {\cal X},\rho)$ is the cardinality of the largest $\eps$-packing.
\end{definition}

An $\eps$-packing is a collection of nonintersecting balls of radius $\eps/2$ and centered at elements in ${\cal X}$.
Although different, the covering number and the packing number provide essentially the same measure of size of a set as formalized next.

\begin{lemma}\label{lem:covering-packing}\cite[Lemma~5.5]{Wainwright2019}
Let $({\cal X},\rho)$ be a metric space and ${\cal C}$ a compact set in ${\cal X}$. For all $\eps > 0$, the packing and the covering number are related according to
$$
M(2\eps; {\cal C},\rho) \le N(\eps; {\cal 
C},\rho) \le M(\eps; {\cal C},\rho).
$$
\end{lemma}

\begin{proof}\cite{Wainwright2019,Prosser1966}
First, choose a minimal $\eps$-covering and a maximal $2\eps$-packing of ${\cal C}$. Since no two centers of the $2\eps$-packing can lie in the same ball of the
$\eps$-covering, it follows that $M(2\eps; {\cal C},\rho) \le N(\eps; {\cal C},\rho)$. To establish $N(\eps; {\cal C},\rho) \le M(\eps; {\cal C},\rho)$, we note that, given a maximal
packing $M(\eps; {\cal C},\rho)$, for any $x\in {\cal C}$, we have the center of at least one of the balls in the packing within distance less than $\eps$. If this were not the case, we could
add another ball to the packing thereby violating its maximality. This maximal packing hence also provides an $\eps$-covering and since $N(\eps; {\cal C},\rho)$ is a minimal covering, we must have $N(\eps; {\cal C},\rho) \le M(\eps; {\cal C},\rho)$.
\end{proof}

We now return to the example in which we computed an upper bound on the metric entropy of ${\cal C}=[-1,1]$ with respect to $\rho(x,x')=|x-x'|$ and show how Lemma \ref{lem:covering-packing} can be employed to establish the scaling behavior of metric entropy. To this end, we simply note that the points $x_i = -1+2(i-1)\eps, \,i \in [1,L]$, are separated according to $|x_i-x_j| = 2\eps > \eps$, for all $i \neq j$, which implies that $M(\eps; {\cal C},| \cdot |) \ge L = \lfloor 1/\eps \rfloor +1 \ge \frac{1}{\eps}$. Combining this with the upper bound (\ref{upper-unit-cube}) and Lemma \ref{lem:covering-packing}, we obtain 
$\log N(\eps; {\cal C},| \cdot |) \asymp \log(\eps^{-1})$. Likewise, it can be established that $\log N(\eps; {\cal C},\| \cdot \|) \asymp d \log(\eps^{-1})$ for the $d$-dimensional unit cube. This illustrates how an explicit construction of a packing set can be used to determine the scaling behavior of metric entropy.

We next formalize the notion that metric entropy is determined by the volume of the corresponding covering balls. Specifically, the following result establishes a relationship between a certain volume ratio and metric entropy.

\begin{lemma}\label{lem:volumes}\cite[Lemma~5.7]{Wainwright2019}
Consider a pair of norms $\| \cdot \|$ and $\| \cdot \|^{\prime}$ on $\R^d$, and let $\cal B$ and ${\cal B}'$ be their corresponding unit balls, i.e., ${\cal B}=\{ x \in \R^d | \|x\| \le 1\}$ and
${\cal B'}=\{ x \in \R^d | \|x\|' \le 1\}$. Then, the $\eps$-covering number of ${\cal B}$ in the $\| \cdot \|'$-norm satisfies
\begin{equation}
\left(\frac{1}{\eps}\right)^{d} \frac{ {\normalfont vol}({\cal B})}{{\normalfont vol}({\cal B}')} \le N(\eps; {\cal B},\| \cdot \|') \le \frac{{\normalfont vol}(\frac{2}{\eps}{\cal B}+{\cal B}')}
{{\normalfont vol}({\cal B}')}.\label{vol-ineq}
\end{equation}
\end{lemma}

\begin{proof}\cite{Wainwright2019}
Let $\{x_1,\dots,x_{N(\eps; {\cal B},\| \cdot \|')}\}$ be an $\eps$-covering of $\cal B$ in $\|\cdot\|'$-norm. Then, we have
$$
{\cal B} \subseteq \bigcup_{j=1}^{N(\eps; {\cal B},\| \cdot \|')}\{x_j+\eps {\cal B}'\},
$$
which implies ${\normalfont vol}({\cal B}) \le N(\eps; {\cal B},\| \cdot \|')\, \eps^{d} \, {\normalfont vol}({\cal B}')$, thus establishing the lower bound in (\ref{vol-ineq}). The upper bound is obtained by starting
with a maximal $\eps$-packing $\{x_1,\dots,x_{M(\eps; {\cal B},\| \cdot \|')}\}$ of ${\cal B}$ in the $\| \cdot \|'$-norm. 
The balls $\{x_j+\frac{\eps}{2}{\cal B}', j=1,\dots,M(\eps; {\cal B},\| \cdot \|')\}$ are all disjoint and contained within ${\cal B}+\frac{\eps}{2}{\cal B}'$. We can therefore conclude that 
$$
\sum_{j=1}^{M(\eps; {\cal B},\| \cdot \|')}{\normalfont vol}\!\left(x_j+\frac{\eps}{2}{\cal B}'\right) \le {\normalfont vol}\!\left({\cal B}+\frac{\eps}{2}{\cal B}'\right),
$$
and hence
$$
M(\eps; {\cal B},\| \cdot \|')\, {\normalfont vol}\!\left(\frac{\eps}{2}{\cal B}'\right) \le {\normalfont vol}\!\left({\cal B}+\frac{\eps}{2}{\cal B}'\right).
$$
Finally, we have ${\normalfont vol}(\frac{\eps}{2}{\cal B}')=(\frac{\eps}{2})^d {\normalfont vol}({\cal B}')$ and ${\normalfont vol}({\cal B}+\frac{\eps}{2}{\cal B}')=
(\frac{\eps}{2})^d {\normalfont vol}(\frac{2}{\eps}{\cal B}+{\cal B}')$,
which, together with $M(\eps; {\cal B},{\| \cdot \|'})\allowbreak \ge N(\eps; {\cal B},\| \cdot \|')$ due to Lemma \ref{lem:covering-packing}, yields the upper bound in (\ref{vol-ineq}).
\end{proof}

This result now allows us to establish the scaling of the metric entropy of unit balls in terms of their own norm, thus yielding a measure of the massiveness of unit balls in $d$-dimensional spaces. Specifically, we set ${\cal B}'={\cal B}$ in Lemma \ref{lem:volumes} and get
$$
{\normalfont vol}\left(\frac{2}{\eps}{\cal B}+{\cal B}'\right)={\normalfont vol}\left(\left(\frac{2}{\eps}+1\right){\cal B}\right)=\left(\frac{2}{\eps}+1\right)^{d}{\normalfont vol}({\cal B}),
$$
which when used in (\ref{vol-ineq}) yields $N(\eps; {\cal B},\| \cdot \|) \asymp \eps^{-d}$ and hence results in metric entropy scaling according to $\log(N(\eps; {\cal B},\| \cdot \|)) 
\asymp d\log(\eps^{-1})$. Particularizing this result to the unit ball ${\cal B}_{\infty}^{d}=[-1,1]^{d}$ and the metric $\| \cdot \|_{\infty}$, we recover the result of our direct analysis in the example above.

So far we have been concerned with the metric entropy of subsets of $\R^d$. We now proceed to analyzing the metric entropy of function classes, which will eventually allow us to establish the desired connection between the optimal exponent $\gamma^*({\cal C})$ and metric entropy. We begin with the simple one-parameter function class considered in \cite[Example~5.9]{Wainwright2019} and follow closely the exposition in \cite{Wainwright2019}.
For a fixed $\theta$, define the real-valued function $f_{\theta}(x)=1-e^{-\theta x}$, and consider the class
$$
{\cal P}=\{f_{\theta}:[0,1] \rightarrow \R\, |\, \theta \in [0,1]\}.
$$
The set ${\cal P}$ constitutes a metric space under the sup-norm given by $\|f-g\|_{L^{\infty}([0,1])}=\sup_{x\in[0,1]}|f(x)-g(x)|$. We show that the covering number of ${\cal P}$ satisfies
$$
1+\left\lfloor \frac{1-1/e}{2\eps} \right\rfloor \le N(\eps; {\cal P},\| \cdot \|_{L^{\infty}([0,1])}) \le \frac{1}{2\eps}+2,
$$
which leads to the scaling behavior $N(\eps; {\cal P},\| \cdot \|_{L^{\infty}([0,1])}) \asymp \eps^{-1}$ and hence to metric entropy scaling according to $\log(N(\eps; {\cal P},\| \cdot \|_{L^{\infty}([0,1])})) \asymp \log(\eps^{-1})$. We start by establishing the upper bound. For given $\eps \in [0,1]$, set $T=\lfloor \frac{1}{2\eps} \rfloor$, and define the points $\theta_{i}=2\eps i$, for $i=0,1,\dots,T$. By also adding the point $\theta_{T+1}=1$, we obtain a collection of $T+2$ points $\{\theta_{0},\theta_{1},\dots,\theta_{T+1}\}$ in $[0,1]$. We show that the associated functions $\{f_{\theta_{0}}, f_{\theta_{1}},\dots,f_{\theta_{T+1}}\}$ form an $\eps$-covering for ${\cal P}$. Indeed, for any $f_{\theta} \in {\cal P}$, we can find some $\theta_{i}$ in the covering such that $|\theta - \theta_{i}| \le \eps$. We then have 
$$
\|f_{\theta}-f_{\theta_{i}}\|_{L^{\infty}([0,1])}=\max_{x\in [0,1]}|e^{-\theta x}-e^{-\theta_{i} x}| \le |\theta - \theta_{i}|,
$$
where we used, for $\theta < \theta_{i}$,
\begin{eqnarray*}
\max_{x\in [0,1]}|e^{-\theta x}-e^{-\theta_{i} x}| & = & \max_{x\in [0,1]}(e^{-\theta x}-e^{-\theta_{i} x})=\max_{x\in [0,1]}e^{-\theta x}(1-e^{-(\theta_{i}-\theta) x})\,\le\,\max_{x\in [0,1]}(1-e^{-(\theta_{i}-\theta) x})\\
& \le & \max_{x\in [0,1]} (\theta_{i}-\theta)x \le \theta_{i}-\theta=|\theta-\theta_{i}|,
\end{eqnarray*}
as a consequence of $1-e^{-x} \le x$, for $x \in [0,1]$, which is easily verified by noting that the function $g(x)=1-e^{-x}-x$ satisfies
$g(0)=0$ and $g'(x) \le 0$, for $x \in [0,1]$. The case $\theta > \theta_{i}$ follows similarly. In summary, we have shown that $N(\eps; {\cal P},\| \cdot \|_{L^{\infty}([0,1])}) \le T+2 \le \frac{1}{2\eps}+2$.

In order to derive the lower bound, we first bound the packing number from below and then use Lemma \ref{lem:covering-packing}. We start by constructing an explicit packing as follows. 
Set $\theta_0=0$ and define
$\theta_i=-\log(1-\eps i)$, for all $i$ such that $\theta_i \le 1$. The largest index $T$ such that this holds is given by
$T=\lfloor \frac{1-1/e}{\eps} \rfloor$. Moreover, note that for all $i,j$ with $i \neq j$,
we have 
$\|f_{\theta_{i}}-f_{\theta_{j}}\|_{L^{\infty}([0,1])} \ge |f_{\theta_{i}}(1) - 
f_{\theta_{j}}(1)|=|\eps(i-j)| \ge \eps$. We can therefore conclude that $M(\eps; {\cal P},\| \cdot \|_{L^{\infty}([0,1])}) \ge \lfloor \frac{1-1/e}{\eps} \rfloor +1$, and hence, due to the lower bound in Lemma \ref{lem:covering-packing},
$$
N(\eps; {\cal P},\| \cdot \|_{L^{\infty}([0,1])}) \ge M(2\eps; {\cal P},\| \cdot \|_{L^{\infty}([0,1])}) \ge \left \lfloor \frac{1-1/e}{2\eps} \right \rfloor +1,
$$
as claimed. We have thus established that the function class ${\cal P}$ has metric entropy scaling according
to $$\log(N(\eps; {\cal P},\| \cdot \|_{L^{\infty}([0,1])})) \asymp \log(1/\eps),\, \mathrm{as}\  
\eps \rightarrow 0.$$
This rate is typical for one-parameter function classes.

We now turn our attention to richer function classes and start by considering 
Lipschitz functions on the $d$-dimensional unit cube, meaning real-valued functions on $[0,1]^d$ such that
$$
|f(x)-f(y)| \le L \|x-y\|_{\infty}, \qquad \mbox{for all}\quad x,y \in [0,1]^d.
$$
This class, denoted as ${\cal F}_{L}([0,1]^d)$, has metric entropy scaling \cite{Kolmogorov1959,Wainwright2019}
\begin{equation}
\log N(\eps; {\cal F}_{L},\| \cdot \|_{L^{\infty}([0,1]^{d})}) \asymp (L/\eps)^d. \label{d-dim-lip}
\end{equation}
Contrasting the exponential dependence of metric entropy in (\ref{d-dim-lip}) on the ambient dimension $d$ to the 
linear dependence we identified earlier for simpler sets such as unit balls in $\R^d$, where we had 
$$
\log N(\eps; {\cal B},\| \cdot \|_{\infty}) \asymp d \log(\eps^{-1}),
$$
shows that ${\cal F}_{L}([0,1]^d)$ is significantly more massive.

We are now ready to relate the optimal exponent $\gamma^{\ast}({\cal C})$ in Definition~\ref{def:optexp} to metric entropy scaling. All the examples
of metric entropy scaling we have seen exhibit a behavior that fits the law $\log(N(\eps; {\cal C},\| \cdot \|)) \asymp \eps^{-1/\gamma}$ or
$\log(N(\eps; {\cal C},\| \cdot \|)) \asymp \eps^{-1/\gamma} \log(\eps^{-1})^{\beta}$. The optimal exponent is hence a crude measure of growth insensitive to $\log$-factors or similar factors that are dominated by the growth of $\eps^{-1/\gamma}$.

While we restrict ourselves to the approximation of functions on Euclidean domains, the framework described in this section can be extended to functions on manifolds (see e.g.~\cite{EHLER201841}). As such, an interesting direction for future research would be the extension of the deep neural network approximation theory developed in this paper to functions on manifolds. First results on the neural network approximation of functions on manifolds have been reported in \cite{ShaCC2015provableAppDNN, boelcskei:2017DNN, schmidthieber2019deep}. For further reading on the general subject of function approximation on manifolds, we recommend \cite{MHASKAR2020253} and references therein.

\section{Approximation with Dictionaries}\label{sec:approx-rep-systems}

We now show how Kolmogorov-Donoho rate-distortion theory can be put to work in the context of optimal approximation with dictionaries. Again, this subsection is of review nature. We start with a brief discussion of basics on optimal approximation in Hilbert spaces. Specifically, we shall consider two types of approximation, namely linear
and nonlinear. 

Let ${\cal H}$ be a Hilbert space equipped with inner product $\langle \cdot, \cdot \rangle$ and induced norm $\|\cdot\|_{\cal H}$ and let $e_{k},\,k=1,2,\dots$, be an
orthonormal basis for ${\cal H}$. For linear approximation, we use the linear space ${\cal H}_{M}:= \mbox{span}\{e_{k}: 1\le k \le M\}$ to approximate a given element
$f \in {\cal H}$. We measure the approximation error by 
$$
E_{M}(f):= \inf_{g \in {\cal H}_{M}} \|f-g\|_{\cal H}.
$$
In nonlinear approximation, we consider best $M$-term approximation, which replaces ${\cal H}_{M}$ by the set $\Sigma_{M}$ consisting of all elements $g \in {\cal H}$ that can be 
expressed as
$$
g=\sum_{k \in \Lambda}c_{k}e_{k},
$$
where $\Lambda \subseteq \N$ is a set of indices with $|\Lambda| \le M$. Note that, in contrast to ${\cal H}_{M}$, the set $\Sigma_{M}$ is not a linear space as a linear combination of two elements in $\Sigma_{M}$ will, in general, need $2M$ terms in its representation by the $e_{k}$. Analogous to $E_{M}$, we define the error of best $M$-term approximation
$$
\Gamma_{M}(f):= \inf_{g \in \Sigma_{M}}\|f-g\|_{\cal H}.
$$
The key difference between linear and nonlinear approximation resides in the fact that in nonlinear approximation, we can choose the $M$ elements $e_k$ participating in the approximation of $f$ freely from the entire orthonormal basis whereas in linear approximation we are constrained to the first $M$ elements. A classical example for linear approximation is the approximation of periodic functions by the Fourier series elements corresponding to the $M$ lowest frequencies (assuming natural ordering of the dictionary). This approach clearly leads to poor approximation if the function under consideration consists of high-frequency components. In contrast, in nonlinear approximation we would seek the $M$ frequencies that yield the smallest approximation error. In summary, it is clear that (nonlinear) best $M$-term approximation can achieve smaller approximation error than linear $M$-term approximation.

We shall consider nonlinear approximation in arbitrary, possibly redundant, dictionaries, i.e., in frames \cite{Morgenshtern-Boelcskei-2012}, and
will exclusively be interested in the case $\mathcal{H}=L^2(\Omega)$, in particular the approximation error will be measured in 
terms of $L^2(\Omega)$-norm. 
Specifically, let $\cC$ be a set of functions in $L^2(\Omega)$ and consider a countable family of functions $\mathcal{D}:=(\varphi_i)_{i \in \N} \subseteq L^2(\Omega)$, termed \emph{dictionary}. 

We consider the {\em best $M$-term approximation error} 
of $f \in \cC$ in $\mathcal{D}$ defined as follows.

\begin{definition}\cite{DL93} \label{def:optimalApproximationRate}
Given $d\in \N$, $\Omega \subseteq \R^d$, a function class $\cC \subseteq L^2(\Omega)$, and a dictionary $\mathcal{D} = (\varphi_i)_{i \in \N} \subseteq L^2(\Omega)$, we define, for $f \in \cC$ and $M\in \N$,
\begin{align} \label{eq:GammaMDictDef}
\Gamma_M^\mathcal{D}(f) := \inf_{\substack{I_{f,M} \, \subseteq \, \N,\\ |I_{f,M}| = M, (c_i)_{i \in I_{f,M}}}} \left\|f - \sum_{i \in I_{f,M}} c_i \varphi_i\right\|_{L^2(\Omega)}.
\end{align}
We call $\Gamma_M^\mathcal{D}(f)$ the {\em best $M$-term approximation error of $f$ in $\mathcal{D}$}.
Every $f_M = \sum_{i \in I_{f,M}} c_i \varphi_i$ attaining the infimum in \eqref{eq:GammaMDictDef} is referred to as a {\em best $M$-term approximation} of $f$ in $\mathcal{D}$.
The supremal $\gamma > 0$ such that 
\[
\sup_{f \in \cC}\Gamma_M^\mathcal{D}(f) \in \mathcal{O}(M^{-\gamma}), \,\, M \rightarrow \infty,
\]
will be denoted by $\gamma^\ast(\mathcal{C},\mathcal{D})$. We say that the  {\em best $M$-term approximation rate of $\cC$ in the dictionary $\mathcal{D}$} is $\gamma^\ast(\mathcal{C},\mathcal{D})$.
\end{definition}

Function classes $\mathcal{C}$ widely studied in the approximation theory literature include unit balls in Lebesgue, Sobolev, or Besov spaces \cite{DeVore1998nonlinear}, as well as $\alpha$-cartoon-like functions \cite{GroKKS2016alphaMolecules}. A wealth of structured dictionaries $\mathcal{D}$ is provided by the area of applied harmonic analysis, starting with wavelets \cite{Dau92}, followed by ridgelets \cite{candes:1998Ridgelets}, curvelets \cite{CD02}, shearlets \cite{GKL06}, parabolic molecules \cite{GK14},
and most generally $\alpha$-molecules \cite{GroKKS2016alphaMolecules}, which include all previously named dictionaries as special cases. Further examples are Gabor frames \cite{grochenig2013foundations}, Wilson bases \cite{grochenig:2000GaborApprox}, and wave atoms \cite{demanet2007wave}. 

The best $M$-term approximation rate $\gamma^\ast(\mathcal{C},\mathcal{D})$ according to Definition~\ref{def:optimalApproximationRate} quantifies how difficult it is to approximate a given function class $\mathcal{C}$ in a fixed dictionary $\mathcal{D}$. It is sensible to ask
whether for given $\cC$, there is a fundamental limit on $\gamma^\ast(\mathcal{C},\mathcal{D})$ when one is allowed to vary over $\mathcal{D}$.
To answer this question, we first note that for every dense (and countable) $\mathcal{D}$, for any given $f \in \mathcal{C}$, by density of $\mathcal{D}$, there exists a 
single dictionary element that approximates $f$ to within arbitrary accuracy thereby effectively realizing a $1$-term approximation for arbitrary approximation error $\eps$. Formally, this can be expressed through 
$\gamma^\ast(\mathcal{C},\mathcal{D}) = \infty$.
Identifying this single dictionary element or, more generally, the $M$ 
elements participating in the best $M$-term approximation is in general, however, practically infeasible as it entails searching through the infinite set $\mathcal{D}$ and requires an infinite number of bits to describe the indices of the participating elements.
This insight leads to the concept of ``best $M$-term approximation subject to polynomial-depth search'' as introduced by Donoho in \cite{Donoho1996}. Here, the basic idea
is to restrict the search for the elements in $\mathcal{D}$ participating in the best $M$-term approximation to the first $\pi(M)$ elements of $\mathcal{D}$, with $\pi$ a polynomial. 
We formalize this under the name of effective best $M$-term approximation as follows.

\begin{definition}\label{def:polydepth}
Let $d\in \N$, $\Omega \subseteq \R^d$, $\cC \subseteq L^2(\Omega)$ be compact, and $\mathcal{D} = (\varphi_i)_{i \in \N} \subseteq L^2(\Omega)$.
We define for $M\in\N$ and $\pi$ a polynomial
\begin{align}\label{eq:gamma-eff-def}
    \eps^{\pi}_{\cC,\mathcal{D}}(M):=\sup_{f \in \mathcal{C}} \inf_{\substack{I_{f,M}\subseteq \{ 1,2,\dots, \pi(M) \}, \\  |I_{f,M}| = M, \, |c_i|\leq\pi(M)}} \left\|f - \sum_{i \in I_{f,M}} c_i \varphi_i\right\|_{L^2(\Omega)}
\end{align}
and
\begin{align} \label{eq:GammaMDictDefeff}
\gamma^{\ast,\text{eff}}(\mathcal{C},\mathcal{D}):=\sup\{\gamma\geq 0\colon \exists\  \mathrm{polynomial}\ \pi\ \mathrm{s.t.}\ \eps^{\pi}_{\cC,\mathcal{D}}(M) \in \mathcal{O}(M^{-\gamma}), \,\, M \rightarrow \infty\}.
\end{align}
We refer to $\gamma^{\ast,\text{eff}}(\mathcal{C},\mathcal{D})$ as the
{\em effective best $M$-term approximation rate of $\cC$ in the dictionary $\mathcal{D}$}.
\end{definition}

Note that we required the coefficients $c_i$ in the approximant in Definition \ref{def:polydepth} to be polynomially 
bounded in $M$. This condition, not present in \cite{DONOHO1993100,grohs2015optimally} and easily met for generic $\mathcal{C}$ and $\mathcal{D}$, 
is imposed for technical reasons underlying the transference results in Section \ref{sec:bestapprox}.
Strictly speaking---relative to \cite{DONOHO1993100,grohs2015optimally}---we hence get a subtly different notion 
of approximation rate. Exploring the implications of this difference is certainly worthwhile, but deemed beyond the scope of this paper.

We next present a central result in best $M$-term approximation theory stating that
for compact $\cC \subseteq L^2(\Omega)$, the effective best $M$-term approximation rate in any dictionary $\mathcal{D}$ 
is upper-bounded by $\gamma^*(\cC)$ and hence limited by the ``description complexity" of $\cC$. This endows $\gamma^*(\cC)$ with operational meaning.  
\begin{theorem}\cite{DONOHO1993100,grohs2015optimally}\label{thm:optDictApproxLwrBd}
    Let $d\in \N$, $\Omega\subseteq \mathbb{R}^d$, and let $\cC\subseteq L^2(\Omega)$ be compact. The effective best $M$-term approximation rate of the function class $\cC \subseteq L^2(\Omega)$ in the dictionary $\mathcal{D} = (\varphi_i)_{i \in \N} \subseteq L^2(\Omega)$
    satisfies
    $$
	    \gamma^{\ast,\text{eff}}(\cC,\mathcal{D}) \leq {\gamma^\ast(\cC)}.
    $$
\end{theorem}

In light of this result the following definition is natural (see also \cite{grohs2015optimally}).

\begin{definition}(Kolmogorov-Donoho optimality)\label{def:repopti}
Let $d\in \N$, $\Omega\subseteq \mathbb{R}^d$, and let $\cC\subseteq L^2(\Omega)$ be compact. If the effective best $M$-term approximation rate of the function class $\cC \subseteq L^2(\Omega)$ in the dictionary $\mathcal{D} = (\varphi_i)_{i \in \N} \subseteq L^2(\Omega)$ 
satisfies
$$
\gamma^{\ast,\text{eff}}(\cC,\mathcal{D}) = {\gamma^\ast(\cC)},
$$
we say that the function class $\cC$ is \emph{optimally representable}\/ by $\mathcal{D}$.
\end{definition}

As the ideas underlying the proof of Theorem~\ref{thm:optDictApproxLwrBd} are essential ingredients in the development of a kindred theory of
best $M$-weight approximation rates for neural networks, we present a detailed proof, which is similar to that in
\cite{grohs2015optimally}. We perform, however, some minor technical modifications with an eye towards rendering the proof a suitable genesis
for the new
theory of best $M$-weight approximation with neural networks, developed in the next section. The spirit of the proof is to construct, for every given $M\in\N$ an encoder
that, for each $f\in\cC$, maps the indices of the dictionary elements participating in the effective best $M$-term approximation\footnote{Note that as we have an infimum in \eqref{eq:gamma-eff-def} an effective best $M$-term approximation need not exist, but we can pick an $M$-term approximation that yields an error arbitrarily close to the infimum.} of $f$, along with the corresponding coefficients $c_i$, to a bitstring. This bitstring needs to be of sufficient length for the decoder to be able to reconstruct an approximation to $f$ with an error which is of the same order as that of the best $M$-term approximation we started from. As elucidated in the proof, this can be accomplished while ensuring that the length of the bitstring is proportional to $M\log(M)$, which upon noting that $\eps=M^{-\gamma}$ implies $M=\eps^{-1/\gamma}$, establishes optimality.

\begin{proof}[Proof of \Cref{thm:optDictApproxLwrBd}]
The proof will be based on showing that for every $\gamma\in\R_+$ the following Implication (I) holds: Assume that there exist a constant $C>0$ and a polynomial $\pi$ such that for every $M\in\N$, the following holds:
For every $f\in\cC$, there are an index set $I_{f,M}\subseteq\{1,2,\dots,\pi(M)\}$ and coefficients $(c_i)_{i \in I_{f,M}}\subseteq\R$ with $|c_i|\le \pi(M)$ so that
\begin{align}\label{eq:good_Mterm_approx}
    \big\| f - \sum_{i \in I_{f,M}} c_i \varphi_i \big \|_{L^2(\Omega)} \le CM^{-\gamma}.
\end{align}
This implies the existence of a constant $C'>0$ such that for every $M\in\N$, there is an encoder-decoder pair $(E_M,D_M)\in\mathfrak{E}^{\ell(M)} \times \mathfrak{D}^{\ell(M)}$ with $\ell(M)\leq C' M\log(M)$ and 
\begin{align}\label{eq:ED_claim}
    \| f - D_M(E_M(f))\|_{L^2(\Omega)} \le C'M^{-\gamma}.
\end{align}
The implication will be proven by explicit construction. For a given $f\in\cC$, we pick an $M$-term approximation according to \eqref{eq:good_Mterm_approx} and 
encode the associated index set $I_{f,M}$ and weights $c_i$ as follows. First, note that owing to $|I_{f,M}| \le \pi(M)$, 
each index in $I_{f,M}$ can be represented by at most $C_{\pi}\log(M)$ bits; this results in a total of $C_{\pi}M\log(M)$ bits needed to encode the indices of all dictionary elements participating in the $M$-term approximation. The encoder and the decoder are assumed to know $C_{\pi}$, which allows stacking of  
the binary representations of the indices such that the decoder can read them off uniquely from the sequence of their binary representations.

We proceed to the encoding of the coefficients $c_i$. First, note that even though the $c_i$ are bounded (namely, polynomially in $M$) by assumption, we did not impose bounds on the norms of the dictionary elements $\{\varphi_{i}\}_{i \in I_{f,M}}$ participating in the $M$-term approximation under consideration. Hence, we can not, in general, expect
to be able to control the approximation error incurred by reconstructing $f$ from quantized $c_i$. We can get around this by performing a Gram-Schmidt orthogonalization on the dictionary elements $\{\varphi_{i}\}_{i \in I_{f,M}}$ and, as will be seen later, using the fact that the function class $\cC$ was assumed to be compact. Specifically, this Gram-Schmidt orthogonalization yields a set of functions $\{\tilde{\varphi}_{i}\}_{i \in \tilde{I}_{f,\widetilde{M}}}$, with $\widetilde{M} \le M$, that has
the same span as $\{\varphi_{i}\}_{i \in I_{f,M}}$. 
Next, we define (implicitly) the coefficients $\tilde{c}_{i}$ according to
\begin{equation}
\sum_{i \in \tilde{I}_{f,\widetilde{M}}}\tilde{c}_{i} \tilde{\varphi}_{i} = \sum_{i \in I_{f,M}} c_{i}\varphi_{i}. \label{gram-schmidt}
\end{equation}
Now, note that
\begin{equation*}
\left\|\sum_{i \in \tilde{I}_{f,\widetilde{M}}}\tilde{c}_{i} \tilde{\varphi}_{i}\right\|_{L^2(\Omega)}^{2}=\left\|f-(f-\sum_{i \in \tilde{I}_{f,\widetilde{M}}}\tilde{c}_{i} \tilde{\varphi}_{i})\right\|_{L^2(\Omega)}^{2} \le \|f\|_{L^2(\Omega)}^{2}+\left\|f-\sum_{i \in I_{f,M}}c_{i} \varphi_{i}\right\|_{L^2(\Omega)}^{2}.
\end{equation*}
Making use of the orthonormality of the $\tilde{\varphi}_{i}$, we can conclude that
\begin{align*}
\sum_{i \in \tilde{I}_{f,\widetilde{M}}} |\tilde{c}_{i}|^{2} \le \sup_{f \in {\cal C}}\|f\|_{L^2(\Omega)}^{2}+C^{2}M^{-2\gamma}.
\end{align*}
As ${\cal C}$ is compact by assumption, we have $\sup_{f \in {\cal C}}\|f\|_{L^2(\Omega)}^{2}\, < \, \infty$, which establishes that
the coefficients $\tilde{c}_{i}$ are uniformly bounded. This, in turn, allows us to quantize them, specifically, we shall
round the $\tilde{c}_{i}$ to integer multiples of $M^{-(\gamma+1/2)}$, and denote the resulting rounded coefficients by $\hat{c}_{i}$. As the $\tilde{c}_{i}$ are uniformly bounded, this results in a number of quantization levels that is proportional to $M^{(\gamma+1/2)}$.
The number of bits needed to store the binary representations of the quantized coefficients is therefore proportional to $M\log(M)$. Again, the proportionality constant is assumed known to encoder and decoder, which allows us to stack the binary representations of the quantized coefficients
in a uniquely
decodable manner. The resulting bitstring is then appended to the bitstring encoding the indices of the participating dictionary elements. We finally note that the specific choice of the exponent $\gamma+1/2$ is informed by the upper bound on the reconstruction error
we are allowed, this will be made explicit below in the description of the decoder.

In summary, we have mapped the function $f$ to a bitstring of length $\mathcal{O}(M\log(M))$. The decoder is presented with this bitstring and reconstructs an
approximation to $f$ as follows. It first reads out the indices of the set $I_{f,M}$ and the quantized coefficients $\hat{c}_{i}$. Recall that this is uniquely possible. Next, the decoder performs
a Gram-Schmidt orthonormalization on the set of dictionary elements indexed by $I_{f,M}$. The error resulting from reconstructing the function $f$ from the
quantized coefficients $\hat{c}_{i}$ rather than the exact coefficients $\tilde{c}_{i}$ can be bounded according to
\begin{align}\begin{split}\label{eq:DE_rec_error}
\left\| f - \sum_{i \in \tilde{I}_{f,\widetilde{M}}} \hat{c}_i \tilde{\varphi}_i \right \|_{L^2(\Omega)} & = \left\| f - \sum_{i \in \tilde{I}_{f,\widetilde{M}}} \tilde{c}_i \tilde{\varphi}_i + \sum_{i \in \tilde{I}_{f,\widetilde{M}}} \tilde{c}_i \tilde{\varphi}_i - \sum_{i \in \tilde{I}_{f,\widetilde{M}}} \hat{c}_i \tilde{\varphi}_i \right \|_{L^2(\Omega)}\\
& \le \left\| f - \sum_{i \in \tilde{I}_{f,\widetilde{M}}} \tilde{c}_i \tilde{\varphi}_i \right \|_{L^2(\Omega)} + \left\| \sum_{i \in \tilde{I}_{f,\widetilde{M}}} (\tilde{c}_i-\hat{c}_{i}) \tilde{\varphi}_i \right \|_{L^2(\Omega)}\\
& = \left\| f - \sum_{i \in \tilde{I}_{f,\widetilde{M}}} \tilde{c}_i \tilde{\varphi}_i \right \|_{L^2(\Omega)} + \left(\sum_{i \in \tilde{I}_{f,\widetilde{M}}} |\tilde{c}_i-\hat{c}_{i}|^{2}\right)^{1/2},
\end{split}\end{align}
where in the last step we again exploited the orthonormality of the $\tilde{\varphi}_i$. Next, note that due to the choice of the quantizer resolution, we have
$|\tilde{c}_i-\hat{c}_{i}|^{2} \le C'' M^{-2\gamma-1}$ for some constant $C''$. With $\widetilde{M} \le M$ this yields 
$$
\sum_{i \in \tilde{I}_{f,\widetilde{M}}} |\tilde{c}_i-\hat{c}_{i}|^{2} \le C'' M^{-2\gamma}.
$$
Combining \eqref{eq:good_Mterm_approx}, \eqref{gram-schmidt}, and \eqref{eq:DE_rec_error}, we obtain
\begin{align*}
    \left\| f - \sum_{i \in \tilde{I}_{f,\widetilde{M}}} \hat{c}_i \tilde{\varphi}_i \right \|_{L^2(\Omega)} \le C'M^{-\gamma},
\end{align*}
for some constant $C'$. As the length of the bitstring used in this construction is proportional to $M\log(M)$, the claim \eqref{eq:ED_claim} is established.

Now, we note that the antecedent of Implication (I) holds for all $\gamma < \gamma^{\ast,\text{eff}}(\cC,\mathcal{D})$. 
Assume next, towards a contradiction, that the antecedent holds for a $\gamma>\gamma^*(\cC)$.
This would imply that for any $\gamma'<\gamma$,
\begin{align}\label{eq:encoder-decoder-contradiction}
 \inf_{(E,D)\in\mathfrak{E}^L \times \mathfrak{D}^L}\sup_{f\in\cC}\|D(E(f)) - f \|_{L^2(\Omega)} \in \mathcal{O} \big(L^{-\gamma'}\big), \, L \rightarrow \infty.
\end{align}
In particular, (\ref{eq:encoder-decoder-contradiction}) would hold for some $\gamma'>\gamma^*(\cC)$ which, owing to \eqref{eq:gamma_star_flipped} stands in contradiction to the definition of $\gamma^*(\cC)$. This completes the proof.
\end{proof}

\begin{center}
\begin{tabular}{ |ll|l|l|ll| }
\hline
Space & & $\mathcal{C}$ & Optimal dictionary & $\gamma^*(\mathcal{C})$ &  \\
\hline
$L^2$-Sobolev & $W_2^m([0,1])$ & $\mathcal{U}(W_2^m([0,1]))$ & Fourier/Wavelet basis &  $m$ & \cite[Sec.~14.2]{Donoho1998}\\
H\"older & $C^\alpha([0,1])$ & $\mathcal{U}(C^\alpha([0,1]))$ & Wavelet basis & $\alpha$ & \cite[Sec.~14.2]{Donoho1998}\\
Bump Algebra & $B_{1,1}^1([0,1])$ & $\mathcal{U}(B_{1,1}^1([0,1]))$ & Wavelet basis & $1$ & \cite[Sec.~14.2]{Donoho1998}\\
Bounded Variation & $BV([0,1])$ & $\mathcal{U}(BV([0,1]))$ & Haar basis & $1$ & \cite[Sec.~14.2]{Donoho1998}\\
$L^p$-Sobolev\footnote{$p\in[1,\infty]$, $m > d (1/p - 1/2)_+$} & $W_p^m(\Omega)$ & $\mathcal{U}(W_p^m(\Omega))$ & Wavelet frame & $\tfrac{m}{d}$ & \cite[Thm.~1.3]{grohs2020phase} \\
Besov\footnote{$p,q\in(0,\infty]$, $m > d (1/p - 1/2)_+$} & $B_{p,q}^m(\Omega)$ & $\mathcal{U}(B_{p,q}^m(\Omega))$ & Wavelet frame & $\tfrac{m}{d}$ &\cite[Thm.~1.3]{grohs2020phase} \\
Modulation\footnote{$1 < p <2$, $s \in \R_+$} & $M^s_{p,p}(\R^d)$ & 
$\mathcal{U}(M^s_{p,p}(\R^d))$ & Wilson basis & $\tiny{(\frac{1}{p}\!-\!\frac{1}{2}\!+\!\frac{2s}{d})^{-1}}$ & \cite[Thm.~4.4]{Compactness_of_Mod_Spaces}\\
Cartoon functions\footnote{This is actually a set of functions and not a (unit) ball in a Banach space.}
& & $\mathcal{E}^{\beta}([-\tfrac{1}{2},\tfrac{1}{2}]^d)$ &$\alpha$-Curvelet frame\footnote{For $d=2$, see \cite{Grohs2016}.} & $\frac{\beta(d-1)}{2}$ & \cite{PetersenVoigtlaender}\\

\hline
\end{tabular}
\captionof{table}{Optimal exponents and corresponding optimal dictionaries. $\mathcal{U}(X)=\{f \in X:\|f\|_{X} \le 1\}$ denotes the unit ball in the space $X$ and
$\Omega\subseteq\R^d$ is a Lipschitz domain. Recall that compactness of these unit balls is w.r.t.\@ $L^2$-norm.}
\label{table_opt_exp}
\end{center}
\pagebreak

The optimal exponent $\gamma^*(\cC)$ is known for various function classes such as unit balls in Besov spaces $B_{p,q}^m(\R^d)$ with $p,q\in(0,\infty]$ and $m > d (1/p - 1/2)_+$, where $\gamma^*(\cC) ={m}/{d}$ (see \cite{grohs2020phase}), and unit balls in (polynomially) weighted modulation spaces $M^s_{p,p}(\R^d)$ with $p\in(1,2)$ and $s\in\R_+$, where $\gamma^*(\cC)=(\frac{1}{p}-\frac{1}{2}+\frac{2s}{d})^{-1}$ (see \cite{Compactness_of_Mod_Spaces}). A further example is the set of $\beta$-cartoon-like functions, which are $\beta$-smooth on some bounded $d$-dimensional domain with sufficiently smooth boundary and zero otherwise. Here, we have $\gamma^*(\cC)={\beta(d-1)}/{2}$ (see \cite{Don2001Sparse,Grohs2016,PetersenVoigtlaender}). These examples along with additional ones are summarized in Table~\ref{table_opt_exp}. For an extensive summary of metric entropy results and techniques for their derivation, we also refer to \cite{Kolmogorov1959}.

We conclude this section with general remarks on certain formal aspects of the Kolmogorov-Donoho rate-distortion framework. First, we note that for the set $\mathcal{C}\subseteq L^2(\Omega)$ to have a well-defined optimal exponent it must be relatively compact\footnote{For the sake of simplicity, we assume, however, compactness throughout even though relative compactness (i.e. having a compact closure) would be sufficient.}. This follows from the fact that the set over which the minimum in the definition \eqref{minimaxDef} of $L(\eps,\mathcal{C})$ is taken must be nonempty for every $\eps\in(0,\infty)$.  
To see this, note that every length-$L(\eps,\cal{C})$ encoder-decoder pair induces an $\eps$-covering of $\mathcal{C}$ with at most $2^{L(\eps,\cal{C})}$ balls (and ball centers $\{D(E(f))\}_{f\in\mathcal{C}})$. 
It hence follows that $\mathcal{C}$ must be totally bounded and thus relatively compact as a consequence of $L^2(\Omega)$ being a complete metric space~\cite[Thm.~45.1]{munkres2000topology}. 

As shown in the proof of Theorem~\ref{thm:optDictApproxLwrBd}, effective best $M$-term approximations construct encoder-decoder pairs and thereby induce $\epsilon$-coverings. By the arguments just made, this implies that also $\gamma^{*,\text{eff}}(\mathcal{C},\mathcal{D})$ is well-defined only for compact function classes $\mathcal{C}$. 

A consequence of the compactness requirement on $\mathcal{C}$ is that the spaces in Table~\ref{table_opt_exp} either consist of functions on bounded domains or, in the case of modulation spaces, are equipped with a weighted norm. In order to provide intuition on why this must be so, let us consider a function space $(X,\|\cdot\|_X)$ with $X\subseteq L^2(\R^d)$ and $\|\cdot\|_X$ translation invariant. Take $\eps>0$ and $f\in X$ with $\|f\|_X=1$ and choose $C>0$ such that $\|f\|_{L^2([-C,C]^d)} > \tfrac{4}{5}\|f\|_{L^2(\R^d)}$. Now, consider the family of translates of $f$ given by $f_i(x):=f(x-2Ci)$, $i\in\Z^d$, and note that $\|f_i\|_X=1$ for all $i\in\Z^d$ by translation invariance of $\|\cdot\|_X$. Furthermore, we have 
\begin{align*}
    \|f_i\|_{L^2([-C,C]^d)} = \left(\|f_i\|^2_{L^2(\R^d)}-\|f_i\|^2_{L^2(\R^{d}\setminus[-C,C]^d)}\right)^{\frac{1}{2}}\leq \left(\|f\|^2_{L^2(\R^d)}-\|f\|^2_{L^2([-C,C]^d)}\right)^{\frac{1}{2}}
    < \tfrac{3}{5}\|f\|_{L^2(\R^d)}
\end{align*}
for all $i\in\Z^d\backslash\{0\}$ by construction. This, in turn, implies 
\begin{align}
    \|f_i-f_j\|_{L^2(\R^d)} = \|f_{i-j}-f\|_{L^2(\R^d)}\geq \|f_{i-j}-f\|_{L^2([-C,C]^d)} >  \tfrac{1}{5}\|f\|_{L^2(\R^d)} \label{fi-spacing}
\end{align}
for all $i,j\in\Z^d$, with $i\neq j$, by the reverse triangle inequality. As such no $\eps$-ball (w.r.t.\@ $L^2(\R^d)$-norm) with $\eps\leq \tfrac{1}{10}\|f\|_{L^2(\R^d)}$ can contain more than one of the infinitely many $(f_i)_{i\in\Z^d}$ which are, however, all contained in the unit ball $\mathcal{U}(X)$ of the space $(X,\|\cdot\|_X)$. This implies that $\mathcal{U}(X)$ cannot be totally bounded and thereby not relatively compact (w.r.t.\@  $L^2(\R^d)$-norm). Somewhat nonchalantly speaking, for spaces equipped with translation-invariant norms this issue can be avoided by considering functions that live on a bounded domain, which ensures that (\ref{fi-spacing}) pertains only to a finite number of translates. Alternatively, for spaces of functions living on unbounded domains once can consider weighted norms that are not translation invariant. Here, the weighting effectively constrains the functions to a bounded domain.

The less restrictive concept of best $M$-term approximation rate $\gamma^{*}(\mathcal{C},\mathcal{D})$ (see Definition~\ref{def:optimalApproximationRate}) is, in apparent contrast, often studied for noncompact function classes $\mathcal{C}$. 

In \cite[Sec.~15.2]{Donoho1998} a condition for $\gamma^{*,\text{eff}}(\mathcal{C},\mathcal{D})$ and $\gamma^{*}(\mathcal{C},\mathcal{D})$ to coincide is presented. Specifically, this condition, referred to as tail compactness, is expressed as follows. Let $\mathcal{C}\subseteq L^2(\Omega)$ be bounded and let
$\mathcal{D}=\{\varphi_i\}_{i\in\N}$ be an ordered orthonormal basis for $\mathcal{C}$. We say that tail compactness holds if there exist
$C,\beta > 0$ such that for all $N\in\N$, 
\begin{align}\label{eq:tailcompactness}
\sup_{f\in\mathcal{C}}\left\|f-\sum_{i=1}^N\langle f,\varphi_i \rangle \varphi_i \right\|_{L^2(\Omega)}\leq CN^{-\beta}.  
\end{align}
In order to see that \eqref{eq:tailcompactness} implies $\gamma^{*,\text{eff}}(\mathcal{C},\mathcal{D})=\gamma^{*}(\mathcal{C},\mathcal{D})$, we consider, for fixed $f\in\mathcal{C}$, the (unconstrained) best $M$-term approximation $f_M=\sum_{i\in I} \langle f,\varphi_i \rangle \varphi_i$ with $I\subseteq\N$, $|I|=M$. 
We now modify this $M$-term approximation by letting $\alpha:=\lceil \gamma^*(\mathcal{C},\mathcal{D})/\beta \rceil \in\N$ and 
removing, in the expansion $f_M=\sum_{i\in I} \langle f,\varphi_i \rangle \varphi_i$, all terms corresponding to indices that are larger than $M^\alpha$. Recalling that in Definition \ref{def:polydepth} the same polynomial $\pi$ bounds the search depth and the size of the coefficients, it follows that the modified approximation we just constructed obeys a polynomial depth search constraint with constraining polynomial $\pi_\alpha(x)=x^\alpha+S$, where $S:=\sup_{f\in\mathcal{C}}\|f\|_{L^2(\Omega)}$. Here, owing to orthonormality of $\mathcal{D}$, $S$ accounts for the size of the expansion coefficients $\langle f,\varphi_i \rangle$.
In order to complete the argument, we need to show that the additional approximation error incurred by removing terms in $f_M=\sum_{i\in I} \langle f,\varphi_i \rangle \varphi_i$ is in $\mathcal{O}(M^{-\gamma^*(\mathcal{C},\mathcal{D})})$, i.e., it is
of the same order as the error corresponding to the original (unconstrained) best $M$-term approximation.
Due to orthonormality of $\mathcal{D}$ this additional error is given by the norm of $\sum_{i\in I, i > \pi_\alpha(M)} \langle f,\varphi_i \rangle\varphi_i$ and can, by virtue of \eqref{eq:tailcompactness}, be bounded as
\begin{align*}
    \left\|\sum_{\substack{i\in I, i > \pi_\alpha(M)}}  \langle f,\varphi_i \rangle\varphi_i \right\|_{L^2(\Omega)} &\leq
        \left\|\sum_{i = \pi_\alpha(M) +1}^\infty  \langle f,\varphi_i \rangle\varphi_i \right\|_{L^2(\Omega)} = 
        \left\|f-\sum_{i=1}^{\pi_{\alpha}(M)} \langle f,\varphi_i \rangle \varphi_i \right\|_{L^2(\Omega)}\\
    &\leq C(\pi_\alpha(M))^{-\beta} \in \mathcal{O}(M^{-\gamma^*(\mathcal{C},\mathcal{D})}),
\end{align*}
which establishes the claim. We have hence shown that under tail compactness of arbitrary rate $\beta > 0$, 
$\gamma^*(\mathcal{C},\mathcal{D})=\gamma^{*,\text{eff}}(\mathcal{C},\mathcal{D})$, and hence there is no cost incurred
by imposing a polynomial depth search constraint combined with a polynomial bound on the size of the expansion coefficients. We hasten to add that the assumptions stated at the beginning of this paragraph together with what was just established imply that $\gamma^{*,\text{eff}}(\mathcal{C},\mathcal{D})$ is, indeed, well-defined.
For the more general case of $\mathcal{D}$ a frame, we refer to \cite[Sec.~5.4.3]{grohs2015optimally} for analogous arguments.
Finally, we remark that the tail compactness inequality \eqref{eq:tailcompactness} can be interpreted as quantifying the rate of linear approximation for $\mathcal{C}$ in $\mathcal{D}$. Two examples of pairs $(\mathcal{C},\mathcal{D})$ satisfying tail compactness, namely Besov spaces with wavelet bases and modulation spaces with Wilson bases, are provided in Appendices~\ref{Besov_tail} and \ref{Modulation_tail}, respectively.

As already mentioned, a larger optimal exponent $\gamma^{\ast}(\mathcal{C})$ leads to faster error decay (specifically according to $L^{-\gamma^{\ast}(\mathcal{C})}$) and hence corresponds to a function class of smaller complexity.
As such, techniques for deriving lower bounds on the optimal exponent are often based on variations of the approach employed in the proof of Theorem~\ref{thm:optDictApproxLwrBd}, namely on the explicit construction of encoder-decoder pairs (in the case of the proof of Theorem~\ref{thm:optDictApproxLwrBd} by encoding the dictionary elements participating in the $M$-term approximation).
A powerful method for deriving upper bounds on the optimal exponent is the hypercube embedding approach proposed by Donoho in \cite{Don2001Sparse}; the basic idea here is to show that the function class $\mathcal{C}$ under consideration contains a sufficiently complex embedded set of orthogonal hypercubes and to then find the exponent corresponding to this set. 
An interesting alternative technique for deriving optimal exponents was proposed in the context of modulation spaces in \cite{Compactness_of_Mod_Spaces}. The essence of this approach is to exploit the isomorphism between weighted modulation spaces and weighted mixed-norm sequence spaces~\cite{grochenig2013foundations} and to then utilize results about entropy numbers of operators between sequence spaces.

\section{Approximation with Deep Neural Networks}\label{subsec:NNapproxintro}

Inspired by the theory of best $M$-term approximation with dictionaries, we now develop the new concept of best $M$-weight approximation through neural networks. At the heart of this theory lies the interpretation of the network weights as the counterpart of the coefficients $c_i$ in best $M$-term approximation.
In other words, parsimony in terms of the number of participating elements in a dictionary is replaced by parsimony in terms of network connectivity. 
Our development will parallel that for best $M$-term approximation in the previous section.

Before proceeding to the specifics, we would like to issue a general remark. While the neural network approximation results in Section~\ref{func-mult} were
formulated in terms of $L^\infty$-norm, we shall be concerned with $L^2$-norm approximation here, on the one hand paralleling the use of $L^2$-norm in the context of best $M$-term approximation, and on the other hand allowing for the approximation of discontinuous functions by ReLU neural networks, which, owing to the continuity of the ReLU nonlinearity, necessarily realize continuous functions.

We start by introducing the concept of best $M$-weight approximation rate.

\begin{definition}\label{def:optimalApproximationRateNN}
Given $d\in \N$, $\Omega \subseteq \R^d$, and a function class $\cC \subseteq L^2(\Omega)$, we define, for $f \in \cC$ and $M\in \N$,
\begin{align} \label{eq:GammaMDef}
\Gamma_M^{\cN}(f) := \inf_{\substack{\Phi \in \cN_{d,1}\\\cM(\Phi)\leq M}} \|f - \Phi \|_{L^2(\Omega)}.
\end{align}
We call $\Gamma_{M}^{\cN}(f)$ the \emph{best $M$-weight approximation error of $f$}.
The supremal $\gamma > 0$ such that 
\[
\sup_{f \in \cC} \Gamma_{M}^{\cN}(f) \in \mathcal{O}(M^{-\gamma}), \,\, M \rightarrow \infty,
\]
 will be denoted by $\gamma_{\cN}^\ast(\cC)$. We say that the {\em best $M$-weight approximation rate of $\cC$ by neural networks} is 
 $\gamma_{\cN}^\ast(\cC)$.
\end{definition}

We emphasize that the infimum in \eqref{eq:GammaMDef} is taken over all networks with fixed input dimension $d$, no more than $M$ nonzero (edge and node) weights, and arbitrary depth $L$. 
In particular, this means that the infimum is with respect to all possible network topologies and weight choices.
The best $M$-weight approximation rate is fundamental as it benchmarks all 
algorithms that map a function $f$ and an $\varepsilon>0$ to a neural network approximating $f$ with error no more than $\varepsilon$.

The two restrictions underlying the concept of effective best $M$-term approximation through dictionaries, namely polynomial depth search and polynomially bounded coefficients, are next addressed in the context of approximation through deep neural networks. We start by noting that the need for the former is obviated by the tree-like-structure of neural networks. To see this, first note that $\mathcal{W}(\Phi) \le \mathcal{M}(\Phi)$ and $\mathcal{L}(\Phi) \le \mathcal{M}(\Phi)$.
As the total number of nonzero weights in the network can not exceed $\mathcal{L}(\Phi)\mathcal{W}(\Phi)(\mathcal{W}(\Phi)+1)$, this yields at most $\mathcal{O}(\mathcal{M}(\Phi)^3)$ possibilities for the ``locations'' (in terms of entries in the $A_\ell$ and the $b_\ell$) of the $\mathcal{M}(\Phi)$ nonzero weights. 
Encoding the locations of the $\mathcal{M}(\Phi)$ nonzero weights hence requires $\log({C\mathcal{M}(\Phi)^3 \choose \mathcal{M}(\Phi)})=\mathcal{O}(\mathcal{M}(\Phi)\log(\mathcal{M}(\Phi)))$ bits. This assumes, however, that the architecture of the network, i.e., the number of layers $\mathcal{L}(\Phi)$ and the $N_{k}$ are known. Proposition \ref{prop:optimalitynoquant} below shows that the 
architecture can, indeed, also be encoded with $\mathcal{O}(\mathcal{M}(\Phi)\log(\mathcal{M}(\Phi)))$ bits.
In summary, we can therefore conclude that the tree-like-structure of neural networks automatically guarantees what we had to enforce through the polynomial depth search constraint in the case of best $M$-term approximation. 

Inspection of the approximation results in Section \ref{func-mult} reveals that a sublinear growth restriction on $\mathcal{L}(\Phi)$ as a function of $\mathcal{M}(\Phi)$ is natural.
Specifically, the approximation results in Section \ref{func-mult}
all have $\mathcal{L}(\Phi)$ proportional to a polynomial in $\log(\varepsilon^{-1})$.  
As we are interested in approximation error decay according to $\mathcal{M}(\Phi)^{-\gamma}$, 
see Definition \ref{def:optimalApproximationRateNN}, this suggests to restrict $\mathcal{L}(\Phi)$ to growth that is polynomial in $\log(\mathcal{M}(\Phi))$.

The second restriction imposed in the definition of effective best $M$-term approximation, namely polynomially bounded coefficients, will be imposed in monomorphic manner on the magnitude of the weights.
This growth condition will turn out natural in the context of the approximation results we are interested in and will, together with polylogarithmic depth growth, be seen below to allow rate-distortion-optimal quantization of the network weights. We remark, however, that networks with weights growing polynomially in $\mathcal{M}(\Phi)$ can be converted into networks with uniformly bounded weights at the expense of increased---albeit still of
polylogarithmic scaling in $\mathcal{M}(\Phi)$---depth (see Proposition \ref{WDtradeoff}).
In summary, we will develop the concept of ``best $M$-weight approximation subject to polylogarithmic depth and polynomial weight growth''. 

We start by introducing the following notation for neural networks with depth and weight magnitude bounded polylogarithmically respectively polynomially w.r.t.\@ their connectivity.

\begin{definition}\label{def:NNpisets}
For $M,d,d'\in\N$, and $\pi$ a polynomial, we define
\begin{align*}
    \cN^\pi_{M,d,d'}:=\left\{\Phi\in\cN_{d,d'} \colon \cM(\Phi)\leq M, \cL(\Phi)\leq\pi(\log(M)), \mathcal{B}(\Phi)\leq \pi(M)\right\}.
\end{align*}
\end{definition}

Next, we formalize the notion of effective best $M$-weight approximation rate subject to polylogarithmic depth and polynomial weight growth.

\pagebreak

\begin{definition}\label{def:repoptiNN}
Let $d\in \N$, $\Omega\subseteq \mathbb{R}^d$, and let $\cC\subseteq L^2(\Omega)$ be compact.
We define for $M\in\N$ and $\pi$ a polynomial
\begin{align*}
    \eps^{\pi}_{\cN}(M):= \sup_{f \in \mathcal{C}}\, \inf_{\Phi \in \cN^\pi_{M,d,1}}  \|f    - \Phi\|_{L^2(\Omega)}
\end{align*}
and 
\begin{align*}
\gamma^{\ast,\text{eff}}_{\cN}(\cC):=\sup\{\gamma\geq 0\colon \exists\ \mathrm{polynomial}\ \pi\ \mathrm{s.t.}\  \eps^{\pi}_{\cN}(M) \in  
\mathcal{O}(M^{-\gamma}), \, M \rightarrow \infty\}.
\end{align*}
We refer to $\gamma^{\ast,\text{eff}}_{\cN}(\cC)$ as the {\em effective best $M$-weight approximation rate of $\cC$}.
\end{definition}

We now state the equivalent of Theorem~\ref{thm:optDictApproxLwrBd} for approximation by deep neural networks.
Specifically, we establish that the optimal exponent $\gamma^*(\cC)$ constitutes a fundamental bound on the effective best $M$-weight
approximation rate of $\cC$ as well.

\begin{theorem}\label{thm:EffRepNN}
Let $d\in \N$, $\Omega\subseteq \mathbb{R}^d$, and let $\cC\subseteq L^2(\Omega)$ be compact. Then, we have 
	$$
	\gamma_{\cN}^{\ast, \text{eff}}(\cC) \leq {\gamma^\ast(\cC)}.
	$$
\end{theorem}

The key ingredients of the proof of Theorem~\ref{thm:EffRepNN} are developed throughout this section and the formal proof appears at the end of the section. Before 
getting started, we note that, in analogy to Definition~\ref{def:repopti}, what we just found suggests the following.
\begin{definition}\label{def:NNopti}
	Let $d\in \N$, $\Omega\subseteq \mathbb{R}^d$, and let $\cC\subseteq L^2(\Omega)$ be compact. We say that the function class $\mathcal{C}\subseteq L^2(\Omega)$ is \emph{optimally representable by neural networks} if 
	$$\gamma_{\cN}^{\ast, \text{eff}}(\cC) = {\gamma^*(\cC)}.$$
\end{definition}

It is interesting to observe that the fundamental limits of effective best $M$-term approximation (through dictionaries) and effective best $M$-weight approximation in neural networks are determined by the same quantity, although the approximants in the two cases
are vastly different. We have linear combinations of elements of a dictionary under polynomial weight growth of the coefficients and with the participating functions identified subject to a polynomial-depth search constraint in the former,
and concatenations of affine functions followed by nonlinearities under polynomial growth constraints on the coefficients of the affine functions 
and with a polylogarithmic growth constraint on the number of concatenations in the latter case.

We now commence the program developing the proof of Theorem~\ref{thm:EffRepNN}. As in the arguments in the proof sketch of Theorem \ref{thm:optDictApproxLwrBd},
the main idea is to compare the length of the bitstring needed to encode the approximating network to the minimax code length of the function class $\mathcal{C}$ to be approximated. To this end, we will need to represent the approximating network's nonzero weights, its architecture, i.e., $L$ and the $N_k$, and the nonzero weights' locations as a bitstring.
As the weights are real numbers and hence require, in principle, an infinite number of bits for their binary representations, we will have to suitably quantize them. In particular, the resolution of the corresponding quantizer will have to increase appropriately with decreasing $\varepsilon$. To formalize this idea, we start by defining the quantization employed.

\begin{definition}\label{def:quantizedweights}
 Let $m\in\N$ and $\eps\in(0,1/2)$. The network $\Phi$ is said to have $(m,\eps)$-quantized weights if all its weights are elements of $2^{-m\lceil\log(\eps^{-1})\rceil}\Z\cap[-
\eps^{-m},\eps^{-m}]$.
\end{definition}

A key ingredient of the proof of Theorem~\ref{thm:EffRepNN} is the following result, which establishes a fundamental lower bound on
the connectivity of networks with quantized weights achieving uniform error $\varepsilon$ over a given function class $\mathcal{C}$.
\begin{proposition}\label{prop:optimalitynoquant}
Let $d, d'\in \N$, $\Omega\subseteq \mathbb{R}^d$,  $\cC \subseteq L^2(\Omega)$, and let $\pi$ be a polynomial. 
Further, let
\[
\Psi: \left(0,\tfrac{1}{2}\right) \times \cC \to \cN_{d, d'}
\]
be a map such that for every $\eps\in(0,1/2)$, $f\in\mathcal{C}$, the network $\Psi(\varepsilon, f)$ has $(\lceil\pi(\log(\eps^{-1}))\rceil,\eps)$-quantized weights and satisfies
\begin{equation*}
    \sup_{f \in \cC}\|f - \Psi(\varepsilon, f)\|_{L^2(\Omega)} \leq \varepsilon.
\end{equation*}
Then, 
\begin{equation*}
    \sup_{f\in \cC}\mathcal{M}(\Psi(\varepsilon, f)) \notin \mathcal{O}\! \left(\varepsilon^{-1/\gamma}\right), \, \varepsilon \rightarrow 0, \quad \mbox{for all }\gamma> {\gamma^\ast(\cC)}.
    \end{equation*}
\end{proposition}

\begin{proof}
The proof is by contradiction. Let $\gamma > \gamma^*(\mathcal{C})$ and assume that $\sup_{f\in \cC}\mathcal{M}(\Psi(\varepsilon, f)) \in \mathcal{O}(\varepsilon^{-1/\gamma}), \varepsilon \rightarrow 0$. The contradiction will be effected by constructing encoder-decoder pairs $(E_{\varepsilon},D_{\varepsilon}) \in \mathfrak{E}^{\ell(\varepsilon)} \times \mathfrak{D}^{\ell(\varepsilon)}$ achieving
uniform error $\varepsilon$ over $\cC$ with
\begin{align}\label{eq:CodeLength}
\ell(\varepsilon) & \leq  C_0 \cdot \sup_{f\in \cC}\left(\mathcal{M}(\Psi(\varepsilon, f))\log(\mathcal{M}(\Psi(\varepsilon, f)))+1\right)(\log(\varepsilon^{-1}))^q\\
& \leq  C_0\left( \varepsilon^{-1/\gamma}\log(\varepsilon^{-1/\gamma})+1\right)(\log(\varepsilon^{-1}))^q\nonumber \\
& \leq  C_1 \left( \varepsilon^{-1/\gamma}(\log(\varepsilon^{-1}))^{q+1}
+(\log(\varepsilon^{-1}))^q \right) \in \mathcal{O}\left(\varepsilon^{-1/\nu}\right), \quad \text{for}\,\, \varepsilon \rightarrow 0,\nonumber
\end{align}
where $C_0,C_1,q>0$ are constants not depending on $f, \eps$ and $\gamma > \nu> \gamma^*(\cC)$. The specific form of the upper bound (\ref{eq:CodeLength}) will become apparent in the construction of the bitstring representing $\Psi$ detailed below.

We proceed to the construction of the encoder-decoder pairs $(E_{\varepsilon},D_{\varepsilon}) \in \mathfrak{E}^{\ell(\varepsilon)} \times \mathfrak{D}^{\ell(\varepsilon)}$, 
which will be accomplished
by encoding the network architecture, its topology, and the quantized weights in bitstrings of length $\ell(\varepsilon)$ satisfying \eqref{eq:CodeLength} while
guaranteeing unique reconstruction (of the network).
For the sake of notational simplicity, we fix $\eps\in(0,1/2)$ and $f\in \cC$ and set
$\Psi:=\Psi(\varepsilon, f)$, $M := \mathcal{M}(\Psi)$, and $L:=\L(\Psi)$. 
Recall that the number of nodes in layers $0,\dots,L$ is denoted by $N_0,\dots,N_L$ and that $N_0=d, N_L=d'$ (see Definition~\ref{def:NN}). Moreover, note that due to our nondegeneracy assumption (see Remark~\ref{remark:non-degeneracy}) we have $\sum_{\ell=0}^L N_\ell\leq 2M$ and $L\leq M$.
The bitstring representing $\Psi$ is constructed according to the following steps.

{\it Step 1:}  If $M=0$, we encode the network by a single $0$. Using the convention $0\log(0)=0$, we then note that (\ref{eq:CodeLength}) holds trivially and we terminate the encoding procedure. Else, we encode the network connectivity, $M$, by starting the overall bitstring with $M$ $1$'s followed by a single $0$. The length of this bitstring is therefore given by $M+1$.

{\it Step 2:} We continue by encoding the number of layers which, due to $L\le M$, requires no more than $\lceil\log(M)\rceil$ bits. We thus reserve the next $\lceil\log(M)\rceil$ bits for the binary representation of $L$.

{\it Step 3:} Next, we store the layer dimensions $N_0,\dots,N_L$. As $L\leq M$ and $N_\ell\leq M$, for all $\ell\in\{0,\dots,L\}$, owing to nondegeneracy, we can encode the layer dimensions using $(M+1)\lceil\log(M)\rceil$ bits.
In combination with Steps 1 and 2 this yields an overall bitstring of length at most
\begin{equation}
\label{sizecode}
M\lceil\log(M)\rceil+M+2\lceil\log(M)\rceil+1.
\end{equation}

{\it Step 4:} We encode the topology of the graph associated with the network $\Psi$.
To this end, we enumerate all nodes 
by assigning a unique index $i$ to each one of them, starting from the $0$-th layer and increasing from left to right within a given layer. 
The indices range from $1$ to $N:=\sum_{\ell=0}^{L} N_\ell\leq 2M$. Each of these indices can be encoded by a bitstring of length $\lceil\log(N)\rceil$. We denote the bitstring corresponding to index $i$ by $b(i)\in \{0,1\}^{\lceil\log(N)\rceil}$ and let for all nodes, except for those in the last layer, $n(i)$ be the number of children of the node with index $i$, i.e., the number of nodes in the next layer connected to the node with index $i$ via an edge. 
For each of these nodes $i$,
we form a bitstring of length $n(i)\lceil\log(N)\rceil$ by concatenating the bitstrings indexing its children. We follow this string with an all-zeros bitstring of length $\lceil\log(N)\rceil$ to signal that all children of the current node have been encoded.
Overall, this yields a bitstring of length
\begin{equation}
\label{topcode}
    \sum_{i=1}^{N-d'} (n(i)+1) \lceil\log(N)\rceil \leq 3M \lceil\log(2M)\rceil,
\end{equation}
where we used $\sum_{i=1}^{N-d'} n(i) \le M$. 

{\it Step 5:} We encode the weights of $\Psi$.
By assumption, $\Psi$ has $(\lceil\pi(\log(\varepsilon^{-1}))\rceil,\eps)$-quantized weights, which means that each weight of $\Psi$ can be represented by 
no more than $B_\eps:=2(\lceil\pi(\log(\eps^{-1}))\rceil\lceil\log(\eps^{-1})\rceil+1)$ bits.
For each node $i=1,\dots,N$, we reserve the first $B_\eps$ bits to encode its associated node weight and, for each of its children a bitstring of length $B_\varepsilon$ to encode
the weight corresponding to the edge between the current node and that child. Concatenating the results in ascending order of child node indices, we get a bitstring of length $(n(i)+1)B_\varepsilon$ for node $i$, and an overall bitstring of length
\begin{equation*}
    \sum_{i = 1}^{N-d'} (n(i)+1)B_\varepsilon + d'B_\varepsilon \leq 3MB_\varepsilon
\end{equation*}
representing the weights.
Combining this with (\ref{sizecode}) and (\ref{topcode}), we find that the overall number of bits needed to encode the network architecture, topology, and weights is no more than
\begin{align}\label{eq:HelmutCalledThisAlpha}
    3MB_\varepsilon+3M\lceil\log(2M)\rceil+(M+2)\lceil\log(M)\rceil+M+1.
\end{align}
The network can be recovered by sequentially reading out $M,L$, the $N_{\ell}$, the topology, and the quantized weights from the overall bitstring.
It is not difficult to verify that the individual steps in the encoding procedure were crafted such that this yields unique 
recovery. 
As \eqref{eq:HelmutCalledThisAlpha} can be upper-bounded by
\begin{equation*} 
    C_0 (M \log(M) +1) (\log \! \left(\varepsilon^{-1}\right))^q
\end{equation*}
for constants $C_0,q>0$ depending on $\pi$ only, we have constructed an encoder-decoder pair $(E_{\varepsilon},D_{\varepsilon})\in \mathfrak{E}^{\ell(\varepsilon)} \times \mathfrak{D}^{\ell(\varepsilon)}$ 
with  $\ell(\varepsilon)$ satisfying (\ref{eq:CodeLength}).
This concludes the proof.
\end{proof}

Proposition \ref{prop:optimalitynoquant} states that the connectivity growth rate of networks with quantized weights achieving uniform approximation error $\varepsilon$ over
a function class $\mathcal{C}$ must exceed
$\mathcal{O}\! \left(\varepsilon^{-1/\gamma^\ast(\cC)}\right), \, \varepsilon \rightarrow 0$.
As Proposition \ref{prop:optimalitynoquant} applies to networks that have each weight represented by a finite number of bits scaling polynomially in $\log(\varepsilon^{-1})$, while guaranteeing that the underlying encoder-decoder pair achieves uniform error $\varepsilon$ over $\cC$, it remains to establish
that such a compatibility is, indeed, possible. Specifically, this requires a careful interplay between the network's depth and connectivity scaling, and its
weight growth, all as a function of $\varepsilon$. Establishing that this delicate balancing is implied by our technical assumptions is the subject of the remainder of this section. We start with a perturbation result quantifying how the error induced by weight quantization in the network translates to the output function realized by the network.

\begin{lemma}\label{lem:PolynomiallyboundedImpliesFiniteBitLength}
 Let $d,d',k\in\N$, $D\in\R_+$, $\Omega\subseteq[-D,D]^d$, $\varepsilon\in(0,1/2)$, let $\Phi\in\cN_{d,d'}$ with $\cM(\Phi)\leq\varepsilon^{-k}$, $\mathcal{B}(\Phi)\leq\varepsilon^{-k}$, and let $m\in\N$ satisfy
 \begin{align}\label{mlowerboundquant}
 m \geq 3k\cL(\Phi)+\log(\lceil D \rceil).
 \end{align}
Then, there exists a network $\tilde{\Phi}\in\cN_{d,d'}$ with $(m,\eps)$-quantized weights satisfying
 $$\sup_{x\in\Omega}\|\Phi(x)-\tilde{\Phi}(x)\|_{\infty} \leq \varepsilon. $$
\end{lemma}

More specifically, the network $\tilde{\Phi}$ can be obtained simply by replacing every weight in $\Phi$ by a closest element in $2^{-m\lceil\log(\eps^{-1})\rceil}\Z\cap[-
\eps^{-m},\eps^{-m}]$.

\begin{proof}[Proof of Theorem~\ref{lem:PolynomiallyboundedImpliesFiniteBitLength}]
We first consider the case $\cL(\Phi)=1$. Here, it follows from Definition \ref{def:NN} that the network simply realizes an affine
transformation and hence
$$\sup_{x\in\Omega}\|\Phi(x)-\tilde{\Phi}(x)\|_{\infty} \leq \cM(\Phi)\lceil D\rceil 2^{-m\lceil\log(\eps^{-1})\rceil-1}\leq\eps.$$
In the remainder of the proof, we can therefore assume that $\cL(\Phi)\geq 2$. For simplicity of notation, we set $L:=\cL(\Phi), M:=\cM(\Phi)$, and, as usual, write
\begin{align*}
    \Phi = W_L\circ\rho \circ W_{L-1} \circ \rho \circ \dots \circ \rho \circ W_{1}
\end{align*}
with $W_\ell(x)=A_\ell x+b_\ell$, $A_\ell\in\R^{N_\ell\times N_{\ell-1}}$, and $b_\ell\in\R^{N_\ell}$.
We now consider the partial networks $\Phi^\ell\colon\Omega\to\R^{N_\ell}$, $\ell \in \{1,2,\dots,L-1\}$, given by
\begin{align*}
\Phi^{\ell} := \begin{cases}
\begin{array}{lc} \rho \circ W_1, & \ell=1\\
\rho \circ W_2 \circ \rho\circ W_1, & \ell=2\\
\rho \circ W_{\ell}\circ\rho\circ W_{\ell-1} \circ\dots\circ \rho \circ W_{1}, & \ell = 3,\dots,L-1,
\end{array} \end{cases}    
\end{align*}
and set $\Phi^L:=\Phi$. We hasten to add that we decided---for ease of exposition---to deviate from the convention used in Definition \ref{def:NN} and to have the partial networks include the application of $\rho$ at the end. Now, for $\ell\in\{1,2,\dots,L\}$, let $\tilde{\Phi}^\ell$ be the (partial) network obtained by replacing all the entries of the $A_\ell$ and $b_\ell$ by a closest
element in ${2^{-m\lceil\log(\eps^{-1})\rceil}\,\Z\,\cap\,[-\eps^{-m},\eps^{-m}]}$. We denote these replacements by $\tilde{A}_{\ell}$ and $\tilde{b}_{\ell}$, respectively, and note that
  \begin{align}\begin{split}\label{PBFBL5}
    \max_{i,j}|A_{\ell,i,j}-\tilde{A}_{\ell,i,j}|&\leq\tfrac{1}{2}\,2^{-m\lceil\log(\eps^{-1})\rceil}\leq \tfrac{1}{2}\,\eps^{m},\\
    \max_{i,j}|b_{\ell,i,j}-\tilde{b}_{\ell,i,j}|&\leq\tfrac{1}{2}\,2^{-m\lceil\log(\eps^{-1})\rceil}\leq \tfrac{1}{2}\,\eps^{m}.
  \end{split}\end{align}
  The proof will be effected by upper-bounding the error building up across layers as a result of this quantization.
  To this end, we define, for $\ell\in\{1,2,\dots,L\}$, the error in the $\ell$-th layer as
  \begin{align*}
    e_\ell:=\sup_{x\in\Omega}\|\Phi^\ell(x)-\tilde{\Phi}^\ell(x)\|_{\infty}.
  \end{align*}
  We further set $C_0:=\lceil D \rceil$ and $C_\ell:=\max\{1,\sup_{x\in\Omega}\|\Phi^\ell(x)\|_{\infty}\}$. 
  As each entry of the vector $\Phi^\ell(x)\in\R^{N_\ell}$ is obtained by applying\footnote{Note that going from $\Phi_{L-1}$ to $\Phi_L$ the activation function is not applied anymore, which nevertheless leads to the same estimate as the identity mapping is $1$-Lipschitz.} the $1$-Lipschitz function $\rho$ to
  the sum of a weighted sum of at most $N_{\ell-1}$ components of the vector $\Phi^{\ell-1}(x)\in\R^{N_{\ell-1}}$ 
  and a bias component $b_{\ell,i}$, and $\mathcal{B}(\Phi)\,\le\,\varepsilon^{-k}$ by assumption, we have
  for all $\ell\in\{1,2,\dots,L\}$,
  \begin{align*}
    C_\ell\leq N_{\ell-1}\eps^{-k}C_{\ell-1}+\eps^{-k}\leq (N_{\ell-1}+1)\,\eps^{-k}C_{\ell-1},
  \end{align*}
  which implies, for all $\ell\in\{1,2,\dots,L\}$, that
  \begin{align}\label{PBFBL4}
    C_\ell\leq C_0\,\eps^{-k\ell}\prod_{i=0}^{\ell-1}(N_i+1).
  \end{align}
  Next, note that the components $(\tilde{\Phi}^1(x))_i, i \in \{1,2,\dots,N_1\}$, of the vector $\tilde{\Phi}^1(x)\in\R^{N_1}$ can be written as
  \begin{align*}
      (\tilde{\Phi}^1(x))_i= \rho \left ( \left(\sum_{j=1}^{N_0}\tilde{A}_{1,i,j}x_j\right)+\tilde{b}_{1,i} \right ),
  \end{align*}
  which, combined with \eqref{PBFBL5} and the fact that $\rho$ is $1$-Lipschitz implies 
  \begin{align}\label{PBFBL1}
    e_1\leq C_0 N_0\tfrac{\eps^m}{2}+\tfrac{\eps^m}{2}\leq C_0 (N_0+1)\tfrac{\eps^m}{2}.
  \end{align}
  Due to $\rho$ and the identity mapping being $1$-Lipschitz, we have, for $\ell=1,\dots,L$,
  \begin{align}\begin{split}\label{PBFBL6}
      e_{\ell}&=\sup_{x\in\Omega}\|\Phi^{\ell}(x)-\tilde{\Phi}^{\ell}(x)\|_{\infty}=\sup_{x\in\Omega,i\in\{1,\dots,N_\ell\}}|(\Phi^{\ell}(x))_i-(\tilde{\Phi}^{\ell}(x))_i|\\
      & \le \sup_{x\in\Omega,i\in\{1,\dots,N_{\ell}\}}\left|\left[\left(\sum_{j=1}^{N_{\ell-1}}A_{\ell,i,j} (\Phi^{\ell-1}(x))_j\right)+b_{\ell,i}\right]-\left[\left(\sum_{j=1}^{N_{\ell-1}}\tilde{A}_{\ell,i,j}
      (\tilde{\Phi}^{\ell-1}(x))_j\right)+\tilde{b}_{\ell,i}\right]\right|\\
      &\leq\sup_{x\in\Omega,i\in\{1,\dots,N_{\ell}\}}\left[\left(\sum_{j=1}^{N_{\ell-1}}\left|A_{\ell,i,j}(\Phi^{\ell-1}(x))_j-\tilde{A}_{\ell,i,j}(\tilde{\Phi}^{\ell-1}(x))_j\right|\right)+\left|b_{\ell,i}-
      \tilde{b}_{\ell,i}\right|\right].
  \end{split}\end{align}
  As $|(\Phi^{\ell-1}(x))_j-(\tilde{\Phi}^{\ell-1}(x))_j|\leq e_{\ell-1}$ and $|(\Phi^{\ell-1}(x))_j|\leq C_{\ell-1}$ for all $x \in \Omega$, $j\in\{1,\dots,N_{\ell-1}\}$ by definition, and $|A_{\ell,i,j}|\leq\eps^{-k}$ by assumption, upon invoking \eqref{PBFBL5}, we get
    \begin{align*}
      |A_{\ell,i,j}(\Phi^{\ell-1}(x))_j-\tilde{A}_{\ell,i,j}(\tilde{\Phi}^{\ell-1}(x))_j|\leq e_{\ell-1}\eps^{-k}+C_{\ell-1}\tfrac{\eps^m}{2}+e_{\ell-1}\tfrac{\eps^m}{2}.
  \end{align*}
  Since $\eps\in(0,1/2)$, it therefore follows from \eqref{PBFBL6}, that for all $\ell\in\{2,\dots,L\}$,
  \begin{align}\label{PBFBL3}
    e_{\ell}\leq N_{\ell-1}(e_{\ell-1}\eps^{-k}+C_{\ell-1}\tfrac{\eps^m}{2}+e_{\ell-1}\tfrac{\eps^m}{2})+\tfrac{\eps^m}{2}\leq (N_{\ell-1}+1)(2e_{\ell-1}\eps^{-k}+C_{\ell-1}\tfrac{\eps^m}{2}).
  \end{align}
  We now claim that, for all $\ell\in\{2,\dots,L\}$,
  \begin{align}\label{PBFBL2}
    e_{\ell}\leq \tfrac{1}{2}(2^{\ell}-1) C_0\eps^{m-(\ell-1)k} \prod_{i=0}^{\ell-1}(N_i+1),
  \end{align}
  which we prove by induction. The base case $\ell=1$ was already established in \eqref{PBFBL1}. For the induction step we assume that \eqref{PBFBL2} holds 
  for a given $\ell$
  which, in combination with \eqref{PBFBL4} and \eqref{PBFBL3}, implies 
  \begin{align*}
    e_{\ell+1}&\leq \left(N_{\ell}+1)(2e_{\ell}\eps^{-k}+C_{\ell}\tfrac{\eps^m}{2}\right)\\
    &\leq (N_{\ell}+1)\left((2^{\ell}-1)C_0\eps^{m-(\ell-1)k}\eps^{-k}\prod_{i=0}^{\ell-1}(N_i+1)+C_0\eps^{-k\ell}\tfrac{\eps^m}{2}\prod_{i=0}^{\ell-1}(N_i+1)\right)\\
    &= \frac{1}{2}(2^{\ell +1}-1)C_0\eps^{m-\ell k}\prod_{i=0}^{\ell} (N_i+1).
  \end{align*}
  This completes the induction argument and establishes \eqref{PBFBL2}. Using $2^{L-1}\leq\eps^{-(L-1)}$, $\prod_{i=0}^{L-1}(N_i+1)\leq M^{L}\leq\eps^{-kL}$, 
  and $m \geq 3kL+\log(\lceil D\rceil)$ by assumption, we get
  \begin{align*}
   \sup_{x\in\Omega}\|\Phi(x)-\tilde{\Phi}(x)\|_{\infty}&=e_L\leq\tfrac{1}{2}(2^L-1)C_0\eps^{m-(L-1)k}\prod_{i=0}^{L-1}(N_i+1)\\
   &\leq \eps^{m-(L-1+kL-k+\log(\lceil D\rceil)+kL)}\\
   &\leq \eps^{m-(3kL+\log(\lceil D\rceil)-1)}\leq\eps.
  \end{align*}
  This completes the proof.
\end{proof}

We are now ready to finalize the proof of Theorem~\ref{thm:EffRepNN}.
\begin{proof}[Proof of Theorem~\ref{thm:EffRepNN}]
Suppose towards a contradiction that $\gamma_{\cN}^{\ast, \text{eff}}(\cC) > {\gamma^\ast(\cC)}$ and let $\gamma \in \big(\gamma^\ast(\cC),\gamma_{\cN}^{\ast, \text{eff}}(\cC)\big)$.
Then, by Definition~\ref{def:repoptiNN}, there exist a polynomial $\pi$ and a constant $C>0$ such that
\begin{align*} 
    \sup_{f \in \mathcal{C}}\, \inf_{\Phi \in \cN^{\pi}_{M,d,1}}  \|f - \Phi\|_{L^2(\Omega)} \leq
C M^{-\gamma}, \,\, \text{for all}\,\, M \in \N.
\end{align*}
Setting $M_\varepsilon := \big\lceil (\varepsilon/(4C))^{-1/\gamma}\big\rceil$, it follows that, for every $f \in \mathcal{C}$ and every $\varepsilon \in (0,1/2)$, there exists a neural network $\Phi_{\varepsilon,f} \in \cN^{\pi}_{ M_\eps, d, 1}$ such that 
\begin{align}\label{eq:supinfeps}
     \|f- \Phi_{\varepsilon,f} \|_{L^2(\Omega)} \leq 2 \sup_{f \in \cC}  \ \inf_{\Phi \in \cN^\pi_{M_\varepsilon, d, 1}} \|f - \Phi\|_{L^2(\Omega)} \leq 2 C M_\varepsilon^{-\gamma} \leq \frac{ \varepsilon}{2}.
\end{align}
By Lemma \ref{lem:PolynomiallyboundedImpliesFiniteBitLength} there exists a polynomial $\pi^*$ such that for every $f \in \mathcal{C}$, $\varepsilon \in (0,1/2)$, there is a network $\widetilde{\Phi}_{\varepsilon,f}$ 
with $(\lceil\pi^*(\log(\eps^{-1}))\rceil,\eps)$-quantized weights satisfying
\begin{align}\label{eq:quantizedweightspsi}
    \left\|\Phi_{\varepsilon,f} - \widetilde{\Phi}_{\varepsilon,f} \right\|_{L^2(\Omega)} \leq \frac{\varepsilon}{2}.
\end{align}
The conditions of Lemma \ref{lem:PolynomiallyboundedImpliesFiniteBitLength} are satisfied as $M_\varepsilon$ can be upper-bounded by $\eps^{-k}$ with a suitably chosen $k$, the weights in $\Phi_{\varepsilon,f}$ are polynomially bounded in $M_\varepsilon$, and (\ref{mlowerboundquant}) follows from the depth of networks in $\Phi \in \cN^\pi_{M_\varepsilon, d, 1}$ being polylogarithmically bounded in $M_\varepsilon$ due to Definition \ref{def:NNpisets}.
Now, defining
\begin{align*}
    \Psi\colon \left(0,\tfrac{1}{2}\right) \times \cC  \to \cN_{d, 1}, \quad (\varepsilon, f) \mapsto \widetilde{\Phi}_{\varepsilon, f},
\end{align*}
it follows from (\ref{eq:supinfeps}) and (\ref{eq:quantizedweightspsi}), by application of the triangle inequality, that
$$
    \sup_{f \in \cC}\|f - \Psi(\varepsilon, f)\|_{L^2(\Omega)} \leq \varepsilon \quad \text{ with } \quad \sup_{f \in \cC} \mathcal{M}(\Psi(\varepsilon, f)) \leq M_\varepsilon  \in \mathcal{O}\big(\varepsilon^{-1/\gamma}\big), \, \, \varepsilon \to 0.
$$
The proof is concluded by noting that $\Psi(\varepsilon,f)$ violates Proposition \ref{prop:optimalitynoquant}.
\end{proof}

We conclude this section with a discussion of the conceptual implications of the results established above. Proposition \ref{prop:optimalitynoquant} combined with Lemma \ref{lem:PolynomiallyboundedImpliesFiniteBitLength} establishes that
neural networks achieving uniform approximation error $\varepsilon$ while having weights that are polynomially bounded in $\varepsilon^{-1}$ and depth 
growing polylogarithmically in $\varepsilon^{-1}$ cannot exhibit connectivity growth rate smaller than $\mathcal{O}(\varepsilon^{-1/\gamma^{*}(\mathcal{C})}), \varepsilon \rightarrow 0$; in other words, a decay of the uniform approximation error, as a function of $M$, faster than $\mathcal{O}(M^{-\gamma^{\ast}(\mathcal{C})}), M \rightarrow \infty$, is not possible.

\section{The Transference Principle}\label{sec:bestapprox}

We have seen that a wide array of function classes can be approximated in Kolmogorov-Donoho optimal fashion through dictionaries, provided that
the dictionary $\mathcal{D}$ is chosen to consort with the function class $\mathcal{C}$ according to $\gamma^{\ast,\text{eff}}(\cC,\mathcal{D}) = {\gamma^\ast(\cC)}$.
Examples of such pairs are unit balls in Besov spaces with wavelet bases and unit balls in weighted modulation spaces with Wilson bases.
A more extensive list of optimal pairs is provided in Table~\ref{table_opt_exp}. On the other hand, as shown in \cite{DONOHO1993100}, Fourier bases are strictly suboptimal---in terms of approximation rate---for balls $\cC$ of finite radius in the spaces $BV(\R)$ and $W_p^{m}(\R)$.

In light of what was just said, it is hence natural to let neural networks play the role of the dictionary $\mathcal{D}$ and to ask which function classes $\cC$ are approximated in Kolmogorov-Donoho-optimal fashion by neural networks. Towards answering this question, we next develop a general framework for transferring results on function approximation through dictionaries to results on approximation by neural networks. 
This will eventually lead us to a characterization of function classes $\mathcal{C}$ that are
optimally representable by neural networks in the sense of Definition~\ref{def:NNopti}.

We start by introducing the notion of effective representability of dictionaries through neural networks.

\begin{definition}\label{def:wellrep} Let $d \in \N$, $\Omega\subseteq \mathbb{R}^d$, and $\mathcal{D} = (\varphi_i)_{i\in \N}\subseteq L^2(\Omega)$ be a dictionary. We call
$\mathcal{D}$ \emph{effectively representable by neural networks}, if there exists a bivariate polynomial $\pi$ such that for all $i \in \N$, 
$\eps\in(0,1/2)$, there is a neural network $\Phi_{i,\eps}\in \cN_{d,1}$ satisfying $\mathcal{M}(\Phi_{i,\eps}) \leq \pi(\log(\eps^{-1}), \log(i))$, $\mathcal{B}(\Phi_{i,\eps}) \leq \pi(\eps^{-1}, i)$, and
    \begin{align*}
        \|\varphi_i - \Phi_{i,\eps}\|_{L^2(\Omega)}\le \eps.
    \end{align*}
\end{definition}

The next result will allow us to conclude that optimality---in the sense of Definition~\ref{def:repopti}---of a dictionary $\mathcal{D}$ for a function class $\mathcal{C}$ combined with effective representability of $\mathcal{D}$ by neural networks implies optimal representability of $\mathcal{C}$ by neural networks. The proof is, in essence, effected by noting that every element of the effectively representable $\mathcal{D}$ participating in a best $M$-term-rate achieving approximation $f_M$ of $f\in\cC$ 
can itself be approximated by neural networks well enough for an overall network to approximate $f_M$ with connectivity $M\pi(\log(M))$.
As this connectivity is only polylogarithmically larger than the number of terms $M$ participating in the best $M$-term approximation $f_M$, we will be able to conclude that the optimal approximation rate, indeed, transfers from approximation in $\mathcal{D}$ to approximation in neural networks. The conditions on 
$\mathcal{M}(\Phi_{i,\eps})$ and $\mathcal{B}(\Phi_{i,\eps})$ in Definition \ref{def:wellrep} guarantee precisely that the connectivity increase is at most by a polylogarithmic factor. To see this, we first recall that effective best $M$-term approximation has a polynomial depth search constraint, which implies that the indices $i$ under consideration are upper-bounded by a polynomial in $M$. In addition, the approximation error behavior we are interested in is $\eps=M^{-\gamma}$. Combining these two insights, it follows that $\mathcal{M}(\Phi_{i,\eps}) \leq \pi(\log(\eps^{-1}), \log(i))$ implies polylogarithmic (in $M$) connectivity for each network $\Phi_{i,\eps}$ and hence connectivity $M\pi(\log(M))$ for the overall network realizing $f_M$, as desired. By the same token, 
$\mathcal{B}(\Phi_{i,\eps}) \leq \pi(\eps^{-1}, i)$ guarantees that the weights of $\Phi_{i,\eps}$ are polynomial in $M$.

There is another aspect to effective representability by neural networks that we would like to illustrate by way of example, namely that of
ordering the dictionary elements. Specifically, we consider, for $d=1$ and $\Omega=[-\pi,\pi)$, the class $\cC$ of real-valued even functions in $\cC=L^2(\Omega)$, and take the dictionary as $\mathcal{D}=\{\cos(ix),i\,\in\,\N_0\}$. As the index $i$ enumerating the dictionary elements corresponds to frequencies, the basis functions in $\mathcal{D}$ are hence ordered 
according to increasing frequencies.  Next, note that the parameter $a$ in Theorem~\ref{sin} corresponds to the frequency index $i$ in our example. As
the network $\Psi_{a,D,\epsilon}$ in Theorem \ref{sin} is of finite width, it hence follows, upon replacing $a$ in the expression for $\mathcal{L}(\Psi_{a,D,\epsilon})$ by $i$, that $\mathcal{M}(\Psi_{i,D,\epsilon}) \leq \pi(\log(\eps^{-1}), \log(i))$. The condition on the weights for effective representability is satisfied trivially, simply as $\mathcal{B}(\Psi_{i,D,\eps}) \le 1 \leq \pi(\eps^{-1}, i)$.

We are now ready to state the rate optimality transfer result.

\begin{theorem}\label{theo:EncodeOfNeuralNetworks}Let $d\in \N$, $\Omega\subseteq \mathbb{R}^d$ be bounded, and consider the compact function class $\mathcal{C}\subseteq L^2(\Omega)$.
 Suppose that
  the dictionary $\mathcal{D}=(\varphi_i)_{i\in \mathbb{N}}\subseteq L^2(\Omega)$ is effectively representable by neural networks.
    Then, for every $\gamma \in (0, \gamma^{\ast, \text{eff}}(\mathcal{C},\mathcal{D}))$, there exist a polynomial $\pi$ and a map
   \begin{align*}
        \Psi: \left(0,\tfrac{1}{2}\right) \times \mathcal{C} \to \cN_{d, 1},
    \end{align*}
    such that for all $f\in \cC$, $\eps\in(0,1/2)$, the network $\Psi (\varepsilon,f)$ has $(\lceil\pi(\log(\eps^{-1}))\rceil,\eps)$-quantized weights while satisfying
    $\|f - \Psi(\varepsilon,f)\|_{L^2(\Omega)}\le \varepsilon$, $\L(\Psi(\varepsilon,f))\leq\pi(\log(\eps^{-1}))$, $\mathcal{B}(\Psi(\varepsilon,f))\leq\pi(\eps^{-1})$, and we have
    \begin{equation}
    \label{Mscaling}
    \mathcal{M}(\Psi(\varepsilon,f)) \in \mathcal {O}(\varepsilon^{-1/\gamma}),\ \varepsilon \rightarrow 0,
    \end{equation}
    with the implicit constant in (\ref{Mscaling}) being independent of $f$. In particular, it holds that
    \begin{align*}
        \gamma_{\cN}^{\ast, \text{eff}}(\mathcal{C})\geq\gamma^{\ast, \text{eff}}(\mathcal{C},\mathcal{D}).
    \end{align*}
\end{theorem}

\begin{remark}
Theorem~\ref{theo:EncodeOfNeuralNetworks} allows us to draw the following conclusion. If $\mathcal{D}$ optimally represents the function class $\cC$ in the sense of Definition~\ref{def:repopti}, i.e., 
$\gamma^{\ast, \text{eff}}(\cC,\mathcal{D}) = {\gamma^\ast(\cC)}$, and if it is, in addition, effectively representable by neural networks in the sense of Definition~\ref{def:wellrep}, then, due to Theorem \ref{thm:EffRepNN}, which states that $\gamma_{\cN}^{\ast, \text{eff}}(\mathcal{C}) \le {\gamma^\ast(\cC)}$, we have $\gamma_{\cN}^{\ast, \text{eff}}(\mathcal{C})={\gamma^\ast(\cC)}$ and hence $\cC$ is optimally representable by neural networks in the sense of Definition~\ref{def:NNopti}.
\end{remark}

\begin{proof}[Proof of Theorem~\ref{theo:EncodeOfNeuralNetworks}]
    Let $\gamma'\in(\gamma,\gamma^{\ast, \text{eff}}(\cC,\mathcal{D}))$. 
    According to Definition~\ref{def:polydepth}, there exist a constant $C\geq 1$ and a polynomial $\pi_1$, such that for every $f\in \cC$, $M\in\N$, there is an index set $I_{f,M}\subseteq \{1,\dots , \pi_1(M)\}$ of cardinality $M$ and coefficients $(c_{i})_{i\in I_{f,M}}$ with $|c_{i}|\le \pi_1(M)$, such that 
    \begin{equation}\label{ENN4}
         \left\|f - \sum_{i \in I_{f,M}} c_{i} \varphi_i\right\|_{L^2(\Omega)} \le \frac{C M^{-\gamma'}}{2}.
    \end{equation}
Let $A:=\max\{1,|\Omega|^{1/2}\}$.
Effective representability of $\mathcal{D}$ according to Definition~\ref{def:wellrep} ensures the existence of a bivariate polynomial $\pi_2$ such that for all $M\in\N$, $i \in I_{f,M}$, there is a neural network $\Phi_{i,M} \in \cN_{d, 1}$ satisfying
        \begin{align}\label{ENN2} 
        \left \|\varphi_i - \Phi_{i,M}\right\|_{L^2(\Omega)} \le \tfrac{C}{4A\pi_1(M)}M^{-(\gamma'+1)}
        \end{align}
with
\begin{align}\begin{split}
\mathcal{M}(\Phi_{i,M}) &\leq \pi_2  \left(\log \left(\left(\tfrac{C}{4A\pi_1(M)}M^{-(\gamma'+1)}\right)^{-1}\right),\log(i)\right)\\
&=\pi_2\left((\gamma'+1)\log(M)+\log\left(\tfrac{4A\pi_1(M)}{C}\right),\log(i)\right), \label{ENN1a}\\
\mathcal{B}(\Phi_{i,M}) &\leq  \pi_2 \left(\left(\tfrac{C}{4A\pi_1(M)}M^{-(\gamma'+1)}\right)^{-1},i \right)=\pi_2 \left(\tfrac{4A\pi_1(M)}{C}M^{\gamma'+1},i \right). 
\end{split}
\end{align}
Consider now for $f\in\mathcal{C}$, $M\in\N$  the networks given by 
\begin{align*}
    \Psi_{f,M}(x):=\sum_{i\in I_{f,M}}c_{i}\Phi_{i,M}(x).
\end{align*}
Due to $\max(I_{f,M})\leq\pi_1(M)$, \eqref{ENN1a} and Lemma~\ref{lem:shared_input_lc} imply the existence of a polynomial $\pi_3$ such that $\L(\Psi_{f,M})\leq \pi_3(\log(M))$, $\M(\Psi_{f,M})\leq M\pi_3(\log(M))$, and $\cB(\Psi_{f,M})\leq\pi_3(M)$, for all $f\in\cC$, $M\in\N$, and, owing to \eqref{ENN2}, we get
\begin{align}\label{ENN5}
   \left\|\Psi_{f,M}-\sum_{i \in I_{f,M}} c_{i} \varphi_i  \right\|_{L^2(\Omega)} 
   \leq\sum_{i\in I_{f,M}}|c_{i}|\tfrac{C}{4A\pi_1(M)}M^{-(\gamma'+1)}
   \leq\tfrac{CM^{-\gamma'}}{4A}\sum_{i=1}^{|I_{f,M}|}\tfrac{\max_{i\in I_{f,M}} \! |c_{i}|}{M\pi_1(M)}\leq\tfrac{CM^{-\gamma'}}{4A}.
\end{align}
Lemma \ref{lem:PolynomiallyboundedImpliesFiniteBitLength} therefore ensures the existence of a polynomial $\pi_4$ such that for all 
$f\in\cC$, $M\in\N$, there is a network $\widetilde{\Psi}_{f,M}\in\cN_{d,1}$ with $(\lceil\pi_4(\log(\frac{4A}{C}M^{\gamma'}))\rceil,\frac{CM^{-\gamma'}}{4A})$-quantized weights satisfying $\L(\widetilde{\Psi}_{f,M})=\L(\Psi_{f,M})$, $\M(\widetilde{\Psi}_{f,M})=\M(\Psi_{f,M})$, 
$\mathcal{B}(\widetilde{\Psi}_{f,M})\leq\mathcal{B}(\Psi_{f,M})+\tfrac{CM^{-\gamma'}}{4A}$, and
\begin{align}\label{ENN3}
            \left\|\Psi_{f,M} - \widetilde{\Psi}_{f,M}\right\|_{L^\infty(\Omega)} \leq \tfrac{CM^{-\gamma'}}{4A}.
\end{align} 
As $\Omega$ is bounded by assumption, we have
\begin{align}\label{psiomega}
    \left\|\Psi_{f,M} - \widetilde{\Psi}_{f,M}\right\|_{L^2(\Omega)}
    \leq|\Omega|^{\frac{1}{2}}\left\|\Psi_{f,M} - \widetilde{\Psi}_{f,M}\right\|_{L^\infty(\Omega)} \leq \tfrac{CM^{-\gamma'}}{4},
\end{align}
for all $f\in\cC$, $M\in\N$.
Combining (\ref{psiomega}) with \eqref{ENN4} and \eqref{ENN5}, we get, for all $f\in\mathcal{C}$, $M\in\N$,
    \begin{align}\begin{split} \label{eq:ErrorBoundedByEpsM} 
            \left\|f - \widetilde{\Psi}_{f,M}\right\|_{L^2(\Omega)} &\leq \left\|f - \sum_{i \in I_{f,M}} c_{i} \varphi_i \right\|_{L^2(\Omega)} \!\!\!+ \left\|\sum_{i \in I_{f,M}} c_{i} \varphi_i - \Psi_{f,M} \right\|_{L^2(\Omega)}
            \!\!\! +  \left\|\Psi_{f,M} - \widetilde{\Psi}_{f,M}\right\|_{L^2(\Omega)}\\
             &\leq CM^{-\gamma'}.
    \end{split}\end{align}
    For $\varepsilon \in (0, 1/2)$ and $f\in\cC$, we now set $ M_\varepsilon := \left\lceil(C/\eps)^{1/\gamma'} \right \rceil$ and
    $$
        \Psi (\varepsilon,f):= \widetilde{\Psi}_{f,M_{\eps}}.
    $$
    Thus, \eqref{eq:ErrorBoundedByEpsM} yields
    \begin{align*}
        \left\|f - \Psi (\varepsilon,f)\right\|_{L^2(\Omega)} \leq CM_\eps^{-\gamma'}\leq \varepsilon.
    \end{align*}
    Next, we note that, for all polynomials $\pi$ and $0\leq m < n$,
        \begin{align*}
        \mathcal{O}(\eps^{-m}\pi(\log(\eps^{-1})))\subseteq \mathcal{O}(\eps^{-n}), \, \eps\to 0.    
    \end{align*}
    As $1/\gamma'<1/\gamma$, this establishes
    \begin{align}\label{ENN6}
     \M(\Psi (\varepsilon,f))\in \mathcal{O}(M_\varepsilon \pi_3(\log(M_\eps)))\subseteq \mathcal{O}(\eps^{-1/\gamma}), \, \eps\to 0.
    \end{align}
    Since $M_\eps$ and $\pi_3$ are independent of $f$, the implicit constant in (\ref{ENN6}) does not depend on $f$. 

Next, note that, in general, an $(n,\eta)$-quantized network is also $(m,\delta)$-quantized for $n\geq m$ and $\eta\leq\delta$, simply as
\begin{align*}
   2^{-m\lceil\log(\delta^{-1})\rceil}\Z\cap[-\delta^{-m},\delta^{-m}]
   \subseteq
   2^{-n\lceil\log(\eta^{-1})\rceil}\Z\cap[-\eta^{-n},\eta^{-n}].
\end{align*}
Since $\frac{CM_\eps^{-\gamma'}}{4A}\leq\eps$ this ensures the existence of a polynomial $\pi$ such that, for every $f\in\mathcal{C}$, $\eps\in(0,1/2)$, the network $\Psi(\eps,f)$ is $(\lceil \pi(\log(\eps^{-1}))\rceil,\eps)$-quantized, $\cL(\Psi(\eps,f))\leq \pi(\log(\eps^{-1}))$, and $\cB(\Psi(\eps,f))\leq\pi(\eps^{-1})$.
With \eqref{ENN6} this establishes the first claim of the theorem. In order to verify the second claim, note that $\Psi(\varepsilon,f)\in\cN^{\pi}_{\M(\Psi(\eps,f)),d,1}$, for all $f\in\mathcal{C}$, $\eps\in(0,1/2)$, which implies 
\begin{align*} 
\sup_{f \in \mathcal{C}}\, \inf_{\Phi \in \cN^{\pi}_{M, d, 1}}  \|f    - \Phi\|_{L^2(\Omega)} \in  
\mathcal{O}(M^{-\gamma}), \, M \rightarrow \infty.
\end{align*}
Therefore, owing to Definition \ref{def:repoptiNN}, we get
\begin{align*}
    \gamma_{\cN}^{\ast, \text{eff}}(\mathcal{C}) \geq \gamma^{\ast, \text{eff}}(\mathcal{C},\mathcal{D}),
\end{align*}
 which concludes the proof.
\end{proof}

\begin{remark}
We note that Theorem \ref{theo:EncodeOfNeuralNetworks} continues to hold for $\Omega=\R^n$ if the elements of $\mathcal{D}=(\varphi_i)_{i \in \N}$ are compactly supported with the size of their support sets growing no more than polynomially in $i$. The technical elements required to show this can be found in the context of the approximation of Gabor dictionaries in the proof of
\Cref{GaborSystemsRep}, but are omitted here for ease of exposition.
\end{remark}

The last piece needed to complete our program is to establish that the conditions in Definition~\ref{def:wellrep} guaranteeing effective representability in neural networks are, indeed, satisfied by a wide variety of dictionaries. 

Inspecting Table \ref{table_opt_exp}, we can see that all
example function classes provided therein are optimally represented either by affine dictionaries, i.e., wavelets, the Haar basis, and curvelets or Weyl-Heisenberg dictionaries, namely Fourier bases and Wilson bases. 
The next two sections will be devoted to proving effective representability of affine dictionaries and Weyl-Heisenberg dictionaries by neural networks, thus allowing us to draw the conclusion that neural networks are universally Kolmogorov-Donoho optimal approximators for all function classes listed in Table~\ref{table_opt_exp}.

\section{Affine Dictionaries are Effectively Representable by Neural Networks}\label{sec:optimalapprox}

The purpose of this section is to establish that \emph{affine dictionaries}, including wavelets \cite{Dau92}, ridgelets \cite{candes:1998Ridgelets}, curvelets \cite{CD02}, shearlets \cite{GKL06}, 
$\alpha$-shearlets and more generally $\alpha$-molecules \cite{GroKKS2016alphaMolecules}, which contain all aforementioned dictionaries as special cases, are effectively representable by neural networks. Due to
Theorem~\ref{theo:EncodeOfNeuralNetworks} and Theorem~\ref{thm:EffRepNN}, this will then allow us to conclude that any function class that is optimally representable---in the sense of Definition \ref{def:repopti}---by an affine dictionary with a suitable generator function is optimally representable by neural networks in the sense of Definition~\ref{def:NNopti}.
By ``suitable'' we mean that the generator function can be approximated well by ReLU networks in a sense to be made precise below.

In order to elucidate the main ideas underlying the general definition of affine dictionaries that are effectively representable by neural networks, we start with a basic example, namely the Haar wavelet dictionary on the unit interval, i.e., the set of functions 
\begin{align*}
    \psi_{n,k}\colon&[0,1]\mapsto\R,\ x\mapsto 2^{\frac{n}{2}}\psi(2^nx-k), \,\,n \in \N_0,\, k=0,\dots,2^n-1,
\end{align*}
with
\begin{align*}
    \psi\colon\R\to\R,\ x\mapsto\begin{cases}
    1, &  x\in [0,1/2)\\
    -1, & x\in [1/2,1)\\
    0, & \mathrm{else}.
    \end{cases}
\end{align*}
We approximate the piecewise constant mother wavelet $\psi$ through a continuous piecewise linear function realized by a neural network as follows
\begin{align*}
    \Psi_\delta(x):=
    \tfrac{1}{2\delta}\rho(x+\delta)
    -\tfrac{1}{2\delta}\rho(x-\delta)
    -\tfrac{1}{\delta}\rho(x-(\tfrac{1}{2}-\delta))
    +\tfrac{1}{\delta}\rho(x-(\tfrac{1}{2}+\delta))
    +\tfrac{1}{2\delta}\rho(x-(1-\delta))
    -\tfrac{1}{2\delta}\rho(x-(1+\delta))
\end{align*}
and, setting $\delta(\eps):=\eps^2$ for $\eps \in (0,1/2)$, let
\begin{align*}
    \Phi_{n,k,\eps}(x):=2^{\frac{n}{2}}\Psi_{\delta(\eps)}(2^nx-k), \,\,n \in \N_0,\, k=0,\dots,2^n-1.
\end{align*}
The basic idea in the approximation of $\psi$ through $\Psi_{\delta}$ is to let the transition regions around $0,1/2,$ and $1$ shrink, as a function of $\eps$, sufficiently fast for the construction to realize an approximation error of no more than $\eps$.
Now, a direct calculation yields that, indeed, for $\eps \in (0,1/2)$,
\begin{align*}
    \|\psi_{n,k}-\Phi_{n,k,\eps}\|_{L^2([0,1])}\leq\eps.
\end{align*}
Moreover, we have $\cM(\Phi_{n,k,\eps}) = 18$ and $\cB(\Phi_{n,k,\eps})\leq \max\{2^{\frac{n}{2}}\eps^{-2},2^n\}$. In order to establish effective representability by neural networks, we need to order the Haar wavelet dictionary suitably. Specifically, we proceed from coarse to fine scales, i.e., we let $(\varphi_i)_{i\in\N}=\mathcal{D}=\{\mathcal{D}_0,\mathcal{D}_1,\dots\}$, with $\mathcal{D}_n:=\{\psi_{n,k}\mapsto\R\colon k=0,\dots,2^n-1\}$, where the ordering within the $\mathcal{D}_n$ may be chosen arbitrarily. Next, note that for every pair $n\in\N_0$, $k\in\{0,\dots,2^n-1\}$, there exists a unique index $i\in\N$ such that $\varphi_i=\psi_{n,k}=\psi_{n(i),k(i)}$ and, owing to $|\mathcal{D}_n|=2^n$, we have $2^{n(i)}\leq i$. Finally, taking $\Phi_{i,\eps}:=\Phi_{n(i),k(i),\eps}$ and $\pi(a,b):=a^2 b + b + 18$, the conditions in Definition~\ref{def:wellrep} for effective representability by neural networks are readily verified.
A more elaborate example, namely spline wavelets, is considered at the end of this section.

We are now ready to proceed to the general definition of affine dictionaries with canonical ordering.
\pagebreak
\subsection{Affine Dictionaries with Canonical Ordering}

\begin{definition}\label{def:affsys}
    Let $d,S\in \mathbb{N}$, $\delta > 0$, $\Omega\subseteq \mathbb{R}^d$ be bounded, and let $g_s\in L^{\infty}(\R^d)$, $s\in\{1,\dots,S\}$, be compactly supported.
    Furthermore, for $s\in\{1,\dots,S\}$, let $J_s\subseteq\N$ and $A_{s,j} \in \R^{d \times d}$, $j \in J_s$, be full-rank and with eigenvalues bounded below by $1$ in absolute value.
    We define the \emph{affine dictionary} $\mathcal{D}\subseteq L^2(\Omega)$ with generator functions $(g_s)_{s = 1}^S$ as
    \begin{align*}
        \mathcal{D}:=&\left\{g_s^{j,e} := \left( | \! \det(A_{s,j})|^{\frac{1}{2}}g_s(A_{s,j}\cdot - \, \delta e)\right)\big|_{\Omega} \colon \ s\in\{1,\dots,S\},\ e\in \mathbb{Z}^d, \ j\in J_s,\ \mathrm{and}\ g_s^{j,e} \neq 0\right\}.
    \end{align*}
    Moreover, we define the sub-dictionaries
    \begin{align*}
        \mathcal{D}_{s,j}&:=\{g_s^{j,e} \in \mathcal{D}: e\in \mathbb{Z}^d\ \mathrm{and}\ g_s^{j,e} \neq 0\},\quad \mathrm{for}\ j\in J_s, \ s\in\{1,\dots,S\}\\
        \mathcal{D}_j&:=\bigcup_{s\in\{1,\dots,S\}\colon j\in J_s}  \mathcal{D}_{s,j},\quad \mathrm{for}\ j\in\N.
    \end{align*}
    We call an affine dictionary canonically ordered if it is arranged according to 
    \begin{equation}
    \label{eq:canonicalordering}
         (\varphi_i)_{i\in\mathbb{N}}
         =\mathcal{D} 
         =\left(\mathcal{D}_1,\mathcal{D}_2,\dots \right),
    \end{equation}
    where the elements within each $\mathcal{D}_j$ may be ordered arbitrarily, and there exist constants $a,c>0$ such that
    \begin{equation}\label{eq:detgrowth}
        \sum_{k=1}^{j-1}|\det(A_{s,k})| \ge c \|A_{s,j}\|_\infty^a, \,\, \text{ for all } j\in J_s\!\setminus\!\{1\},\,  s\in\{1,\dots,S\}. 
    \end{equation}
    We call an affine dictionary nondegenerate if for every $ j\in J_s$, $s\in\{1,\dots,S\}$, the sub-dictionary $\mathcal{D}_{s,j}$ contains at least one element.
    \end{definition}
Note that for sake of greater generality, we associate possibly different sets $J_s\subseteq\N$ with the generator functions $g_s$ and, in particular, also allow these sets to be finite. The Haar wavelet dictionary example above is recovered as a nondegenerate affine dictionary by taking $d=1$, $\Omega=[0,1]$, $S=1$, $J_s=\N$, $g_1=\psi$, $\delta=1$, $A_{1,j}=2^{j-1}$, $a=1$, $c=1/2$, and noting that nondegeneracy is verified as for scale $j$, the sub-dictionary $\mathcal{D}_{s,j}$ contains $2^{j-1}$ elements. 
Moreover, the weights of the networks approximating the individual Haar wavelet dictionary elements grow linearly in the index of the dictionary elements.
This is a consequence of the weights being determined by the dilation factor $2^n$ and $2^{n(i)} \le i$ due to the ordering we chose.
As will be shown below, morally this continues to hold for general nondegenerate affine dictionaries, thereby revealing what informed our definition of canonical ordering. Besides, our notion of canonical ordering is also inspired by the ordering employed in the tail compactness considerations for Besov spaces and orthonormal wavelet dictionaries as detailed in Appendix~\ref{Besov_tail}. 
We remark that \eqref{eq:detgrowth} constitutes a very weak restriction on how fast the size of dilations may grow; in fact, we are not aware of any affine dictionaries in the literature that would violate this condition. 
Finally, we note that the dilations $A_{s,j}$ are not required to be ordered in ascending size, as was the case in the Haar wavelet dictionary example. Canonical ordering does, however, ensure a modicum of ordering.

\subsection{Invariance to Affine Transformations}

Affine dictionaries consist of dilations and translations of a given generator function. It is therefore important to
understand the impact of these operations on the approximability---by neural networks---of a given function. 
As neural networks realize concatenations of affine functions and nonlinearities, it is clear that translations and dilations
can be absorbed into the first layer of the network and the transformed function should inherit the approximability properties of
the generator function. However, what we will have to understand is how the weights, the connectivity, and the
domain of approximation of the resulting network are impacted. The following result makes this quantitative.

\begin{proposition}\label{prop:affscalinv}
    Let $d\in \mathbb{N}$, $p\in[1,\infty]$, and $f\in L^p(\R^d)$. Assume
    that there exists a bivariate polynomial $\pi$ such that for all $D\in \R_+$, $\varepsilon\in(0,1/2)$, there is a network $\Phi_{D,\varepsilon}\in \cN_{d,1}$ satisfying 
    \begin{equation}\label{eq:nodilapp}
        \|f - \Phi_{D,\varepsilon}\|_{L^p([-D,D]^d)} \leq \varepsilon,
    \end{equation}
    with $\M(\Phi_{D,\eps})\leq\pi(\log(\varepsilon^{-1}),\log(\lceil D \rceil))$. Then, for all full-rank matrices $A\in \mathbb{R}^{d \times d}$, and all $e\in \mathbb{R}^d$, $E\in \R_+$, and $\eta\in(0,1/2)$, there is a network $\Psi_{A,e,E,\eta}\in \cN_{d,1}$ satisfying
       \begin{equation*}
        \left\||\!\det(A)|^{\frac{1}{p}}f(A\cdot-\,e) - \Psi_{A,e,E,\eta}\right\|_{L^p([-E,E]^d)} \leq \eta,
    \end{equation*} 
    with $\M(\Psi_{A,e,E,\eta})\leq\pi'(\log(\eta^{-1}),\log(\lceil F \rceil))$ and $\cB(\Psi_{A,e,E,\eta})\leq\max\{\cB(\Phi_{F,\eta}),|\!\det(A)|^{\frac{1}{p}},\|A\|_\infty,\|e\|_\infty\}$, where $F= d E \|A\|_\infty +\|e\|_\infty$ and $\pi'$ is of the same degree as $\pi$.
\end{proposition}
\begin{proof}
    By a change of variables, we have for every $\Phi \in \cN_{d,1}$,
    \begin{align}\label{eq:weApplyThisInTheFirstEstimate}
        \big\| |\!\det(A)|^{\frac1p}f(A\cdot- \,  e) - |\!\det(A)|^{\frac1p}\Phi(A\cdot - \, e)\big\|_{L^p([-E,E]^d)}
        =\|f-\Phi\|_{L^p(A\cdot [-E,E]^d\, - \, e)}.
    \end{align}
    Furthermore, observe that
   \begin{align}\label{eq:weApplyThisInTheSecondEstimate}
        A\cdot [-E,E]^d-\,e\subseteq \left[-(d E \|A\|_\infty +\|e\|_\infty), (d E \|A\|_\infty +\|e\|_\infty)\right]^d=[-F,F]^d.
    \end{align}
    Next, we consider the affine transformations $W_{A,e}(x):=Ax-e$, $W'_A(x):=|\!\det(A)|^{\frac{1}{p}}x$ as depth-$1$ networks and take $\Psi_{A,e,E,\eta}:= W'_A\circ\Phi_{F,\eta}\circ W_{A,e}$ according to Lemma~\ref{network_conc}.
    Combining \eqref{eq:weApplyThisInTheFirstEstimate} and \eqref{eq:weApplyThisInTheSecondEstimate} yields
    \begin{align*}
             &\big\| |\!\det(A)|^{\frac1p}f(A\cdot-\,e) - \Psi_{A,e,E,\eta}\big\|_{L^p([-E,E]^d)}
              =  \  \left\|f - \Phi_{F,\eta} \right\|_{L^p(A\cdot[-E,E]^d -\,e)} \leq \ \left\|f - \Phi_{F,\eta} \right\|_{L^p(\left[-F, F \right]^d)} \leq\eta.
    \end{align*}
    The desired bounds on $\M(\Psi_{A,e,E,\eta})$ and $\cB(\Psi_{A,e,E,\eta})$ follow directly by construction.
\end{proof}

\subsection{Canonically Ordered Affine Dictionaries are Effectively Representable}

The next result establishes that canonically ordered affine dictionaries with generator functions that can be approximated well by neural networks are effectively representable by neural networks.

\pagebreak
\begin{theorem}\label{thm:affdicopt} Let $d,S\in \N$, $\Omega\subseteq \mathbb{R}^d$ be bounded with nonempty interior, $(g_s)_{s=1}^S\in L^{\infty}(\R^d)$ compactly supported, and $\mathcal{D}=(\varphi_i)_{i\in \mathbb{N}}\subseteq L^2(\Omega)$ a nondegenerate canonically ordered affine dictionary with generator functions $(g_s)_{s=1}^S$. Assume that there exists a polynomial $\pi$ such that, for all $s\in\{1,\dots,S\}$, $\varepsilon\in(0,1/2)$, there is a network $\Phi_{s,\eps}\in\cN_{d,1}$ satisfying 
    \begin{equation}\label{eq:bootstrapbound}
        \|g_s - \Phi_{s,\eps}\|_{L^2(\R^d)} \leq \varepsilon,
    \end{equation}
    with $\M(\Phi_{s,\eps})\leq\pi(\log(\eps^{-1}))$ and $\mathcal{B}(\Phi_{s,\eps})\leq \pi(\eps^{-1})$.
    Then, $\mathcal{D}$ is effectively representable by neural networks.
\end{theorem}

\begin{proof}
By Definition \ref{def:wellrep} we need to establish the existence of a bivariate polynomial $\pi$ such that for each $i\in \mathbb{N}$, $\eta\in(0,1/2)$, there is a network $\Phi_{i,\eta}\in\cN_{d,1}$ satisfying
\begin{align}\label{eq:approxOfDictElements}
           \|\varphi_i- \Phi_{i,\eta}\|_{L^2(\Omega)}\le \eta,
\end{align}
with $\mathcal{M}(\Phi_{i,\eta}) \leq \pi(\log(\eta^{-1}),\log(i))$ and $\mathcal{B}(\Phi_{i,\eta})\leq \pi(\eta^{-1},i)$.
Note that we have
\begin{equation*} 
    \varphi_i = g^{j_i,e_i}_{s_i}=  \left(|\!\det(A_{s_i,j_i})|^{\frac{1}{2}}g_{s_i}(A_{s_i,j_i}\cdot\,-\,\delta e_i)\right)\big|_{\Omega},
\end{equation*}
for $s_i\in \{1,\dots , S\}$, $j_i\in J_{s_i}$, and $e_i\in \mathbb{Z}^d$.
In order to devise networks satisfying~\eqref{eq:approxOfDictElements}, we employ Proposition~\ref{prop:affscalinv}, upon noting that, by virtue of \eqref{eq:bootstrapbound}, the networks $\Phi_{s,\eps}$ satisfy \eqref{eq:nodilapp} with $p=2$, $f=g_s$, for every $D\in\R_+$. 
Consequently Proposition~\ref{prop:affscalinv} yields a connectivity bound that is even slightly stronger than needed, as it is independent of $i$.
It remains to ensure that the desired bound on $\cB(\Phi_{i,\eta})$ holds. This is the case for $\|A_{s_i,j_i}\|_\infty$ and $\|e_i\|_\infty$ both bounded polynomially in $i$.
In order to verify this, we first bound $\|e_i\|_\infty$ relative to $\|A_{s_i,j_i}\|_\infty$.
As the generators $(g_s)_{s=1}^S$ are compactly supported by assumption, there exists $E\in\R_+$ such that, for every $s\in\{1,\dots,S\}$, the support of $g_s$ is contained in $[-E,E]^d$. We thus get, for all $s\in \{1, \dots, S \}$, $j\in J_s$, and $e \in \Z^d$, that
\begin{align*}
    \|\delta e\|_\infty \geq \sup_{x\in\Omega}\|A_{s,j}x\|_\infty+E \implies g_s^{j,e}(x)=0, \, \forall x\in\Omega \implies g_s^{j,e}\notin\mathcal{D}_j.
\end{align*}
Since $\Omega$ is bounded by assumption, there hence exists a constant $c=c(\Omega,(g_s)_{s = 1}^S, \delta, d)$ such that, for all $s\in \{1, \dots, S \}$, $j\in J_s$, and $e \in \Z^d$, we have
\begin{align*}
g_s^{j,e} \in \mathcal{D}_j &\implies \|e\|_\infty \leq c \|A_{s,j}\|_\infty.
\end{align*}
It remains to show that $\|A_{s_i,j_i}\|_\infty$ is polynomially bounded in $i$.
We start by claiming that, for every $s\in\{1,\dots,S\}$, there is a constant $c_s := c_s(\Omega, \delta, d) > 0$ such that
\begin{align}\label{eq:WeNeedThisAssumption}
    |\det(A_{s,j})|\leq c_s|\mathcal{D}_{s,j}|, \text{ for all } 
    j\in J_s.
\end{align}
To verify this claim, first note that $|\mathcal{D}_{s,j}|\geq 1$, for all $s\in\{1,\dots,S\}, j\in J_s$, owing to the nondegeneracy condition. Thus, for
every $s\in\{1,\dots,S\}$, $j\in J_s$, there exist $x_0\in\Omega$ and $e_0\in\Z^d$ such that $g^{j,e_0}_s(x_0)\neq 0$, which implies
\begin{align*}
    g^{j,e}_{s}(x_0+A_{s,j}^{-1}\delta(e-e_0))=|\det(A_{s,j})|^{\frac{1}{2}}g_s(A_{s,j}x_0-\delta e_0)= g^{j,e_0}_s(x_0)\neq 0.
\end{align*}
We can therefore conclude that $x_0+A_{s,j}^{-1}\delta(e-e_0)\in\Omega$ implies $g^{j,e}_s\in\mathcal{D}_{s,j}$. Consequently, we have
\begin{align*}
    |\mathcal{D}_{s,j}|\geq|\{e\in\Z^d \colon x_0+A_{s,j}^{-1}\delta(e-e_0)\in\Omega\}|
    =|\{e\in\Z^d\colon A_{s,j}^{-1}\delta e\in\Omega-x_0\}|
    = |\Z^d\cap \tfrac{1}{\delta}A_{s,j}(\Omega-x_0)|.
\end{align*}
As $\Omega$ was assumed to have nonempty interior, there exists a constant $C=C(\Omega)$ such that 
\begin{align*}
    |\Z^d\cap \tfrac{1}{\delta}A_{s,j}(\Omega-x_0)|\geq C\, {\normalfont vol}\left( \tfrac{1}{\delta}A_{s,j}(\Omega-x_0)\right) = C\,\delta^{-d}|\det(A_{s,j})|\,{\normalfont vol}(\Omega).
\end{align*}
We have hence established the claim \eqref{eq:WeNeedThisAssumption}.
Combining \eqref{eq:detgrowth} and \eqref{eq:WeNeedThisAssumption}, we obtain, for all 
$s_i\in\{1,\dots,S\}$, $j\in J_s\!\setminus\!\{1\}$,
\begin{align*}
    c\|A_{s_i,j_i}\|_\infty^a\leq\sum_{k=1}^{j_i-1}|\det(A_{s_i,k})|\leq c_{s_i}\sum_{k=1}^{j_i-1}|\mathcal{D}_{k,s_i}|\leq c_s i,
\end{align*}
where the last inequality follows from the fact that $\varphi_i \in \mathcal{D}_{j_i,s_i}$ and hence its index $i$ must be larger than the number of elements contained in preceding sub-dictionaries.
This ensures that 
\begin{align*}
    \|A_{s_i,j_i}\|_\infty \le
    \left(\frac{1}{c} \max_{s=1,\dots,S}c_s \right)^{\frac{1}{a}} i^{\frac{1}{a}}+ \max_{s=1,\dots,S}\|A_{s,1}\|_\infty, \quad \text{for all}\,\,i\in\N,
\end{align*}
thereby completing the proof.
\end{proof}

\begin{remark}
Theorem \ref{thm:affdicopt} is restricted, for ease of exposition, to bounded $\Omega$ and compactly supported generator functions $g_s$.
The result can be extended to $\Omega=\R^d$ and to generator functions $g_s$ of unbounded support but sufficiently fast decay. This extension requires additional technical steps and an alternative definition of canonical ordering.
For conciseness we do not provide the details here, but instead refer to the proofs of Theorems \ref{GaborSystemsRep} and \ref{NNGaussian}, which 
deal with the corresponding technical aspects in the context of approximation
of Gabor dictionaries by neural networks. 
\end{remark}

We can now put the results together to conclude a remarkable universality and optimality property of neural networks:
Consider an affine dictionary generated by functions $g_s$ that can be approximated well by neural networks. If this dictionary provides Kolmogorov-Donoho-optimal approximation for a given function class, then so do neural networks.

\begin{theorem}
    \label{thm:optitransgeneral}  Let $d,S\in \N$, $\Omega\subseteq \mathbb{R}^d$ be bounded with nonempty interior, $(g_s)_{s=1}^S\in L^{\infty}(\R^d)$ compactly supported, and $\mathcal{D}=(\varphi_i)_{i\in \mathbb{N}}\subseteq L^2(\Omega)$ a nondegenerate canonically ordered affine dictionary with generator functions $(g_s)_{s=1}^S$. Assume that there exists a polynomial $\pi$ such that, for all $s\in\{1,\dots,S\}$, $\varepsilon\in(0,1/2)$, there is a network $\Phi_{s,\eps}\in\cN_{d,1}$ satisfying $\|g_s - \Phi_{s,\eps}\|_{L^2(\R^d)} \leq \varepsilon$
    with $\M(\Phi_{s,\eps})\leq\pi(\log(\eps^{-1}))$ and $\mathcal{B}(\Phi_{s,\eps})\leq \pi(\eps^{-1})$. Then, we have
      \begin{align*}
        \gamma_{\cN}^{\ast, \text{eff}}(\mathcal{C})\geq\gamma^{\ast, \text{eff}}(\mathcal{C},\mathcal{D})
    \end{align*}
   for all compact function classes $\cC\subseteq L^2(\Omega)$.
    In particular, if $\mathcal{C}$ is optimally representable by $\mathcal{D}$ (in the sense of Definition \ref{def:repopti}), then $\mathcal{C}$ is optimally representable by neural networks (in the sense of Definition \ref{def:NNopti}).
\end{theorem}

\begin{proof}
    The first statement follows from Theorem~\ref{theo:EncodeOfNeuralNetworks} and Theorem~\ref{thm:affdicopt}, the second from Theorem~\ref{thm:EffRepNN}.
\end{proof}

\subsection{Spline wavelets}

We next particularize the results developed above to show that
neural networks Kolmogorov-Donoho optimally represent all function classes $\cC$ that are optimally representable by spline wavelet dictionaries.
As spline wavelet dictionaries have B-splines as generator functions, we start by showing how B-splines can be realized through neural networks. 
For simplicity of exposition, we restrict ourselves to the univariate case throughout.
\begin{definition}
Let $N_1 := \chi_{[0,1]}$ and for $m\in\N$, define
\begin{align*}
  \quad N_{m+1} := N_1 * N_{m},
  \end{align*}
where $ * $ stands for convolution.
We refer to $N_m$ as the \emph{univariate cardinal B-spline of order $m$}.
\end{definition}

Recognizing that B-splines are piecewise polynomial, we can build on Proposition \ref{relu_poly} to get the following statement on the approximation of B-splines by deep neural networks.

\begin{lemma}\label{thm:EffApproxWithSplines}
Let $m \in \N$. There exists a constant $C>0$ such that for all $\eps\in(0,1/2)$, there is a neural network $\Phi_{\varepsilon} \in \cN_{1,1}$ satisfying 
\begin{align*}
\|\Phi_{\varepsilon}-N_{m}\|_{L^{\infty}(\R)} \leq \varepsilon,
\end{align*}
with $\M(\Phi_{\eps})\leq C\log(\eps^{-1})$ and $\mathcal{B}(\Phi_{\eps})\leq 1$.
\end{lemma}

\begin{proof}
The proof is based on the following representation \cite[Eq. 19]{unser:1997Bspline}
\begin{align}\label{spline-rep-unser}
N_{m}(x) = \frac{1}{m!} \sum_{k=0}^{m+1} (-1)^k {m+1\choose k} \rho((x-k)^m).
\end{align}
While $N_m$ is supported on $[0,m]$, the networks $\Phi_{\eps}$ can have support outside $[0,m]$ as well. We only need to ensure that
$\Phi_{\eps}$ is ``close'' to $N_m$ on $[0,m]$ and at the same time ``small'' outside the interval $[0,m]$. To accomplish this, we first approximate $N_m$ on the slightly larger domain $[-1,m+1]$ by a linear combination of networks realizing shifted monomials according to (\ref{spline-rep-unser}), and then multiply the resulting network by another one that takes on the value $1$ on $[0,m]$ and $0$ outside of $[-1,m+1]$.
Specifically, we proceed as follows.
Proposition \ref{relu_poly} ensures the existence of a constant $C_1$ such that for all $\eps\in(0,1/2)$, 
there is a network $\Psi_{m+2,\eps}\in\cN_{1,1}$ satisfying 
\begin{align*}
        \|\Psi_{m+2,\eps}(x) - x^m \|_{ L^{\infty}([-(m+2),m+2])} \leq \tfrac{\eps}{4(m+2)},
\end{align*}
with $\M(\Psi_{m+2,\eps})\leq C_1 \log(\eps^{-1})$ and $\mathcal{B}(\Psi_{m+2,\eps})\leq 1$. Note that we did not make the dependence of $\mathcal{M}(\Psi_{m+2,\eps})$ on $m$ explicit as we consider $m$ to be fixed.
Next, let $T_k(x):=x-k$ and observe that $\rho((x-k)^m)$ can be realized as a neural network according to $\rho\circ\Psi_{m+2,\eps}\circ T_k$, where $T_k$ is taken pursuant to Corollary~\ref{cor:matrix_mult}.
Next, we define, for $\eps\in(0,1/2)$, the network 
\begin{align*}
    \widetilde{\Phi}_{\eps}:= \frac{1}{m!} \sum_{k=0}^{m+1} (-1)^k {m+1\choose k}\, \rho\circ\Psi_{m+2,\eps}\circ T_k
\end{align*}
and note that 
\begin{equation*}
   \frac{1}{m!} {m+1\choose k} = \frac{m+1}{k! (m-k+1)!}\leq 2,
\end{equation*}
for $k=0,\dots,m+1$. As $\rho$ is $1$-Lipschitz, we have, for all $\eps\in(0,1/2)$,
\begin{align}\begin{split}\label{Spline_proof_est}
    \|\widetilde{\Phi}_{\varepsilon} - N_{m}\|_{L^{\infty}([-1,m+1])}&\leq \sum_{k=0}^{m+1}\frac{1}{m!} {m+1\choose k} \|\rho\circ\Psi_{m+2,\eps}\circ T_k - \rho\circ T_k^m \|_{ L^{\infty}([-1,m+1])} \\
    &\leq 2\sum_{k=0}^{m+1}\|\Psi_{m+2,\eps}(x) - x^m \|_{ L^{\infty}([-(m+2),m+2])}\leq\tfrac{\eps}{2}.
\end{split}\end{align}
Let now $\Gamma(x):=\rho(x+1)-\rho(x)-\rho(x-m)+\rho(x-(m+1))$, note that $0 \le \Gamma(x) \le 1$, and take $\Phi^{\mathrm{mult}}_{1+\eps/2,\eps/2}$ to be the multiplication network from Lemma~\ref{relu_mult}.
We define $\Phi_\eps:= \Phi^{\mathrm{mult}}_{1+\eps/2,\eps/2} \circ (\widetilde{\Phi}_\eps,\Gamma)$ according to Lemma~\ref{network_conc} and Lemma~\ref{lem:shared_input_para} and note that
\begin{align}\label{Spline_proof_est2}
    \|\Phi_{\eps}-N_m\|_{L^\infty(\R)}\leq
    \|\Phi^{\mathrm{mult}}_{1+\eps/2,\eps/2} \circ (\widetilde{\Phi}_\eps,\Gamma)-\widetilde{\Phi}_\eps\cdot\Gamma\|_{L^\infty([-1,m+1])}
    +\|\widetilde{\Phi}_\eps\cdot\Gamma-N_m\|_{L^\infty([-1,m+1])}
\end{align}
as both $N_m$ and $\Gamma$ vanish outside $[-1,m+1]$ and $\Phi^{\mathrm{mult}}_{1+\eps/2,\eps/2}$ delivers zero whenever at least one of its inputs is zero. Note that the first term on the right-hand-side of (\ref{Spline_proof_est2}) is upper-bounded by $\tfrac{\eps}{2}$ as a consequence of $N_{m}(x) \le 1$ and hence $\widetilde{\Phi}_\eps(x)\leq 1+\tfrac{\eps}{2}$, for $x\in[-1,m+1]$, owing to \eqref{Spline_proof_est}.
For the second term, we split up the interval $[-1,m+1]$ and first note that, for $x\in[0,m]$, $\Gamma(x)=1$, which implies
$\|\widetilde{\Phi}_\eps\cdot\Gamma-N_m\|_{L^\infty([0,m])} = \|\widetilde{\Phi}_\eps-N_m\|_{L^\infty([0,m])} \le \eps/2$, again owing to (\ref{Spline_proof_est}). For $x\in[-1,m+1]\setminus[0,m]$, we have $N_m(x)=0$ and $\Gamma(x)\leq 1$, which yields
\begin{align*}
    |\widetilde{\Phi}_\eps(x)\cdot\Gamma(x)-N_m(x)|\leq|\widetilde{\Phi}_\eps(x)|\leq |\widetilde{\Phi}_\eps(x)-N_m(x)|+|N_m(x)|=|\widetilde{\Phi}_\eps(x)-N_m(x)| \le \eps/2,
\end{align*}
again by (\ref{Spline_proof_est}). In summary, 
\eqref{Spline_proof_est} hence ensures that the second term in \eqref{Spline_proof_est2} is also upper-bounded by $\tfrac{\eps}{2}$ and therefore
$\|\Phi_{\varepsilon}-N_{m}\|_{L^{\infty}(\R)} \leq \varepsilon$.
Combining Lemma~\ref{network_conc}, Proposition~\ref{relu_mult}, Corollary~\ref{cor:matrix_mult}, Lemma~\ref{lem:finite_width_linearcombination}, and Lemma~\ref{lem:shared_input_para} establishes the desired bounds on $\M(\Phi_{D,\eps})$ and $\cB(\Phi_{D,\eps})$. 
\end{proof}

\begin{remark}\label{rem:L2forsplines}
As both $N_m$ and the approximating networks $\Phi_{\varepsilon}$ we constructed in the proof of Lemma~\ref{thm:EffApproxWithSplines} are supported in $[-1,m+1]$, we have $\|\Phi_{\varepsilon}-N_{m}\|_{L^2(\R)}\leq (m+2)^{1/2}\|\Phi_{\varepsilon}-N_{m}\|_{L^\infty(\R)}$, which shows that 
Lemma~\ref{thm:EffApproxWithSplines} continues to hold when the approximation error is measured in $L^2(\R)$-norm, albeit with a different constant $C$.
\end{remark}

We are now ready to introduce spline wavelet dictionaries. For $n,j \in \mathbb{Z}$,
set
\begin{align*}
V_n := \text{clos}_{L^2}\Big(\spann \{ N_m(2^n x - k) : k \in \mathbb{Z}\}\Big),
\end{align*}
where $\text{clos}_{L^2}$ denotes closure with respect to $L^2$-norm. Spline spaces $V_n$, $n \in \mathbb{Z}$, constitute 
a multiresolution analysis \cite{Mallat1989} of $L^2(\mathbb{R})$ according to
\[ \{0\} \subseteq \dots V_{-1} \subseteq V_0 \subseteq V_1 \subseteq \dots \subseteq L^2(\mathbb{R}).\]
Moreover, with the orthogonal complements $(\dots, W_{-1},W_0,W_1, \dots)$ such that $V_{n+1} = V_n \oplus W_n$, where $\oplus$ denotes the orthogonal sum, 
we have
\[ L^2(\mathbb{R}) = V_0 \oplus \bigoplus_{k = 0}^{\infty} W_k.\] 

\begin{theorem}[{\cite[Theorem 1]{chui-1992Wavelet}}] \label{spline-wavelet}
Let $m\in\N$. The $m$-th order spline 
\begin{equation} \label{def_wavelet}
\psi_m(x) = \frac{1}{2^{m-1}} \sum_{j=0}^{2m-2} (-1)^j N_{2m}(j+1) \frac{d^m}{dx^m} N_{2m}(2x-j), \end{equation}
with support $[0,2m-1]$, is a basic wavelet that generates $W_0$ and thereby all the spaces $W_n$, $n \in \mathbb{Z}$. 
Consequently, the set 
\begin{equation}
    \label{spline_system}
    \Wcal_m:=\{ \psi_{k,n}(x) =  2^{n/2} \psi_m(2^n x - k): n \in \mathbb{N}_0, k \in \mathbb{Z}\} \cup \{ \phi_{k}(x) =  N_m(x - k): k \in \mathbb{Z}\}
\end{equation}
is a countable complete orthonormal wavelet basis in $L^2(\mathbb{R})$.
\end{theorem}

Taking $\Omega\subseteq\R$, $S=2$, $J_1=\N$, $J_2=\{1\}$, $A_{1,j}=2^{j-1}$ for $j\in\N$, and $A_{2,1}=1$, we get that    
\begin{align}\begin{split}
    \label{splinesystem_restricted}
    \mathcal{D}\,:=\,& \Big \{ g_s^{j,e}(x):= \Big(|A_j|^{\frac{1}{2}} g_s(A_j\,\cdot\, - \, \delta e)\Big) \Big |_{\Omega} : s\in\{1,2\}, \ e \in \mathbb{Z}, \ j \in J_s, \text{ and }  g_s^{j,e} \neq 0 \Big \}
    =\, \Wcal_m
\end{split}\end{align}
is a nondegenerate canonically ordered affine dictionary with generators $g_1=\psi_m$ and $g_2=N_m$. The canonical ordering condition \eqref{eq:detgrowth} is satisfied with $a=1$ and $c=1/2$. Nondegeneracy follows upon noting that $\supp(\psi_{k,n})=[2^{-n}k,2^{-n}(2m-1+k)]$ and $\supp(N_m(\,\cdot\,-k))=[k,m+k]$, which implies that all sub-dictionaries contain at least one element as required.

We have therefore established the following.

\begin{theorem}
    \label{thm:optitransgeneral_spline} Let $\Omega\subseteq \mathbb{R}$ be bounded and of nonempty interior and $\mathcal{D}=(\varphi_i)_{i\in \mathbb{N}}\subseteq L^2(\Omega)$ a spline wavelet dictionary according to \eqref{splinesystem_restricted} ordered per \eqref{eq:canonicalordering}. 
    Then, all compact function classes $\cC\subseteq L^2(\Omega)$ that are
    optimally representable by $\mathcal{D}$ (in the sense of Definition \ref{def:repopti}) are optimally representable by neural networks (in the sense of Definition \ref{def:NNopti}).
\end{theorem}

\begin{proof}
As the canonical ordering and the nondegeneracy conditions were already verified,
it remains to establish that the generators $\psi_m$ and $N_m$ satisfy the antecedent of Theorem~\ref{thm:affdicopt}. To this end, we first devise an alternative representation of \eqref{def_wavelet}. Specifically,
using the identity \cite[Eq.~2.2]{chui-1992Wavelet} 
\[ \frac{d^m}{dx^m} N_{2m}(x) = \sum_{j=0}^m (-1)^j {m\choose j} N_m(x-j), \]
we get 
\begin{equation} \label{altdef_wavelet} \psi_m(x) = \sum_{n=1}^{3m-1} q_n N_m(2x-n+1),\end{equation}
with
\[ q_n = \frac{(-1)^{n+1}}{2^{m-1}} \sum_{j=0}^m {m\choose j} N_{2m}(n-j).\]
{}
As \eqref{altdef_wavelet} shows that $\psi_m$ is a linear combination of shifts and dilations of $N_m$, combining Lemma~\ref{thm:EffApproxWithSplines} and Remark~\ref{rem:L2forsplines} with Lemma~\ref{network_linearcombination} and Proposition~\ref{prop:affscalinv} ensures that \eqref{eq:bootstrapbound} is satisfied. Application of Theorem~\ref{thm:optitransgeneral} then establishes the claim.
\end{proof}

\section{Weyl-Heisenberg dictionaries}\label{sec:gabor}

In this section, we consider 
Weyl-Heisenberg a.k.a. Gabor dictionaries \cite{grochenig2013foundations}, which consist of time-frequency translates of a given generator function. 
Gabor dictionaries play a fundamental role in time-frequency analysis \cite{grochenig2013foundations} and in the study of partial differential equations \cite{Fefferman83}. We start with the formal definition of Gabor dictionaries.

\begin{definition}[Gabor dictionaries]
  Let $d\in\N$, $f\in L^2(\R^d)$, and $x,\xi\in\R^d$. We define the translation operator $T_x\colon L^2(\R^d)\to L^2(\R^d)$ as
  \begin{align*}
    T_xf(t):=f(t-x)
  \end{align*}
  and the modulation operator $M_{\xi}\colon L^2(\R^d)\to L^2(\R^d,\C)$ as
  \begin{align*}
    M_{\xi}f(t):=e^{2\pi i \langle \xi,t \rangle}f(t).
  \end{align*}
  Let $\Omega\subseteq\R^d$, $\alpha,\beta>0$, and $g\in L^2(\R^d)$. The Gabor dictionary $\G(g,\alpha,\beta,\Omega)\subseteq L^2(\Omega)$ is defined as
  \begin{align*}
    \G(g,\alpha,\beta,\Omega):=\left\{M_{\xi}T_xg\big|_{\Omega}\colon (x,\xi)\in\alpha\Z^d\times\beta\Z^d \right\}.
  \end{align*}
\end{definition}

  In order to describe representability in neural networks in the sense of Definition \ref{def:wellrep}, we need to order the elements in $\G(g,\alpha,\beta,\Omega)$. 
  To this end, let $\G_0(g,\alpha,\beta,\Omega):=\{g\big|_{\Omega}\}$ and define $\G_n(g,\alpha,\beta,\Omega)$, $n\in\N$, recursively according to
  \begin{align*}
    \G_n(g,\alpha,\beta,\Omega):=\{M_{\xi}T_xg\big|_{\Omega} \colon (x,\xi)\in\alpha\Z^d\times\beta\Z^d,  \|x\|_{\infty}\leq n\alpha, \|\xi\|_{\infty}\leq n\beta\}\backslash\bigcup_{k=0}^{n-1}\G_k(g,\alpha,\beta,\Omega).
  \end{align*}
  We then organize $\G(g,\alpha,\beta,\Omega)$ as
  \begin{align}\label{GaborOrdering}
    \G(g,\alpha,\beta,\Omega)=(\G_0(g,\alpha,\beta,\Omega),\, \G_1(g,\alpha,\beta,\Omega),\,\dots),
  \end{align}
  where the ordering within the sets $\G_n(g,\alpha,\beta,\Omega)$ is arbitrary. We hasten to add that the specifics of the overall ordering in (\ref{GaborOrdering}) are irrelevant as long as $\G(g,\alpha,\beta,\Omega)=(\varphi_i)_{i\in\N}$ with $\varphi_i=\M_{\xi(i)}T_{x(i)}g\big|_{\Omega}$ is such that $\|x(i)\|_\infty$ and $\|\xi(i)\|_\infty$ do not grow faster than polynomially in $i$; this will become apparent in the proof of Theorem \ref{GaborSystemsRep}. We note that this ordering is also inspired by that employed in the tail compactness considerations for modulation spaces and Wilson bases as detailed in Appendix~\ref{Modulation_tail}.

As Gabor dictionaries are built from time-shifted and modulated versions of the generator function $g$, and invariance to time-shifts was already
established in Proposition~\ref{prop:affscalinv}, we proceed to showing that the approximation-theoretic properties of the generator function are inherited
by its modulated versions. This result can be interpreted as an invariance property to frequency shifts akin to that established in Proposition~\ref{prop:affscalinv} for affine transformations in the context of affine dictionaries. In summary, neural networks exhibit a 
remarkable invariance property both to the affine group operations of scaling and translation and to the Weyl-Heisenberg group operations of modulation and translation.

\begin{lemma}\label{NNmodulation}
  Let $d\in\N$, $f\in L^2(\R^d)\cap L^\infty(\R^d)$, and for every $D \in \R_+$, $\eps\in(0,1/2)$, let $\Phi_{D,\eps}\in\cN_{d,1}$ satisfy
  \begin{align*}
    \|f-\Phi_{D,\eps}\|_{L^{\infty}([-D,D]^d)}\leq\eps.
  \end{align*}
  Then, there exists a constant $C>0$ (which does not depend on $f$) such that for all $D\in \R_+$, $\eps\in(0,1/2)$, $\xi\in\R^d$, there are networks $\Phi^{\RE}_{D,\xi,\eps},\Phi^{\IM}_{D,\xi,\eps}\in\cN_{d,1}$ satisfying
    \begin{align*}
    \|\RE(M_{\xi}f)-\Phi^{\RE}_{D,\xi,\eps}\|_{L^{\infty}([-D,D]^d)}+\|\IM(M_{\xi}f)-\Phi^{\IM}_{D,\xi,\eps}\|_{L^{\infty}([-D,D]^d)}\leq 3\eps
  \end{align*}
  with
  \begin{align*}
    \L(\Phi^{\RE}_{D,\xi,\eps}),\L(\Phi^{\IM}_{D,\xi,\eps})&\leq C((\log(\eps^{-1}))^2+\log(\lceil d D\|\xi\|_{\infty}\rceil)+(\log(\lceil S_f \rceil))^2)+\L(\Phi_{D,\eps}),\\
    \M(\Phi^{\RE}_{D,\xi,\eps}),\M(\Phi^{\IM}_{D,\xi,\eps})&\leq C((\log(\eps^{-1}))^2+\log(\lceil d D\|\xi\|_{\infty}\rceil)+(\log(\lceil S_f \rceil))^2+d)+4\M(\Phi_{D,\eps})+ 4\L(\Phi_{D,\eps}),
  \end{align*}
  and $\cB(\Phi^{\RE}_{D,\xi,\eps})\leq 1$, where $S_f:=\max\{1,\|f\|_{L^{\infty}(\R^d)}\}$.
\end{lemma}

\begin{proof}
  All statements in the proof involving $\varepsilon$ pertain to $\eps\in(0,1/2)$ without explicitly stating this every time.
  We start by observing that 
    \begin{align*}
    &\RE(M_{\xi}f)(t)=\cos(2\pi\langle \xi,t \rangle)f(t)\\
    &\IM(M_{\xi}f)(t)=\sin(2\pi\langle \xi,t \rangle)f(t)
  \end{align*}
  due to $f \in \R$.
    Note that for given $\xi\in\R^d$, the map $t\mapsto\langle\xi,t\rangle=\xi^Tt=t_1\xi_1+\dots+t_d\xi_d$ is simply a linear transformation. Hence, combining Lemma~\ref{network_conc}, Theorem~\ref{sin}, and Corollary~\ref{cor:matrix_mult} establishes the existence of a constant $C_1$ such that for all $D\in \R_+$, $\xi\in\R^d$, $\eps\in(0,1/2)$, there is a network $\Psi_{D,\xi,\eps}\in\cN_{d,1}$ satisfying      \begin{align}\label{NNmod7}
    \sup_{t \in [-D,D]^d} |\cos(2\pi \langle \xi,t\rangle)-\Psi_{D,\xi,\eps}(t)|\leq\tfrac{\eps}{6S_f}
  \end{align}
  with
   \begin{align}\begin{split}\label{NNmod5}
    \L(\Psi_{D,\xi,\eps})&\leq C_1((\log(\eps^{-1}))^2+(\log(S_f))^2+\log(\lceil 
    dD\|\xi\|_{\infty}\rceil)),\\
    \M(\Psi_{D,\xi,\eps})&\leq C_1((\log(\eps^{-1}))^2+(\log(S_f))^2+\log(\lceil 
    dD\|\xi\|_{\infty}\rceil)+d),
 \end{split}\end{align}
  and $\mathcal{B}(\Psi_{D,\xi,\eps})\leq 1$.
 Moreover, Proposition \ref{relu_mult} guarantees the existence of a constant $C_2>0$ such that for all $\eps\in(0,1/2)$, there is a network $\mu_{\eps}\in\cN_{2,1}$  satisfying
    \begin{align}\label{NNmod3}
     \sup_{x,y\in[-S_f-1/2,S_f+1/2]}|\mu_{\eps}(x,y) - xy|\leq\tfrac{\eps}{6}
  \end{align}
  with
  \begin{align}\label{NNmod6}
    \L(\mu_{\eps}), \M(\mu_{\eps})\leq C_2(\log(\eps^{-1})+\log(\lceil S_f \rceil))
  \end{align}
  and $\mathcal{B}(\mu_\eps)\leq 1$. 
  Using Lemmas~\ref{network_extension} and~\ref{network_parallelization}, we get that the
  network $\Gamma_{D,\xi,\eps}:=(\Psi_{D,\xi,\eps},\Phi_{D,\eps}) \in \cN_{d,2}$ satisfies 
  \begin{align*}
    \L(\Gamma_{D,\xi,\eps})&\le \max \{ \L(\Psi_{D,\xi,\eps}), \L(\Phi_{D,\eps})\},\\
    \M(\Gamma_{D,\xi,\eps})&\leq 2\,\M(\Psi_{D,\xi,\eps}) + 2\,\M(\Phi_{D,\eps}) + 2\,\L(\Psi_{D,\xi,\eps}) + 2\,\L(\Phi_{D,\eps}),
  \end{align*} 
  and $\mathcal{B}(\Gamma_{D,\xi,\eps})\leq 1$.
  Finally, applying Lemma \ref{network_conc} to concatenate the networks $\Gamma_{D,\xi,\eps}$ and $\mu_{\eps}$, we obtain the network
  \begin{align*}
    \Phi^{\RE}_{D,\xi,\eps}:= \mu_{\eps}\circ\Gamma_{D,\xi,\eps} = \mu_{\eps}\circ(\Psi_{D,\xi,\eps}
    ,\Phi_{D,\eps}) 
    \in\mathcal{N}_{d,1}
  \end{align*}
  satisfying
  \begin{align}
  \label{NN_bound2}
      \L(\Phi^{\RE}_{D,\xi,\eps}) &\leq \max \{ \L(\Psi_{D,\xi,\eps}), \L(\Phi_{D,\eps})\} + \L(\mu_{\eps}), \\
  \label{NN_bound1}
      \M(\Phi^{\RE}_{D,\xi,\eps}) &\leq 4\M(\Psi_{D,\xi,\eps}) + 4\M(\Phi_{D,\eps}) + 4\cL(\Psi_{D,\xi,\eps})+4\cL(\Phi_{D,\eps})+2\cM(\mu_\eps),
  \end{align} 
  and $\cB(\Phi^{\RE}_{D,\xi,\eps})\leq 1$.
  Next, observe that \eqref{NNmod7} and \eqref{NNmod3} imply that
  \begin{align*}
    \|\Phi^{\RE}_{D,\xi,\eps}-\RE(M_{\xi}f)\|_{L^\infty([-D,D]^d)}&=\|\mu_{\eps}(\Psi_{D,\xi,\eps}(
    \,\cdot\,),\Phi_{D,\eps}(\,\cdot\,))-\cos(2\pi\langle \xi,\,\cdot\, \rangle)f(\,\cdot\,)\|_{L^\infty([-D,D]^d)}\\
    &\leq \|\mu_{\eps}(\Psi_{D,\xi,\eps}(
    \,\cdot\,),\Phi_{D,\eps}(\,\cdot\,))-\Psi_{D,\xi,\eps}
    (\,\cdot\,)\Phi_{D,\eps}(\,\cdot\,)\|_{L^\infty([-D,D]^d)}\\
    &\quad+\|\Psi_{D,\xi,\eps}(
    \,\cdot\,)\Phi_{D,\eps}(\,\cdot\,)-\cos(2\pi\langle \xi,\,\cdot\, \rangle)f(\,\cdot\,)\|_{L^\infty([-D,D]^d)}\\
    &\leq \|\mu_{\eps}(\Psi_{D,\xi,\eps}(
    \,\cdot\,),\Phi_{D,\eps}(\,\cdot\,))-\Psi_{D,\xi,\eps}
    (\,\cdot\,)\Phi_{D,\eps}(\,\cdot\,)\|_{L^\infty([-D,D]^d)}\\
    &\quad+\|\Psi_{D,\xi,\eps}(
    \,\cdot\,)(\Phi_{D,\eps}(\,\cdot\,)- f(\,\cdot\,))\|_{L^\infty([-D,D]^d)}\\
    &\quad+\|\Psi_{D,\xi,\eps}(
    \,\cdot\,)f(\,\cdot\,)-\cos(2\pi\langle \xi,\,\cdot\, \rangle)f(\,\cdot\,)\|_{L^\infty([-D,D]^d)}\\
        &\leq \tfrac{\eps}{6}+(1+\tfrac{\eps}{6S_f})\eps+\tfrac{\eps}{6} \leq \tfrac{3}{2}\eps.
  \end{align*}
  Combining \eqref{NNmod5}, \eqref{NNmod6}, \eqref{NN_bound1}, and \eqref{NN_bound2} we can further see that there exists a constant $C>0$ such that
  \begin{align*}
    \L(\Phi^{\RE}_{D,\xi,\eps})&\leq C((\log(\eps^{-1}))^2+\log(\lceil 
    d D\|\xi\|_{\infty}\rceil)+(\log(\lceil S_f  
    \rceil))^2)+\L(\Phi_{D,\eps}),\\
    \M(\Phi^{\RE}_{D,\xi,\eps})&\leq C((\log(\eps^{-1}))^2+\log(\lceil  
    d D\|\xi\|_{\infty}\rceil)+(\log(\lceil S_f  
    \rceil))^2+d)+4\M(\Phi_{D,\eps})+ 4\L(\Phi_{D,\eps}),
  \end{align*}
  and $\cB(\Phi^{\RE}_{D,\xi,\eps}))\leq 1$. The results for $\Phi^{\IM}_{D,\xi,\eps}$ follow analogously, simply by using $\sin(x)=\cos(x-\pi/2)$.
\end{proof}

Note that Gabor dictionaries necessarily contain complex-valued functions. The theory developed so far was, however, phrased for neural networks with real-valued outputs.
As is evident from the proof of Lemma \ref{NNmodulation}, this is not problematic when the generator function $g$ is real-valued. For complex-valued generator functions we would need a version of Proposition \ref{relu_mult} that applies to the multiplication of complex numbers. Due to $(a+ib)(a'+ib')=(aa'-bb')+i(ab'+a'b)$ such a network 
can be constructed by realizing the real and imaginary parts of the product as a sum of real-valued multiplication networks and then proceeding as in the proof above.
We omit the details as they are straightforward and would not lead to new conceptual insights. Furthermore, an extension---to the complex-valued case---of the concept of effective representability by neural networks according to 
Definition \ref{def:wellrep} would be needed. This can be effected by considering the set of neural networks with $1$-dimensional complex-valued output as neural networks 
with $2$-dimensional real-valued output, i.e., by setting
\begin{align*}
  \cN^{\C}_{d,1}:=\cN_{d,2},
\end{align*}
with the convention that the first component represents the real part and the second the imaginary part.

We proceed to establish conditions for effective representability of Gabor dictionaries by neural networks. 
\begin{theorem}\label{GaborSystemsRep}
  Let $d\in\N$, $\Omega\subseteq\R^d$, $\alpha,\beta>0$, $g\in L^2(\R^d)\cap L^\infty(\R^d)$, and let $\mathcal{G}(g,\alpha,\beta,\Omega)$ be the corresponding Gabor dictionary with ordering as defined in \eqref{GaborOrdering}. Assume that $\Omega$ is bounded or that $\Omega=\R^d$ and $g$ is compactly supported. 
  Further, suppose that there exists a polynomial $\pi$ such that for every $x\in\R^d$, $\eps\in(0,1/2)$, there is a network $\Phi_{x,\eps}\in\cN_{d,1}$ satisfying 
  \begin{align}\label{GSRass2}
    \|g-\Phi_{x,\eps}\|_{L^{\infty}(x+\Omega)}\leq\eps,
\end{align}
with $\M(\Phi_{x,\eps})\leq\pi(\log(\eps^{-1}),\log(\|x\|_\infty))$, $\mathcal{B}(\Phi_{x,\eps})\leq\pi(\eps^{-1},\|x\|_\infty)$. Then, $\mathcal{G}(g,\alpha,\beta,\Omega)$ is effectively representable by neural networks. 
\end{theorem}

\begin{proof}
 We start by noting that owing to \eqref{GaborOrdering}, we have $\G(g,\alpha,\beta,\Omega)=(\varphi_i)_{i\in\N}$ with $\varphi_i=\M_{\xi(i)}T_{x(i)}g\in\G_{n(i)}(g,\alpha,\beta,\Omega)$, where
  \begin{align}\label{GSR1}
    \|\xi(i)\|_{\infty}\leq n(i)\beta\leq i\beta\quad\mathrm{and}\quad\|x(i)\|_{\infty}\leq n(i)\alpha\leq i\alpha.
  \end{align}
Next, we take the affine transformation $W_x(y):=y-x$ to be a depth-$1$ network and observe that, 
due to \eqref{GSRass2} and Lemma~\ref{network_conc}, we have, for all $x\in\R^d$, $\eps\in(0,1/2)$,
\begin{align}\label{Gabor_shift}
    \|T_x g-\Phi_{-x,\eps}\circ W_x\|_{L^{\infty}(\Omega)}=\|g-\Phi_{-x,\eps}\|_{L^{\infty}(-x+\Omega)}\leq\eps,
\end{align}
with
\begin{align*}
\cM(\Phi_{-x,\eps}\circ W_x) & \leq  2(\pi(\log(\eps^{-1}),\log(\|x\|_\infty))+2d)\\
\cB(\cM(\Phi_{-x,\eps}\circ W_x)) & \leq  \max\{\cB(\Phi_{-x,\eps}),\|x\|_\infty\}\leq\pi(\eps^{-1},\|x\|_\infty)+\|x\|_\infty.
\end{align*}
We first consider the case where $\Omega$ is bounded and let $E\in\R_+$ be such that $\Omega\subseteq[-E,E]^d$.
Combining \eqref{Gabor_shift} with Proposition~\ref{prop:affscalinv} and Lemma~\ref{NNmodulation}, we can infer the existence of a multivariate polynomial $\pi_1$ such that for all $i\in\N$, $\eps\in(0,1/2)$, there is a network $\Phi_{i,\eps}=(\Phi_{i,\eps}^{\text{Re}},\Phi_{i,\eps}^{\text{Im}})\in\cN^{\C}_{d,1}$ satisfying 
  \begin{align}\label{GSR2}
    \|\RE(\M_{\xi(i)}T_{x(i)}g)-\Phi_{i,\eps}^{\RE}\|_{L^\infty(\Omega)}+\|\IM(\M_{\xi(i)}T_{x(i)}g)-\Phi_{i,\eps}^{\IM}\|_{L^\infty(\Omega)}\leq(2E)^{-\frac{d}{2}}\eps,
  \end{align}
  with
  \begin{align}\begin{split}\label{Phi_REIM_est}
    \M(\Phi^{\RE}_{i,\eps}),\M(\Phi^{\IM}_{i,\eps})&\leq \pi_1(\log(\eps^{-1}),\log(\|\xi(i)\|_{\infty}),\log(\|x(i)\|_{\infty})),\\
    \mathcal{B}(\Phi^{\RE}_{i,\eps}),\mathcal{B}(\Phi^{\IM}_{i,\eps})&\leq\pi_1(\eps^{-1},\|\xi(i)\|_{\infty},\|x(i)\|_{\infty}).
  \end{split}
  \end{align}
  Note that here we did not make the dependence of the connectivity and the weight upper bounds on $d$ and $E$ explicit as these quantities are irrelevant for the purposes of what we want to show, as long as they are finite, of course, which is the case by assumption. Likewise, we did not explicitly indicate the dependence of $\pi_1$ on $g$.
As $|z|\leq |\RE(z)|+|\IM(z)|$, it follows from \eqref{GSR2} that for all $i\in\N$, $\eps\in(0,1/2)$,
  \begin{align*}
    \|\varphi_i-\Phi_{i,\eps}\|_{L^2(\Omega,\C)}&\leq(2E)^{\frac{d}{2}} \|\varphi_i-\Phi_{i,\eps}\|_{L^\infty(\Omega,\C)}\\
    &\leq(2E)^{\frac{d}{2}}\left(\|\RE(\varphi_i)-\Phi_{i,\eps}^\RE\|_{L^\infty(\Omega)}+\|\IM(\varphi_i)-\Phi_{i,\eps}^\IM\|_{L^\infty(\Omega)}\right)\leq\eps.
  \end{align*}
  Moreover, \eqref{GSR1} and \eqref{Phi_REIM_est} imply the existence of a polynomial $\pi_2$ such that
  \begin{align*}
      \M(\Phi^{\RE}_{i,\eps}),\M(\Phi^{\IM}_{i,\eps})\leq \pi_2(\log(\eps^{-1}),\log(i)), \quad \mathcal{B}(\Phi^{\RE}_{i,\eps}),\mathcal{B}(\Phi^{\IM}_{i,\eps})\leq \pi_2(\eps^{-1},i),
      \end{align*}
  for all $i\in\N$, $\eps\in(0,1/2)$. We can therefore conclude that $\G(g,\alpha,\beta,\Omega)$ is effectively representable by neural networks.
  
  We proceed to proving the statement for the case $\Omega=\R^d$ and $g$ compactly supported, i.e., there exists $E\in\R_+$ such that $\supp(g)\subseteq[-E,E]^d$. This implies
  \begin{align*}
  \supp(M_\xi T_xg)=\supp(T_x g)\subseteq x+[-E,E]^d\subseteq[-(\|x\|_\infty+E),\|x\|_\infty+E]^d.
  \end{align*}
  Again, combining \eqref{Gabor_shift} with Proposition~\ref{prop:affscalinv} and Lemma~\ref{NNmodulation} establishes the existence of a polynomial $\pi_3$ such that for all $x,\xi\in\R^d$, $\eps\in(0,1/2)$, there are networks $\Psi^{\RE}_{x,\xi,\eps},\Psi^{\IM}_{x,\xi,\eps}\in\cN_{d,1}$ satisfying
    \begin{align}\label{WERx2}
    \|\RE(M_{\xi}T_x g)-\Psi^{\RE}_{x,\xi,\eps}\|_{L^{\infty}(S_x)}+\|\IM(M_{\xi}T_x g)-\Psi^{\IM}_{x,\xi,\eps}\|_{L^{\infty}(S_x)}\leq \tfrac{\eps}{2s_x},
  \end{align} 
  with
  \begin{align*}
  \M(\Psi^{\RE}_{x,\xi,\eps}),\M(\Psi^{\IM}_{x,\xi,\eps})&\leq\pi_3(\log(\eps^{-1}),\log(\|x\|_\infty),\log(\|\xi\|_\infty)),\\ \mathcal{B}(\Psi^{\RE}_{x,\xi,\eps}),\mathcal{B}(\Psi^{\IM}_{x,\xi,\eps})&\leq\pi_3(\eps^{-1},\|x\|_\infty,\|\xi\|_\infty),
  \end{align*}
  where we set $S_x:=[-(\|x\|_\infty+E+1),\|x\|_\infty+E+1]^d$ and $s_x:=|S_x|^{1/2}$ to simplify notation.
  As we want to establish effective representability for $\Omega=\R^d$, the estimate in (\ref{WERx2}) is insufficient. In particular, we have no control over the behavior of the networks
  $\Psi^{\RE}_{x,\xi,\eps},\Psi^{\IM}_{x,\xi,\eps}$
  outside the set $S_x$. We can, however, construct networks which exhibit the same scaling behavior in terms of $\M$ and $\mathcal{B}$, are supported in $S_x$, and realize the same output for all inputs in $S_x$.
  To this end let, for $y\in\R_+$, the network $\alpha_y\in\cN_{1,1}$ be given by
  \begin{align*}
    \alpha_y(t):=\rho(t-(-y-1))-\rho(t-(-y))-\rho(t-y)+\rho(t-(y+1)),\quad t\in\R.
  \end{align*}
  Note that $\alpha_y(t)=1$ for $t\in[-y,y]$, $\alpha_y(t)=0$ for $t\notin[-y-1,y+1]$, and $\alpha_y(t)\in(0,1)$ else. Next, consider, 
  for $x\in\R^d$, the network given by
  \begin{align*}
    \chi_x(t):=\rho\left(\left[\sum_{i=1}^d\alpha_{\|x\|_{\infty}+E}(t_i)\right] -(d-1)\right),\quad t=(t_1,t_2,\dots,t_d)\in\R^d,
  \end{align*}
 and note that
  \begin{align*}
      \chi_x(t)&=1, \quad \forall t\in [-(\|x\|_\infty+E),\|x\|_\infty+E]^d\\ 
      \chi_x(t)&=0, \quad \forall t\notin[-(\|x\|_\infty+E+1),\|x\|_\infty+E+1]^d\\
      0\leq\chi_x(t)&\leq 1, \quad \forall t\in\R^d.
  \end{align*}
  As $d$ and $E$ are considered fixed here, there exists a constant $C_1$ such that, for all $x\in\R^d$, we have $\M(\chi_x)\leq C_1$ and $\mathcal{B}(\chi_x)\leq C_1\max\{1,\| x\|_\infty\}$. Now, let $B:=\max\{1,\|g\|_{L^\infty(\R)}\}$. Next, by Proposition \ref{relu_mult} there exists a constant $C_2$ such that, for all $x\in\R^d$, $\eps\in(0,1/2)$, there is a network $\mu_{x,\eps}\in\cN_{1,1}$ satisfying
  \begin{align}\label{WERx4}
  \sup_{y,z\in[-2B,2B]}|\mu_{x,\eps}(y,z)-yz|\leq\tfrac{\eps}{4s_x},
  \end{align}
  and, for all $y\in\R$, 
  \begin{align}\label{WERx8}
  \mu_{x,\eps}(0,y)=\mu_{x,\eps}(y,0)=0,
  \end{align}
  with $\M(\mu_{x,\eps})\leq C_2(\log(\eps^{-1})+\log(s_x))$ and $\mathcal{B}(\mu_{x,\eps})\leq 1$. Note that in the upper bound on $\M(\mu_{x,\eps})$, we did not make the dependence on $B$ explicit as we consider $g$ fixed for the purposes of the proof. Next, as $E$ is fixed, there exists a constant $C_3$ such that $\M(\mu_{x,\eps})\leq C_3(\log(\eps^{-1})+\log(\|x\|_\infty+1))$, for all $x\in\R^d$, $\eps\in(0,1/2)$.
  
  We now take 
  \begin{align*}
      \Gamma^{\RE}_{x,\xi,\eps}:=\mu_{x,\eps}\circ(\Psi^{\RE}_{x,\xi,\eps},\chi_x)\quad\mathrm{and}\quad
      \Gamma^{\IM}_{x,\xi,\eps}:=\mu_{x,\eps}\circ(\Psi^{\IM}_{x,\xi,\eps},\chi_x)
  \end{align*}
  according to Lemmas~\ref{network_parallelization} and~\ref{network_conc}, which ensures the existence of a polynomial $\pi_4$ such that, for all $x,\xi\in\R^d$, $\eps\in(0,1/2)$, 
  \begin{align}\begin{split}\label{MB-bounds}
      \M(\Gamma^{\RE}_{x,\xi,\eps}),\M(\Gamma^{\IM}_{x,\xi,\eps})&\leq\pi_4(\log(\eps^{-1}),\log(\|x\|_\infty),\log(\|\xi\|_\infty)),\\
      \mathcal{B}(\Gamma^{\RE}_{x,\xi,\eps}),\mathcal{B}(\Gamma^{\IM}_{x,\xi,\eps})&\leq\pi_4(\eps^{-1},\|x\|_\infty,\|\xi\|_\infty). 
      \end{split}
  \end{align}
  Furthermore,
  \begin{align}
  \begin{split}\label{WER_Gamma_Est}
      \|\Gamma^{\RE}_{x,\xi,\eps} - \RE(M_{\xi}T_x g)\|_{L^{\infty}(S_x)}
      &\leq\|\mu_{x,\eps}\circ(\Psi^{\RE}_{x,\xi,\eps},\chi_x)- \Psi^{\RE}_{x,\xi,\eps}\cdot\chi_x\|_{L^{\infty}(S_x)}\\
      &\quad+ \|\Psi^{\RE}_{x,\xi,\eps}\cdot\chi_x - \RE(M_{\xi}T_x g)\|_{L^{\infty}(S_x)},
      \end{split}
  \end{align}
  where the first term is upper-bounded by $\tfrac{\eps}{4s_x}$ due to \eqref{WERx4}. The second term on the right-hand side of (\ref{WER_Gamma_Est}) is upper-bounded as follows. First, note that for $t\in S_x\setminus [-(\|x\|_\infty+E),\|x\|_\infty+E]^d$, we have $\RE(M_{\xi}T_x g)(t)=0$ and $|\chi_x(t)|\leq 1$, which implies 
  \begin{align*}
      |\Psi^{\RE}_{x,\xi,\eps}(t)\cdot\chi_x(t) - \RE(M_{\xi}T_x g)(t)|&\leq |\Psi^{\RE}_{x,\xi,\eps}(t)| \leq  |\Psi^{\RE}_{x,\xi,\eps}(t) - \RE(M_{\xi}T_x g)(t)| +  |\RE(M_{\xi}T_x g)(t)|\\
      &=  |\Psi^{\RE}_{x,\xi,\eps}(t) - \RE(M_{\xi}T_x g)(t)|.
  \end{align*}
  As $|\chi_x(t)|=1$ for $t\in[-(\|x\|_\infty+E),\|x\|_\infty+E]^d$, together with \eqref{WER_Gamma_Est}, this yields
  \begin{align*}
      \|\Gamma^{\RE}_{x,\xi,\eps} - \RE(M_{\xi}T_x g)\|_{L^{\infty}(S_x)} \leq \tfrac{\eps}{4s_x} + \|\Psi^{\RE}_{x,\xi,\eps} - \RE(M_{\xi}T_x g)\|_{L^{\infty}(S_x)}.
  \end{align*}
  The analogous estimate for $\|\Gamma^{\IM}_{x,\xi,\eps}-\IM(M_{\xi}T_x g)\|_{L^{\infty}(S_x)}$ is obtained in exactly the same manner. Together with \eqref{WERx2}, we can finally infer that, for all 
  $x,\xi\in\R^d$, $\eps\in(0,1/2)$,
    \begin{align*}
    \|\RE(M_{\xi}T_x g)-\Gamma^{\RE}_{x,\xi,\eps}\|_{L^{\infty}(S_x)}+\|\IM(M_{\xi}T_x g)-\Gamma^{\IM}_{x,\xi,\eps}\|_{L^{\infty}(S_x)}\leq \tfrac{\eps}{s_x}.
  \end{align*} 
  As $M_{\xi}T_x g$, $\Gamma^{\RE}_{x,\xi,\eps}$, and $\Gamma^{\IM}_{x,\xi,\eps}$ are supported in $S_x$ for all $x,\xi\in\R^d$, $\eps\in(0,1/2)$, using
  \eqref{WERx8}, we get
  \begin{align}\begin{split}\label{WERx10}
          \quad&\|\RE(M_{\xi}T_x g)-\Gamma^{\RE}_{x,\xi,\eps}\|_{L^2(\R^d)}+\|\IM(M_{\xi}T_x g)-\Gamma^{\IM}_{x,\xi,\eps}\|_{L^2(\R^d)}\\
          &= \|\RE(M_{\xi}T_x g)-\Gamma^{\RE}_{x,\xi,\eps}\|_{L^2(S_x)}+\|\IM(M_{\xi}T_x g)-\Gamma^{\IM}_{x,\xi,\eps}\|_{L^2(S_x)}\\
          &\leq s_x\|\RE(M_{\xi}T_x g)-\Gamma^{\RE}_{x,\xi,\eps}\|_{L^\infty(S_x)}+s_x\|\IM(M_{\xi}T_x g)-\Gamma^{\IM}_{x,\xi,\eps}\|_{L^\infty(S_x)}\leq \eps.
  \end{split}\end{align}
    Consider now, for $i\in\N$, $\eps\in(0,1/2)$, the complex-valued network $\Gamma_{i,\eps}\in\cN^{\C}_{d,1}$ given by
    \begin{align*}
        \Gamma_{i,\eps}:=(\Gamma^{\RE}_{x(i),\xi(i),\eps},\Gamma^{\IM}_{x(i),\xi(i),\eps})
    \end{align*}
  and note that, for $f\in L^2(\Omega,\C)$,
  \begin{align*}
      \|f\|_{L^2(\Omega,\C)}&=\left(\int_\Omega |f(t)|^2 \text{d}t\right)^{\frac{1}{2}}=
      \left(\int_\Omega |\RE (f(t))|^2 + |\IM (f(t))|^2 \text{d}t\right)^{\frac{1}{2}}=
      \left(\|\RE (f)\|^2_{L^2(\Omega)}+\|\IM
      (f)\|^2_{L^2(\Omega)}\right)^{\frac{1}{2}}\\
      &\leq \|\RE (f)\|_{L^2(\Omega)}+\|\IM
      (f)\|_{L^2(\Omega)}.
  \end{align*}
  Hence, \eqref{WERx10} implies that, for all $i\in\N$, $\eps\in(0,1/2)$,
  \begin{align*}
      \|\varphi_i-\Gamma_{i,\eps}\|_{L^2(\R^d,\C)}=\|M_{\xi(i)}T_{x(i)}g-(\Gamma^{\RE}_{x(i),\xi(i),\eps},\Gamma^{\IM}_{x(i),\xi(i),\eps})\|_{L^2(\R^d,\C)}\leq\eps.
  \end{align*}
  Finally, using \eqref{GSR1} in \eqref{MB-bounds}, it follows that there exists a polynomial $\pi_5$ such that for all $i\in\N$, $\eps\in(0,1/2)$, we have
    $\M(\Gamma^{\RE}_{x(i),\xi(i),\eps}),\M(\Gamma^{\IM}_{x(i),\xi(i),\eps})\leq\pi_5(\log(\eps^{-1}),\log(i))$ and $\mathcal{B}(\Gamma^{\RE}_{x(i),\xi(i),\eps}),\mathcal{B}(\Gamma^{\IM}_{x(i),\xi(i),\eps})\leq\pi_5(\eps^{-1},i)$, which 
    finalizes the proof.
\end{proof}

Next, we establish the central result of this section. To this end, we first recall that according to Theorem \ref{thm:optitransgeneral} neural networks provide optimal approximations for all function classes that are optimally approximated by affine dictionaries (generated by functions $f$ that can be approximated well by neural networks). While this universality property is significant as it applies to all affine dictionaries, it is perhaps not completely surprising as affine dictionaries are generated by affine transformations and neural networks consist of concatenations of affine transformations and nonlinearities. Gabor dictionaries, on the other hand, exhibit a fundamentally different mathematical structure. The next result shows that neural networks also provide optimal approximations for all function classes that are optimally approximated by Gabor dictionaries (again, with generator functions that can be approximated well by neural networks). 

\pagebreak

\begin{theorem}\label{GaborSystemTrans}
  Let $d\in\N$, $\Omega\subseteq\R^d$, $\alpha,\beta>0$, $g\in L^2(\R^d)\cap L^\infty(\R^d)$, and let $\mathcal{G}(g,\alpha,\beta,\Omega)$  be the corresponding Gabor dictionary with ordering as defined in \eqref{GaborOrdering}. Assume that $\Omega$ is bounded or that $\Omega=\R^d$ and $g$ is compactly supported. 
  Further, suppose that there exists a polynomial $\pi$ such that for every $x\in\R^d$, $\eps\in(0,1/2)$, there is a network $\Phi_{x,\eps}\in\cN_{d,1}$ satisfying 
  \begin{align*}
    \|g-\Phi_{x,\eps}\|_{L^{\infty}(x+\Omega)}\leq\eps,
\end{align*}
with $\M(\Phi_{x,\eps})\leq\pi(\log(\eps^{-1}),\log(\|x\|_\infty))$, $\mathcal{B}(\Phi_{x,\eps})\leq\pi(\eps^{-1},\|x\|_\infty)$. Then, for all compact function classes $\cC\subseteq L^2(\Omega)$, we have
    \begin{align*}
        \gamma_{\cN}^{\ast, \text{eff}}(\mathcal{C})\geq\gamma^{\ast, \text{eff}}(\mathcal{C},\mathcal{G}(g,\alpha,\beta,\Omega)).
    \end{align*}
    In particular, if $\mathcal{C}$ is optimally representable by $\mathcal{G}(g,\alpha,\beta,\Omega)$ (in the sense of Definition \ref{def:repopti}), then $\mathcal{C}$ is optimally representable by neural networks (in the sense of Definition \ref{def:NNopti}).

\end{theorem}

\begin{proof}
    The first statement follows from Theorem~\ref{theo:EncodeOfNeuralNetworks} and Theorem~\ref{GaborSystemsRep}, the second is by Theorem~\ref{thm:EffRepNN}.
\end{proof}

We complete the program in this section by showing that the Gaussian function satisfies the conditions on the generator $g$ in Theorem \ref{GaborSystemsRep} for bounded $\Omega$. Gaussian functions are widely used generator functions for Gabor dictionaries owing to their excellent time-frequency localization and their frame-theoretic optimality properties \cite{grochenig2013foundations}. We hasten to add that the result below can be extended to any generator function $g$ of
sufficiently fast decay and sufficient smoothness.

\begin{lemma}\label{NNGaussian}
For $d\in\N$, let $g_d\in L^2(\R^d)$ be given by
  \begin{align*}
    g_d(x):=e^{-\|x\|^2_2}.
  \end{align*}
There exists a constant $C>0$ such that, for all $d\in\N$ and $\eps\in(0,1/2)$, there is a network $\Phi_{d,\eps}\in\cN_{d,1}$ satisfying 
 \begin{align*}
   \|\Phi_{d,\eps}-g\|_{L^{\infty}(\R^d)}\leq\eps,
 \end{align*}
with $\M(\Phi_{d,\eps})\leq C d (\log(\eps^{-1}))^2((\log(\eps^{-1}))^2+\log(d))$, $\mathcal{B}(\Phi_{d,\eps})\leq 1$.
\end{lemma}

\begin{proof}
 Observe that $g_d$ can be written as the composition $h \circ f_d$ of the functions $f_d\colon\R^d\to\R_+$ and $h\colon \R_+\to\R$ given by 
 \begin{align*}
    f_d(x):=\|x\|^2_2=\sum_{i=1}^d x_i^2\quad\text{and}\quad h(y):=e^{-y}.
 \end{align*}
 By Proposition~\ref{relu_mult} and Lemma~\ref{network_linearcombination}, there exists a constant $C_1>0$ such that, for every $d\in\N$, $D\in[1,\infty)$, $\eps\in(0,1/2)$, there is a network $\Psi_{d,D,\eps}\in\cN_{d,1}$ satisfying
  \begin{align}\label{GP1}
   &\sup_{x\in[-D,D]^d}|\Psi_{d,D,\eps}(x)-\|x\|_2^2| \leq \tfrac{\eps}{2},\\
   \label{boundsonnet1}
    &\cM(\Psi_{d,D,\eps})\leq C_1 d(\log(\eps^{-1})+ \log(\lceil D \rceil) ), \quad \cB(\Psi_{d,D,\eps})\leq 1.
 \end{align}
 Moreover, as $|\tfrac{d^n}{dy^n}e^{-y}|=|e^{-y}|\leq 1$ for all $n\in\N$, $y\geq 0$, Lemma~\ref{lem:Sfunctions_general} implies the existence of a constant $C_2>0$ such that for every $d \in\N$, $D \in [1,\infty)$, $\eps\in(0,1/2)$, there is a network $\Gamma_{d,D,\eps}\in\cN_{1,1}$ satisfying
 \begin{align}\label{GP2}
   &\sup_{y\in[0,d D^2]}|\Gamma_{d,D,\eps}(y)-e^{-y}|\leq\tfrac{\eps}{2},\\
   \label{boundsonnet2}
    &\cM(\Gamma_{d,D,\eps})\leq C_2 dD^2 ((\log(\eps^{-1}))^2+\log(d)+\log(\lceil D \rceil)), \quad \cB(\Gamma_{D,\eps})\leq 1.
\end{align}
Now, let $D_\eps:=\log(\eps^{-1})$ and take $\widetilde{\Phi}_{d,\eps}:=\Gamma_{d,D_\eps,\eps}\circ\Psi_{d,D_\eps,\eps}$ according to Lemma~\ref{network_conc}. Consequently, it follows from (\ref{boundsonnet1}) and (\ref{boundsonnet2}) that there exists a constant $C_2>0$ such that for all $d\in\N$, $\eps\in(0,1/2)$, we have $\cM(\widetilde{\Phi}_{d,\eps})\leq C_2 d (\log(\eps^{-1}))^2((\log(\eps^{-1}))^2+\log(d))$ and $\cB(\widetilde{\Phi}_{d,\eps})\leq 1$. Moreover, as  $|e^{-y}|\leq 1$ for all $y\geq 0$, combining \eqref{GP1} and \eqref{GP2} yields for all $\eps\in(0,1/2)$, $x\in[-D_\eps,D_\eps]^d$,
\begin{align*}
    |g(x)-\widetilde{\Phi}_{d,\eps}(x)|&=|e^{-\|x\|^2_2}-\Gamma_{d,D_\eps,\eps}(\Psi_{d,D_\eps,\eps}(x))|\\
    &\leq|e^{-\|x\|^2_2}-e^{-\Psi_{d,D_\eps,\eps}(x)}|+|e^{-\Psi_{d,D_\eps,\eps}(x)}-\Gamma_{d,D_\eps,\eps}(\Psi_{d,D_\eps,\eps}(x))|\\
    &\leq\tfrac{\eps}{2}+\tfrac{\eps}{2}=\eps.
\end{align*}
We can now use the same approach as in the proof of \Cref{GaborSystemsRep} to construct networks $\Phi_{d,\eps}$ supported on the interval $[-D_\eps,D_\eps]^d$ over which they approximate $g$ to within error $\eps$, and obey $\M(\Phi_\eps) \leq C d (\log(\eps^{-1}))^2((\log(\eps^{-1}))^2+\log(d))$, $\mathcal{B}(\Phi_{d,\eps}) \leq 1$ for some absolute constant $C$.   
Together with $|g(x)|\leq\eps$, for all $x\in\R^d\backslash[-D_\eps,D_\eps]^d$, this completes the proof.
\end{proof}
 
 \begin{remark}
 Note that Lemma \ref{NNGaussian} establishes an approximation result that is even stronger than what is required by Theorem \ref{GaborSystemsRep}. Specifically, we achieve $\eps$-approximation 
 over all of $\R^d$ with a network that does not depend on the shift parameter $x$, while exhibiting the desired growth rates on  
  $\mathcal{M}$ and $\mathcal{B}$, which consequently do not depend on the shift parameter as well. The idea underlying this construction can be used to strengthen Theorem~\ref{GaborSystemsRep} to apply to $\Omega=\R^d$ and generator functions of unbounded support, but sufficiently rapid decay.
 \end{remark}

We conclude this section with a remark on the neural network approximation of the
real-valued counterpart of Gabor dictionaries known as Wilson dictionaries \cite{grochenig:2000GaborApprox,grochenig2013foundations} and consisting of cosine-modulated and time-shifted versions of a given generator function, see also Appendix~\ref{Modulation_tail}. The techniques developed in this section, mutatis mutandis, 
show that neural networks provide Kolmogorov-Donoho optimal approximation for all function classes that are optimally approximated by Wilson dictionaries (generated by functions that can be approximated well by neural networks). 
Specifically, we point out that the proofs of Lemma \ref{NNmodulation} and Theorem~\ref{GaborSystemsRep} explicitly construct neural network approximations of time-shifted and cosine- and sine-modulated versions of the generator $g$. 
As identified in Table~\ref{table_opt_exp}, Wilson bases provide optimal nonlinear approximation of (unit) balls in modulation spaces \cite{feichtinger:1981ModulationSpaces,grochenig:2000GaborApprox}.
Finally, we note that similarly the techniques developed in the proofs of Lemma \ref{NNmodulation} and Theorem~\ref{GaborSystemsRep} can be used to establish optimal representability of Fourier bases.

\section{Improving Polynomial Approximation Rates to Exponential Rates}\label{sec:weierstrass}

Having established that for all function classes listed in Table~\ref{table_opt_exp},
Kolmogorov-Donoho-optimal approximation through neural networks is possible, this section proceeds to show that neural networks, in addition to their striking Kolmogorov-Donoho universality property, can also do something that has no classical equivalent.

Specifically, as mentioned in the introduction, for the class of oscillatory textures as considered below and for the Weierstrass function, there are no known methods that achieve exponential accuracy, i.e., an approximation error that decays exponentially in the number of parameters employed
in the approximant. We establish below that deep networks fill this gap.

Let us start by defining one-dimensional ``oscillatory textures'' according to \cite{demanet2007wave}. To this end, we recall the following definition from Lemma~\ref{lem:Sfunctions_general},
 \begin{align*}
   \Ss_{[a,b]}=\left\{f\in C^{\infty}([a,b],\R)\colon \|f^{(n)}(x)\|_{ L^{\infty}([a,b])} \leq n!,\, \text{\emph{ for all }} n\in\N_0\right\}.
 \end{align*} 
\begin{definition}
	\label{FunctionClassFa}
	Let the sets $\F_{D,a}$,\,$D,a\in\R_+$, be given by
	\begin{align*}
	\F_{D,a}=\left\{\cos(ag)h \colon g,h\in\Ss_{[-D,D]}\right\}.
	\end{align*}
\end{definition}
The efficient approximation of functions in $\mathcal{F}_{D,a}$ with $a$ large represents a notoriously difficult problem due to the combination of the rapidly oscillating cosine term and the warping function $g$. 
The best approximation results available in the literature \cite{demanet2007wave} are based on wave-atom dictionaries\footnote{To be precise, the results of \cite{demanet2007wave} are concerned with the two-dimensional case, whereas here we focus on the one-dimensional case. Note, however, that all our results are readily extended to the multi-dimensional case.} and yield low-order polynomial approximation rates. In what follows we show that finite-width deep networks drastically improve these results to exponential approximation rates.

We start with our statement on the neural network approximation of oscillatory textures.

\begin{proposition}\label{Fafunctions}
  There exists a constant $C>0$ such that for all $D,a\in \mathbb{R}_+, f\in\F_{D,a}$, and $\eps\in(0,1/2)$, 
  there is a network $\Gamma_{f,\eps}\in\cN_{1,1}$ satisfying 
  \begin{align*}
    \|f-\Gamma_{f,\eps}\|_{ L^{\infty}([-D,D])}\leq\eps,
  \end{align*}
  with $\L(\Gamma_{f,\eps})\le C\lceil D\rceil((\log(\eps^{-1}) + \log(\lceil a\rceil))^2+\log(\lceil D \rceil)+\log(\lceil D^{-1}\rceil))$, $\mathcal{W}(\Gamma_{f,\eps})\leq 32$,  $\mathcal{B}(\Gamma_{f,\eps})\leq 1$.
  \end{proposition}
\begin{proof}
  For $D,a\,\in\,\mathbb{R}_{+}$, $f\in\F_{D,a}$, let $g_f,h_f\in\Ss_{[-D,D]}$ be functions such that $f=\cos(ag_f)h_f$.
  Note that Lemma \ref{lem:Sfunctions_general} guarantees the existence of a constant $C_1>0$ such that for all $D,a\,\in\,\mathbb{R}_{+}$, $\eps\in(0,1/2)$,
  there are networks $\Psi_{g_f,\eps},\Psi_{h_f,\eps}\in\cN_{1,1}$
  satisfying 
    \begin{align}\label{FaghEst}
    \|\Psi_{g_f,\eps}-g_f\|_{ L^{\infty}([-D,D])} \leq\tfrac{\eps}{12\lceil a \rceil}, \quad \|\Psi_{h_f,\eps}-h_f\|_{ L^{\infty}([-D,D])}\leq\tfrac{\eps}{12\lceil a \rceil}
  \end{align}
  with 
  \begin{align*}
    \L(\Psi_{g_f,\eps}), \L(\Psi_{h_f,\eps})\leq C_1\lceil D\rceil(\log((\tfrac{\eps}{12\lceil a \rceil})^{-1})^2+\log(\lceil D \rceil)+\log(\lceil D^{-1}\rceil)),  
  \end{align*}
  $\mathcal{W}(\Psi_{g_f,\eps}), \mathcal{W}(\Psi_{h_f,\eps}) \le 16$, and 
  $\mathcal{B}(\Psi_{g_f,\eps}), \mathcal{B}(\Psi_{h_f,\eps})\leq 1$.
  Furthermore, Theorem \ref{sin} ensures the existence of a constant $C_2>0$ such that for all $D,a\in\R_+$, $\eps\in(0,1/2)$, there is a neural network $\Phi_{a,D,\eps}\in\cN_{1,1}$ satisfying
    \begin{align}\label{FaCosEst}
    \|\Phi_{a,D,\eps}-\cos(a\,\cdot\,)\|_{ L^{\infty}([-3/2,3/2])}\leq\tfrac{\eps}{3},
  \end{align}
  with $\L(\Phi_{a,D,\eps}) \le C_2((\log(\eps^{-1}))^2+\log(\lceil 3a/2 \rceil))$, $\mathcal{W}(\Phi_{a,D,\eps}) \le 9$, and $\mathcal{B}(\Phi_{a,D,\eps})\leq 1$.
  Moreover, due to Proposition \ref{relu_mult}, there exists a constant $C_3>0$ such that for all $\eps\in(0,1/2)$, there is a network $\mu_{\eps}\in\cN_{2,1}$ satisfying 
    \begin{align}\label{FamuEst}
    \sup_{x,y\in[-3/2,3/2]}|\mu_{\eps}(x,y)-xy|\leq\tfrac{\eps}{3},
  \end{align}
  with $\L(\mu_{\eps}) \le C_3\log(\eps^{-1})$, $\mathcal{W}(\mu_{\eps}) \le 5$, and $\mathcal{B}(\mu_{\eps}) \le 1$.
  By Lemma \ref{network_conc} there exists a network $\Psi^1$ satisfying $\Psi^1=\Phi_{a,D,\eps}\circ\Psi_{g_f,\eps}$ with $\mathcal{W}(\Psi^1) \leq 16$, $\L(\Psi^1) = \L(\Phi_{a,D,\eps}) + \L(\Psi_{g_f,\eps})$, and $\cB(\Psi^1)\leq 1$.
  Furthermore, combining Lemma~\ref{network_extension} and Lemma~\ref{lem:shared_input_para}, we can conclude the existence of a network  $\Psi^2(x)=(\Psi^1(x),\Psi_{h_f,\eps}(x)) = (\Phi_{a,D,\eps}(\Psi_{g_f,\eps}(x)),\Psi_{h_f,\eps}(x))$ with $\mathcal{W}(\Psi^2) \leq 32$,
  $\L(\Psi^2) = \max \{ \L(\Phi_{a,D,\eps}) + \L(\Psi_{g_f,\eps}), \L(\Psi_{h_f,\eps}) \}$, and $\cB(\Psi^2)\leq 1$.
  Next, for all $D,a\,\in\,\mathbb{R}_{+}$, $f\in\F_{D,a}$, $\eps\in(0,1/2)$, we define the network $\Gamma_{f,\eps}:= \mu_{\eps} \circ\Psi^2$.	
  By \eqref{FaghEst}, \eqref{FaCosEst}, and $\sup_{x\in\R}|\tfrac{d}{dx}\cos(ax)|=a$, we have, for all $x\in[-D,D]$,
  \begin{align*}\begin{split}
   |\Phi_{a,D,\eps}(\Psi_{g_f,\eps}(x))-\cos(ag_f(x))|&\leq|\Phi_{a,D,\eps}(\Psi_{g_f,\eps}(x))-\cos(a\Psi_{g_f,\eps}(x))|\\
   &\quad\,+|\cos(a\Psi_{g_f,\eps}(x))-\cos(ag_f(x))|\\
   &\leq\tfrac{\eps}{3}+a\tfrac{\eps}{12\lceil a \rceil}\leq\tfrac{5\eps}{12}.
  \end{split}\end{align*}
  Combining this with \eqref{FaghEst}, \eqref{FamuEst}, and $\|\cos\|_{L^\infty ([-D,D])},\|f\|_{L^\infty ([-D,D])} \le 1$ yields for all $x\,\in\,[-D,D]$,
  \begin{align*}
   |\Gamma_{f,\eps}(x)-f(x)|&=|\mu_{\eps}(\Phi_{a,D,\eps}(\Psi_{g_f,\eps}(x)),\Psi_{h_f,\eps}(x))-\cos(ag_f(x))h_f(x)|\\
   &\leq|\mu_{\eps}(\Phi_{a,D,\eps}(\Psi_{g_f,\eps}(x)),\Psi_{h_f,\eps}(x))-\Phi_{a,D,\eps}(\Psi_{g_f,\eps}(x))\Psi_{h_f,\eps}(x)|\\
   &\quad+|\Phi_{a,D,\eps}(\Psi_{g_f,\eps}(x))\Psi_{h_f,\eps}(x)-\cos(ag_f(x))\Psi_{h_f,\eps}(x)|\\
   &\quad+|\cos(ag_f(x))\Psi_{h_f,\eps}(x)-\cos(ag_f(x))h_f(x)|\\
   &\leq \tfrac{\eps}{3}+\tfrac{5\eps}{12}\left(1+\tfrac{\eps}{12\lceil a \rceil}\right)+\tfrac{\eps}{12\lceil a \rceil}\leq\eps.
  \end{align*}
  Finally, by Lemma \ref{network_conc} there exists a constant $C_4$ such that for all $D,a\,\in\,\mathbb{R}_{+}$, $f\in\F_{D,a}$, $\eps\in(0,1/2)$, it holds that $\Wcal(\Gamma_{f,\eps}) \leq 32$, 
  \begin{align*}
     \L(\Gamma_{f,\eps})&\leq\L(\mu_{\eps})+\max\{\L(\Phi_{a,D,\eps})+\L(\Psi_{g_f,\eps}),\L(\Psi_{h_f,\eps})\}\\
     &\leq C_4\lceil D\rceil((\log(\eps^{-1}) + \log(\lceil a\rceil))^2+\log(\lceil D \rceil)+\log(\lceil D^{-1}\rceil)),
  \end{align*}
  and $\cB(\Gamma_{f,\eps})\leq 1$.
\end{proof}
\pagebreak

\begin{figure}
	\includegraphics[width=0.5\textwidth]{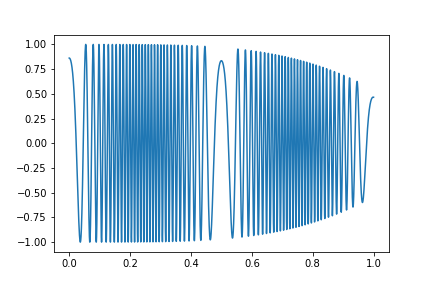}\hfill
	\includegraphics[width=0.5\textwidth]{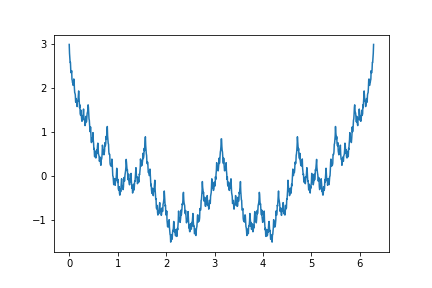}
	\caption{Left: A function in $\mathcal{F}_{1,100}$. Right: The function $W_{\frac{1}{\sqrt{2}},2}$.}
	\label{fig1}
\end{figure}

Finally, we show how the Weierstrass function---a fractal function, which is continuous everywhere but differentiable nowhere---can be approximated with exponential accuracy by deep ReLU networks. Specifically, we consider
\begin{align*}
W_{p,a}(x)=\sum_{k=0}^{\infty}p^k\cos(a^k\pi x), \quad \mbox{for}\,\,  p\in(0,1/2), \, a \in\R_+, \,\, \text{with} \,\, ap \ge 1,
\end{align*}
and let $\alpha = -\frac{\log(p)}{\log(a)}$, see Figure \ref{fig1} right for an example. It is well known \cite{zygmund2002trigonometric} that $W_{p,a}$ possesses H\"older smoothness $\alpha$ which may be made arbitrarily small by suitable choice of $a$.
While classical approximation methods achieve polynomial approximation rates only, it turns out
that finite-width deep networks yield exponential approximation rates. This is formalized as follows.
\begin{proposition}\label{prop:Weierstrass}
There exists a constant $C>0$ such that for all $\eps,p\in(0,1/2)$, $D,a\in \mathbb{R}_+$, there is a network $\Psi_{p,a,D,\eps}\in\cN_{1,1}$ satisfying
\begin{align*}
        \|\Psi_{p,a,D,\eps}-W_{p,a}\|_{ L^{\infty}([-D,D])}\leq\eps,
    \end{align*}
with $\cL(\Psi_{p,a,D,\eps}) \le C((\log(\eps^{-1}))^3+(\log(\eps^{-1}))^2\log(\lceil a \rceil)+\log(\eps^{-1})\log(\lceil D\rceil )),$ 
$\mathcal{W}(\Psi_{p,a,D,\eps}) \le 13$, ${\mathcal{B}(\Psi_{p,a,D,\eps})\leq 1}$.
\end{proposition} 
\begin{proof}
  For every $N\in\N$, $p\in(0,1/2)$, $a\,\in\,\mathbb{R}_{+}$, $x\in\R$, let $S_{N,p,a}(x)=\sum_{k=0}^{N}p^k\cos(a^k\pi x)$ and note that
  \begin{align}\label{Weier1}
    |S_{N,p,a}(x)-W_{p,a}(x)|\leq \sum_{k=N+1}^{\infty}|p^k\cos(a^k\pi x)|\leq\sum_{k=N+1}^{\infty}p^k=\tfrac{1}{1-p}-\tfrac{1-p^{N+1}}{1-p}\leq 2^{-N}.
  \end{align}
  Let $N_{\eps}:=\lceil \log(2/\eps) \rceil$ for $\eps\in(0,1/2)$.
    Next, note that \Cref{sin} ensures the existence of a constant $C_1>0$ such that for all $D,a \in\R_+$, $k\in\N_0$, $\eps\in(0,1/2)$, 
    there is a network $\phi_{a^k,D,\eps}\in\cN_{1,1}$ satisfying
    \begin{align}\label{Weier42}
    \|\phi_{a^k,D,\eps}-\cos(a^k\pi\,\cdot\,)\|_{ L^{\infty}([-D,D])}\leq\tfrac{\eps}{4},
  \end{align}
    with ${\L(\phi_{a^k,D,\eps}) \le C_1((\log(\eps^{-1}))^2+\log(\lceil  a^k\pi D\rceil))}$, $\mathcal{W}(\phi_{a^k,D,\eps}) \le 9$, $\mathcal{B}(\phi_{a^k,D,\eps})\leq 1$. Let $A\colon\R^3\to\R^3$ and $B\colon\R^3\to\R$ be the affine transformations given by $A(x_1,x_2,x_3)=(x_1,x_1,x_2+x_3)^T$ and $B(x_1,x_2,x_3)=x_2+x_3$, respectively. We now define, for all
$p\in(0,1/2)$, $D,a \in\R_+$, $k\in\N_0$, $\eps\in(0,1/2)$, the networks
  \begin{align*}
    \psi^{p,a,0}_{D,\eps}(x)=\begin{pmatrix}x \\ p^0\phi_{a^0,D,\eps}(x) \\ 0\end{pmatrix}\quad\mathrm{and}\quad \psi^{p,a,k}_{D,\eps}(x_1,x_2,x_3)=\begin{pmatrix} x_1 \\ p^k\phi_{a^k,D,\eps}(x_2) \\ x_3\end{pmatrix},\, k>0,
  \end{align*}
 and, for all $p\in(0,1/2)$, $D,a\in\R_+$, $\eps\in(0,1/2)$, the network 
  \begin{align*}
    \Psi_{p,a,D,\eps}:=
    B\circ\psi^{p,a,N_{\eps}}_{D,\eps}\circ A \circ\psi^{p,a,N_{\eps}-1}_{D,\eps}\circ\dots\circ A\circ\psi^{p,a,0}_{D,\eps}.
 \end{align*}
  Due to \eqref{Weier42} we get, for all $p\in(0,1/2)$, $D,a\in\R_+$,  $\eps\in(0,1/2)$, $x\in[-D,D]$, that
    \begin{align*}
    |\Psi_{p,a,D,\eps}(x)-S_{N_{\eps},p,a}(x)|&=\left|\sum_{k=0}^{N_{\eps}}p^k\phi_{a^k,D,\eps}(x)-\sum_{k=0}^{N_{\eps}}p^k\cos(a^k\pi x)\right|\\
    &\leq \sum_{k=0}^{N_{\eps}}p^k|\phi_{a^k,D,\eps}(x)-\cos(a^k\pi x)|\leq \tfrac{\eps}{4}\sum_{k=0}^{N_{\eps}} 2^{-k} \le \tfrac{\eps}{2}.
  \end{align*}
  Combining this with \eqref{Weier1} establishes, for all $p\in(0,1/2)$, $D,a\in\R_+$, $\eps\in(0,1/2)$, $x\in[-D,D]$,
    \begin{align*}
   |\Psi_{p,a,D,\eps}(x)-W_{p,a}(x)|\leq 2^{-\lceil \log(\frac{2}{\eps})\rceil }+\tfrac{\eps}{2}\leq\tfrac{\eps}{2}+\tfrac{\eps}{2}=\eps.
  \end{align*}
  Applying Lemmas~\ref{network_conc}, \ref{network_extension}, and \ref{network_parallelization} establishes the existence of a constant $C_2$ such that for all $p\in(0,1/2)$, $D,a\in\R_+$, $\eps\in(0,1/2)$,
  \begin{align*}
   \L(\Psi_{p,a,D,\eps})&\leq \sum_{k=0}^{N_{\eps}}(\L(\phi_{a^k,D,\eps})+1) \leq N_\eps+1 + (N_{\eps}+1)C_1((\log(\eps^{-1}))^2+\log(\lceil  a^{N_{\eps}}\pi D\rceil))\\
    &\leq C_2((\log(\eps^{-1}))^3+(\log(\eps^{-1}))^2\log(\lceil a \rceil)+\log(\eps^{-1})\log(\lceil D \rceil)),
  \end{align*}
  $\Wcal(\Psi_{p,a,D,\eps}) \le 13$, and $\mathcal{B}(\Psi_{p,a,D,\eps})\leq 1$.
\end{proof}

We finally note that the restriction $p \in (0,1/2)$ in Proposition~\ref{prop:Weierstrass} was made for simplicity of exposition and can be relaxed to $p\in(0,r)$, with $r<1$, while only changing the constant $C$.

\section{Impossibility results for finite-depth networks}\label{sec:depth-width}

The recent successes of neural networks in machine learning applications have been enabled by various technological factors, but they all have in common the use of deep networks as opposed to shallow networks studied intensely in the 1990s. It is hence of interest to understand whether the use of depth offers fundamental advantages. In this spirit, the goal of 
this section is to make a formal case for depth in neural network approximation by establishing that, for nonconstant periodic functions, finite-width deep networks require asymptotically---in the function's ``highest frequency"---smaller connectivity than finite-depth wide networks. This statement is then extended to sufficiently smooth nonperiodic functions, thereby formalizing the benefit of deep networks over shallow networks for the approximation of a broad class of functions.

We start with preparatory material taken from \cite{telgarsky:2015Sawtooth}.
\begin{definition}[\cite{telgarsky:2015Sawtooth}]
\label{defsawtooth}
Let $k \in \mathbb{N}$. A function $f: \mathbb{R} \rightarrow \mathbb{R}$ is called $k$-sawtooth if it is piecewise linear with no more than $k$ pieces, i.e., its domain $\mathbb{R}$ can be partitioned into $k$ intervals such that $f$ is linear on each of these intervals.
\end{definition}

\begin{lemma}[\cite{telgarsky:2015Sawtooth}]\label{telgarsky2015}
\label{sawtoothness}
Every $\Phi \in \cN_{1,1}$ is $(2\Wcal(\Phi))^{\L(\Phi)}$-sawtooth.
\end{lemma}

\begin{definition}
For a $u$-periodic function $f \in C(\mathbb{R})$, we define 
\[\xi(f):= \sup_{\delta \in [0,u)}\inf_{c,d \in \R} \| f(x) - (cx +d) \|_{L^{\infty}([\delta,\delta+u])}.\]
\end{definition}
The quantity $\xi(f)$ measures the error incurred by the best linear approximation of $f$ on any segment of length equal to the period of $f$; $\xi(f)$ can hence be interpreted as quantifying the nonlinearity of $f$. 
The next result states that finite-depth networks with width and hence also connectivity scaling polylogarithmically in the ``highest frequency'' of the periodic function to be approximated can not achieve
arbitrarily small approximation error.

\begin{proposition}
\label{log-not-enough}
Let $f \in C(\mathbb{R})$ be a nonconstant $u$-periodic function, $L\in\N$, and $\pi$ a polynomial. Then, there exists an $a\in\N$ such that for every network $\Phi \in \cN_{1,1}$ with $\cL(\Phi)\leq L$ and $\mathcal{W}(\Phi) \leq \pi(\log(a))$, we have
\[ \| f(a \,\cdot\,) - \Phi \|_{L^{\infty}([0,u])} \geq \xi(f)>0. \]
\end{proposition}
\begin{proof}
First note that there exists an even $a\in\N$ such that $a/2>(2\pi(\log(a)))^L$. Lemma \ref{telgarsky2015} now implies that every network $\Phi \in \cN_{1,1}$ with $\cL(\Phi)\leq L$ and $\mathcal{W}(\Phi) \leq \pi(\log(a))$ is $(2\pi(\log(a)))^L$-sawtooth and therefore consists of no more than $a/2$ different linear pieces. Hence, there exists an interval $[u_1,u_2]\subseteq[0,u]$ with $u_2-u_1\geq (2u/a)$ on which $\Phi$ is linear. Since $u_2-u_1\geq (2u/a)$ the interval supports two full periods of $f(a\,\cdot\,)$ and we can therefore conclude that
\begin{align*}
\| f(a \,\cdot\,) - \Phi \|_{L^{\infty}([0,u])} & \geq \| f(a\, \cdot\,) - \Phi \|_{L^{\infty}([u_1,u_2])} \geq \inf_{c,d\in\R} \| f(x) - (cx +d) \|_{L^{\infty}([0,2u])}\\
& \ge \sup_{\delta \in [0,u)} \inf_{c,d\in\R} \| f(x) - (cx +d) \|_{L^{\infty}([\delta,u+\delta])} = \xi(f).
\end{align*}
Finally, note that $\xi(f)>0$ as $\xi(f)=0$ for $u$-periodic $f \in C(\mathbb{R})$ necessarily implies that $f$ is constant, which, however, is ruled out by assumption.
\end{proof}

Application of Proposition \ref{log-not-enough} to $f(x)=\cos(x)$ shows that finite-depth networks, owing to $\xi(\cos)>0$, require faster than polylogarithmic 
growth of connectivity in $a$ to approximate $x\mapsto\cos(ax)$ with arbitrarily small error, whereas finite-width networks, due to Theorem \ref{sin}, 
can accomplish this with polylogarithmic connectivity growth.

The following result from \cite{frenzen:2010Pl-approx} allows a similar observation for functions that are sufficiently smooth.

\begin{theorem}[\cite{frenzen:2010Pl-approx}]
\label{pl-approx}
Let $[a,b]\subseteq \R$, $f \in C^3([a, b])$, and for $\eps\in(0,1/2)$, let $s(\eps)\in\N$ denote the smallest number such that there exists a piecewise linear approximation of $f$ with $s(\eps)$ pieces and error at most $\eps$ in $L^{\infty}([a,b])$-norm. Then, it holds that
\[s(\epsilon) \sim \frac{c}{\sqrt{\epsilon}},\, \eps\to 0,\,
\mbox{ where } \, c = \frac{1}{4} \int_a^b \sqrt{| f''(x) |}dx.\]
\end{theorem}
Combining this with Lemma \ref{sawtoothness} yields the following result on depth-width tradeoff for three-times continuously differentiable functions.
\begin{theorem}
\label{log-not-enough-NN2}
Let $f \in C^3([a, b])$ with $\int_a^b \sqrt{| f''(x) |}dx>0$, $L\in\N$, and $\pi$ a polynomial. Then, there exists $\epsilon>0$ such that for every network $\Phi \in \cN_{1,1}$ with $\cL(\Phi)\leq L$ and $\mathcal{W}(\Phi) \leq \pi(\log(\epsilon^{-1}))$, we have
\[ \| f - \Phi \|_{L^{\infty}([a,b])} > \epsilon. \]
\end{theorem}
\begin{proof}
The proof will be effected by contradiction. Assume that for every $\eps>0$, there exists a network $\Phi_\eps \in \cN_{1,1}$ with $\cL(\Phi_\eps)\leq L$, $\mathcal{W}(\Phi_\eps) \leq \pi(\log(\epsilon^{-1}))$, and
$\| f - \Phi_\eps \|_{L^{\infty}([a,b])} \leq \epsilon$. By Lemma \ref{telgarsky2015} every (ReLU) neural network realizes a piecewise linear function. Application of Theorem \ref{pl-approx} hence allows us to conclude the existence of a constant $C$ such that, for all $\eps>0$, the network $\Phi_\eps$ must have at least $C\eps^{-\frac{1}{2}}$ different linear pieces. This, however, leads to a contradiction as, by Lemma \ref{telgarsky2015}, $\Phi_\eps$ is at most $(2\pi(\log(\eps^{-1})))^L$-sawtooth 
and $\tilde{\pi}(\log(\epsilon^{-1})) \in o(\epsilon^{-1/2})$, $\epsilon \rightarrow 0$, for every polynomial $\tilde{\pi}$.
\end{proof}
In summary, we have hence established that any function which is at least three times continuously differentiable (and does not have a vanishing second derivative) cannot be approximated by finite-depth networks with connectivity scaling polylogarithmically in the inverse of the approximation error.
Our results in Section \ref{func-mult} establish that, in contrast, this ``is" possible with finite-width deep networks for various interesting types of smooth functions such as polynomials and sinusoidal functions. Further results on the limitations of finite-depth networks akin to Theorem \ref{log-not-enough-NN2} were reported in \cite{PetersenVoigtlaender}.

\section*{Acknowledgments}
The authors are indebted to R.~G\"ul and W.~Ou for their careful proofreading of the paper, to E.~Riegler and the reviewers for their constructive and insightful comments, and to the handling editor, P.~Narayan, for his helpful comments and his patience.

\newpage
\begin{appendices}

\section{Auxiliary neural network constructions}

The following three results are concerned with the realization of affine transformations of arbitrary weights by neural networks
with weights upper-bounded by $1$.

\begin{lemma}\label{lem:scalar_mult}
 Let $d\in\N$ and $a\in\R$. There exists a network $\Phi_a\in\cN_{d,d}$ satisfying $\Phi_a(x)=ax$, with $\cL(\Phi_a)\leq\lfloor \log(|a|)\rfloor+4$, 
 $\cW(\Phi_a)\leq3d$, $\cB(\Phi_a)\leq1$. 
\end{lemma}

\begin{proof}
    First note that for $|a|\leq 1$ the claim holds trivially, which can be seen by taking $\Phi_a$ to be the affine transformation $x\mapsto ax$ and interpreting it according to Definition~\ref{def:NN} as a depth-$1$ neural network.
    Next, we consider the case $|a|>1$ for $d=1$, set $K:=\lfloor \log(a)\rfloor$, $\alpha:=a2^{-(K+1)}$, and define $A_1:=(1,-1)^T\in\R^{2\times 1}$,
    \begin{align*}
        A_2:=\begin{pmatrix}
          1 & 0 \\ 1 & 1 \\ 0 & 1
        \end{pmatrix}\in\R^{3\times 2},\quad
        A_k:=\begin{pmatrix}
        1 & 1 & -1\\
        1 & 1 & 1\\
        -1 & 1 & 1
        \end{pmatrix}\in\R^{3\times 3}, \quad k\in\{3,\dots,K+3\},
    \end{align*}
     and $A_{K+4}:=(\alpha,0,-\alpha)$. Note that $(\rho \circ A_2 \circ \rho \circ A_1)(x)=(\rho(x),\rho(x)+\rho(-x),\rho(-x))$ and 
     $\rho(A_k(x,x+y,y)^T)=2(x,x+y,y)$, for $k\in\{3,\dots,K+3\}$. The network $\Psi_a:=A_{K+4}\circ\rho\circ\dots\circ\rho\circ A_1$ hence satisfies
     $\Psi_a(x)=ax$, 
     $\cL(\Psi_a)=\lfloor \log(a)\rfloor+4$, 
     $\cW(\Psi_a)=3$, and 
     $\cB(\Phi_a) \le 1$.
     Applying Lemma~\ref{network_parallelization} to get a parallelization of $d$ copies of $\Psi_a$ completes the proof.
    \end{proof}
    
\begin{corollary}\label{cor:matrix_mult}
    Let $d,d'\in\N$, $a\in\R_+$, $A\in[-a,a]^{d'\times d}$, and $b\in[-a,a]^{d'}$. There exists a network $\Phi_{A,b}\in\cN_{d,d'}$ satisfying $\Phi_{A,b}(x)=Ax+b$, with
    $\cL(\Phi_{A,b})\leq\lfloor \log(|a|)\rfloor+5$,
    $\cW(\Phi_{A,b})\leq\max\{d,3d'\}$, $\cB(\Phi_{A,b})\leq1$.
\end{corollary}    
\begin{proof}
  Let $\Phi_a \in \cN_{d',d'}$ be the multiplication network from Lemma~\ref{lem:scalar_mult}, consider $W(x):=a^{-1}(Ax+b)$ as a $1$-layer network, and take $\Phi_{A,b}:=\Phi_a\circ W$ according to Lemma~\ref{network_conc}.
\end{proof}

\begin{proposition}\label{WDtradeoff}
Let $d,d'\in\N$ and $\Phi\in\cN_{d,d'}$. There exists a network $\Psi\in\cN_{d,d'}$ satisfying $\Psi(x)=\Phi(x)$, for all $x\in \R^d$, and with $\cL(\Psi)\leq(\lceil \log(\cB(\Phi))\rceil+5)\cL(\Phi)$, $\cW(\Psi)\leq\max\{3d',\cW(\Phi)\}$, $\cB(\Psi)\leq1$.
\end{proposition}

\begin{proof}
    We write $\Phi=W_{\cL(\Phi)}\,\circ\,\rho\,\circ\,\dots\,\circ\,\rho\,\circ W_1$ and set $\widetilde{W}_{\ell}:=(\cB(\Phi))^{-1}W_\ell$, for $\ell\in\{1,\dots,\cL(\Phi)\}$, and $a:=\cB(\Phi)^{\cL(\Phi)}$. Let $\Phi_{a}\in\cN_{d',d'}$ be the multiplication network from Lemma~\ref{lem:scalar_mult} and define 
    \begin{align*}
        \widetilde{\Phi}:=\widetilde{W}_{\cL(\Phi)}\circ\rho\circ\dots\circ\rho\circ \widetilde{W}_1,
    \end{align*}
    and $\Psi:=\Phi_{a}\circ\widetilde{\Phi}$ according to Lemma~\ref{network_conc}. Note that $\widetilde{\Phi}$ has weights upper-bounded by $1$ and is of the same depth and width as $\Phi$.
    As $\rho$ is positively homogeneous, i.e., $\rho(\lambda x)=\lambda\rho(x)$, for all $\lambda\geq 0$, $x\in\R$, we have $\Psi(x)=\Phi(x)$, for all $x\in\R^d$.
    Application of Lemma~\ref{network_conc} and Lemma~\ref{lem:scalar_mult} completes the proof.
\end{proof}

Next we record a technical Lemma on how to realize a sum of networks with the same input by a network whose width is independent of the number of constituent networks.

\begin{lemma}\label{lem:finite_width_linearcombination}
 Let $d,d'\in\N$, $N\in\N$, and $\Phi_i\in\cN_{d,d'}$, $i\in\{1,\dots,N\}$. There exists a network $\Phi\in\cN_{d,d'}$ satisfying
 \begin{align*}
     \Phi(x)=\sum_{i=1}^N\Phi_i(x), \quad \mbox{for all}\, x\in\R^d,
 \end{align*}
 with $\cL(\Phi)=\sum_{i=1}^N\cL(\Phi_i)$, $\cW(\Phi)\leq 2d+2d'+\max\{2d,\max_i\{\cW(\Phi_i)\}\}$, $\cB(\Phi)=\max\{1,\max_{i}\cB(\Phi_i)\}$.
\end{lemma}

\begin{proof}
    We set $L_i=\cL(\Phi_i)$ and write the networks $\Phi_i$ as
    \begin{align*}
        \Phi_i = W^i_{L_i}\circ\rho \circ W^i_{L_i-1} \circ \rho \circ \dots \circ \rho \circ W^i_{1},
    \end{align*}
    with $W^i_\ell(x)=A^i_\ell x+b^i_\ell$, where $A^i_\ell\in\R^{N^i_\ell\times N^i_{\ell-1}}$ and $b^i_\ell\in\R^{N^i_\ell}$. 
    Next, using Lemma~\ref{network_extension}, we turn the identity matrices $\mathbb{I}_d$ and $\mathbb{I}_{d'}$ into
    networks $\mathbb{I}^i_{d}$ and $\mathbb{I}^i_{d'}$, respectively, of depth $L_i$ and then parallelize these networks, according to Lemma~\ref{network_parallelization}, to get $\Psi_i:=(\mathbb{I}^i_{d},\,\mathbb{I}^i_{d'},\,\Phi_i)$. Let $V^i_1(x)=E^i_1 x + f^i_1$ and $V^i_{L_i}(x)=E^i_{L_i} x + f^i_{L_i}$ denote the first and last, respectively, affine transformation of the network $\Psi_i$. By construction we have 
    \begin{align*}
        E^i_1=\begin{pmatrix}
        \mathbb{I}_d & 0 & 0\\
        -\mathbb{I}_d & 0 & 0\\
        0 &\mathbb{I}_{d'} & 0\\
        0 &-\mathbb{I}_{d'} & 0\\
        0 & 0 & A^i_1
        \end{pmatrix}\in\R^{(2d+2d'+N^i_1)\times(2d+d')}, \quad f^i_{1}=\begin{pmatrix}
        0\\0\\0\\0\\b^i_1
        \end{pmatrix}\in\R^{2d+2d'+N^i_1}
    \end{align*}
    and
    \begin{align*}
        E^i_{L_i}=\begin{pmatrix}
        \mathbb{I}_d & -\mathbb{I}_d &  0 & 0 & 0\\
        0 & 0 & \mathbb{I}_{d'} & -\mathbb{I}_{d'} & 0\\
        0 & 0 & 0 & 0 & A^i_{L_i}
        \end{pmatrix}\in\R^{(d+2d')\times(2d+2d'+N^i_{L_i-1})}, \quad f^i_{L_i}=\begin{pmatrix}
        0\\0\\b^i_{L_i}
        \end{pmatrix}\in\R^{d+2d'}.
    \end{align*}
    Next, we define the matrices 
    \begin{align*}
        A_{\text{in}}&:=\begin{pmatrix}\mathbb{I}_d \\ 0 \\ \mathbb{I}_d \end{pmatrix}\in\R^{(2d+d')\times d}, 
        \quad 
        A:=\begin{pmatrix}
        \mathbb{I}_d & 0 & 0\\
        0 & \mathbb{I}_{d'} & \mathbb{I}_{d'}\\
        \mathbb{I}_d & 0 & 0
        \end{pmatrix}\in\R^{(2d+d')\times(d+2d')},\\
        A_{\text{out}}&:=\begin{pmatrix}
        0 & \mathbb{I}_{d'} & \mathbb{I}_{d'}
        \end{pmatrix}\in\R^{d'\times(d+2d')},
    \end{align*}
    and note that $A_{\text{in}}x=(x,0,x)$, $A(x,y,z)^T=(x,y+z,x)^T$, and $A_{\text{out}}(x,y,z)^T=y+z$, for $x\in\R^d,y,z\in\R^{d'}$.
    We construct
    \begin{itemize}
        \item the network $\widetilde{\Psi}_1$ by taking $\Psi_1$ and replacing $E^1_1$ with $E^1_1 A_{\text{in}}$, $E^1_{L_1}$ with $A E^1_{L_1}$, and $f^1_{L_1}$ with $A f^1_{L_1}$,
        \item the network $\widetilde{\Psi}_N$ by taking $\Psi_N$ and replacing $E^N_{L_N}$ with $A_{\text{out}}E^N_{L_N}$ and $f^N_{L_N}$ with $A_{\text{out}}f^N_{L_N}$,
        \item the networks $\widetilde{\Psi}_i$, $i\in\{2,\dots, N-1\}$ by taking $\Psi_i$ and replacing $E^i_{L_i}$ with $A E^i_{L_i}$ and $f^i_{L_i}$ with $A f^i_{L_i}$. 
    \end{itemize}
    We can now verify that 
    \begin{align*}
        \Phi=\widetilde{\Psi}_N\circ\widetilde{\Psi}_{N-1}\circ\dots\circ\widetilde{\Psi}_1,
    \end{align*}
     when the compositions are taken in the sense of Lemma~\ref{network_conc}. Due to Lemmas~\ref{network_extension} and~\ref{network_parallelization}, we have $\cL(\Psi_i)=\cL(\Phi_i)$, $\cW(\Psi_i)=2d+2d'+\cW(\Phi_i)$, and $\cB(\Psi_i)=\max\{1,\cB(\Phi_i)\}$. The proof is finalized by noting that, owing to the structure of the involved matrices, the depth and the weight magnitude remain unchanged by turning $\Psi_i$ into $\widetilde{\Psi}_i$, whereas the width can not increase, but may decrease owing to the replacement of $E^{1}_1$ by $E^{1}_{1}A_{\text{in}}$.
\end{proof}

The following lemma shows how to patch together local approximations using multiplication networks and a partition of unity consisting of hat functions. We note that this argument can be extended to higher dimensions using tensor products (which can be realized efficiently through multiplication networks) of the one-dimensional hat function.

\begin{lemma}\label{lem:stitching}
Let $\eps \in (0,1/2)$, $n\in\N$, $a_0<a_1<\dots<a_n\in\R$, $f\in L^\infty([a_0,a_n])$, and 
\begin{align*}
A:=\big\lceil\max\{|a_0|,|a_n|,2\max_{i\in\{2,\dots,n-1\}}\tfrac{1}{|a_i-a_{i-1}|}\}\big\rceil,\quad B:=\max\{1,\|f\|_{L^\infty([a_0,a_n])}\}.\end{align*}
Assume that for every $i\in\{1,\dots,n-1\}$, there exists a network $\Phi_i\in\cN_{1,1}$ with $\|f-\Phi_i\|_{L^\infty([a_{i-1},a_{i+1}])}\leq\eps/3$.
Then, there is a network $\Phi\in\cN_{1,1}$ satisfying
\begin{align*}
    \|f-\Phi\|_{L^\infty([a_0,a_n])}\leq\eps,
\end{align*}
with $\cL(\Phi)\leq \sum_{i=1}^{n-1}\cL(\Phi_i)+ Cn(\log(\eps^{-1})+\log(B)+\log(A))$, {$\displaystyle\cW(\Phi)\leq 7+\max\{2,\max_{i\in\{1,\dots,n-1\}}\cW(\Phi_i)$}\}, $\cB(\Phi) = \max\{1,\max_{i}\cB(\Phi_i)\}$, and with $C>0$ an absolute constant, i.e., independent of $\eps,n,f,a_0,\dots,a_n$.
\end{lemma}

\begin{proof}
    We first define the neural networks $(\Psi_i)_{i=1}^{n-1}\in\cN_{1,1}$ forming a partition of unity according to
    \begin{align*}
        \Psi_1(x)&:=1-\tfrac{1}{a_2-a_1}\,\rho(x-a_1)+\tfrac{1}{a_2-a_1}\,\rho(x-a_2),\\
        \Psi_i(x)&:=\tfrac{1}{a_i-a_{i-1}}\,\rho(x-a_{i-1})-(\tfrac{1}{a_i-a_{i-1}}+\tfrac{1}{a_{i+1}-a_i})\,\rho(x-a_i)+\tfrac{1}{a_{i+1}-a_i}\,\rho(x-a_{i+1}),\quad i\in\{2,\dots,n-2\},\\
        \Psi_{n-1}(x)&:=\tfrac{1}{a_{n-1}-a_{n-2}}\,\rho(x-a_{n-2})-\tfrac{1}{a_{n-1}-a_{n-2}}\,\rho(x-a_{n-1}).
    \end{align*}
    Note that $\supp(\Psi_1)=(\infty,a_2)$, $\supp(\Psi_{n-1})=[a_{n-2},\infty)$, and $\supp(\Psi_i)=[a_{i-1},a_{i+1}]$.
    Proposition~\ref{WDtradeoff} now ensures that, for all $i\in\{1,\dots,n-1\}$, $\Psi_i$ can be realized as a network with $\cL(\Psi_i)\leq 2(\lceil \log(A)\rceil +5)$, $\cW(\Psi_i)\leq 3$, and $\cB(\Psi_i)\leq 1$. Next, let $\Phi_{B+1/6,\eps/3}\in\cN_{2,1}$ be the multiplication network according to Proposition~\ref{relu_mult}
    and define the networks
     \begin{align*}
         \widetilde{\Phi}_i(x):=\Phi_{B+1/6,\eps/3}(\Phi_i(x),\Psi_i(x))
     \end{align*}
     according to Lemma~\ref{network_parallelization} and Lemma~\ref{network_conc}, along with their sum
    \begin{align*}
        \Phi(x):=\sum_{i=1}^{n-1}\widetilde{\Phi}_i(x)
    \end{align*}
    according to Lemma~\ref{lem:finite_width_linearcombination}.
    Proposition~\ref{relu_mult} ensures, for all $i\in\{1,\dots,n-1\}$, $x\in[a_{i-1},a_{i+1}]$, that
    \begin{align*}
        |f(x)\Psi_i(x)-\widetilde{\Phi}_i(x)|
        &\leq |f(x)\Psi_i(x)-\Phi_i(x)\Psi_i(x)|
        +|\Phi_i(x)\Psi_i(x)-\Phi_{B+1/6,\eps/3}(\Phi_i(x),\Psi_i(x))|\\
        &\leq (\Psi_i(x)+1)\tfrac{\eps}{3}
    \end{align*}
    and $\supp(\widetilde{\Phi}_i)=[a_{i-1},a_{i+1}]$. In particular, for every $x\in[a_0,a_n]$, the set
    \begin{align*}
    I(x):=\{i\in\{1,\dots,n-1\}\colon\widetilde{\Phi}_i(x)\neq 0\}    
    \end{align*}
    of active indices contains at most two elements. Moreover, we have $\sum_{i\in I(x)} \Psi_i(x)=1$ by construction, which implies that, for all $x\in\R$, 
   \begin{align*}
        |f(x)-\Phi(x)|=\left|\sum_{i\in I(x)}\Psi_i(x)f(x)-\sum_{i\in I(x)}\tilde{\Phi}_{i}(x)\right|\leq \sum_{i\in I(x)}(\Psi_{i}(x)+1)\tfrac{\eps}{3}\leq\eps.
   \end{align*}
   Due to Lemma~\ref{network_conc}, Lemma~\ref{network_parallelization}, Proposition~\ref{relu_mult}, and Lemma~\ref{lem:finite_width_linearcombination}, we can conclude that $\Phi$, indeed, satisfies the claimed properties.
\end{proof}

Next, we present an extension of Lemma~\ref{Sfunctions} to arbitrary (finite) intervals.

\begin{lemma}\label{lem:Sfunctions_general}
 For $a,b\in\R$ with $a<b$, let
 \begin{align*}
   \Ss_{[a,b]}:=\left\{f\in C^{\infty}([a,b],\R)\colon \|f^{(n)}(x)\|_{ L^{\infty}([a,b])} \leq n!,\, \text{\emph{ for all }} n\in\N_0\right\}.
 \end{align*}
  There exists a constant $C>0$ such that for all $a,b\in\R$ with $a<b$, $f\in\Ss_{[a,b]}$, and $\eps\in(0,1/2)$, there is a network $\Psi_{f,\eps}\in\cN_{1,1}$ satisfying 
  \begin{align*}
    \|\Psi_{f,\eps}-f\|_{ L^{\infty}([a,b])}\leq\eps,
  \end{align*}
  with $\mathcal{L}(\Psi_{f,\eps}) \le C\max\{2,(b-a)\}((\log(\eps^{-1}))^2+\log(\lceil\max\{|a|,|b|\}\rceil)+\log(\lceil\tfrac{1}{b-a}\rceil))$, $\mathcal{W}(\Psi_{f,\eps}) \le 16$,
  $\mathcal{B}(\Psi_{f,\eps})\leq 1$.
\end{lemma}

\begin{proof}
  We first recall that the case $[a,b]=[-1,1]$ has already been dealt with in Lemma~\ref{Sfunctions}. 
  Here, we will first prove the statement for the interval $[-D,D]$ with $D\in(0,1)$ and then use this result to establish the general case through a patching argument according to Lemma~\ref{lem:stitching}. We start by noting that for $g\in\Ss_{[-D,D]}$, the function $f_g\colon [-1,1]\to\R, x\mapsto g(Dx)$ is in $\Ss_{[-1,1]}$ due to $D<1$.
  Hence, by Lemma~\ref{Sfunctions}, there exists a constant $C>0$ such that for all 
  $g\in\Ss_{[-D,D]}$ and $\eps \in (0,1/2)$, there is a network $\widetilde{\Psi}_{g,\eps}\in\cN_{1,1}$ satisfying $\|\widetilde{\Psi}_{g,\eps}-f_g\|_{L^{\infty}([-1,1])}\leq\eps$, with
  $\mathcal{L}(\widetilde{\Psi}_{g,\eps}) \leq C(\log(\eps^{-1}))^2$, $\mathcal{W}(\widetilde{\Psi}_{g,\eps}) \le 9$, $\mathcal{B}(\widetilde{\Psi}_{g,\eps})\leq 1$.
  The claim is then established by taking the network approximating $g$ to be $\Psi_{g,\eps} := \widetilde{\Psi}_{g,\eps}\circ \Phi_{D^{-1}}$, where $\Phi_{D^{-1}}$ is the scalar multiplication network from Lemma~\ref{lem:scalar_mult}, and noting that 
  \begin{align*}
    \|\Psi_{g,\eps}(x)-g(x)\|_{L^{\infty}([-D,D])}&=\sup_{x\in[-D,D]}|\widetilde{\Psi}_{g,\eps}(\tfrac{x}{D})-f_g(\tfrac{x}{D})|\\
    &=\sup_{x\in[-1,1]}|\widetilde{\Psi}_{g,\eps}(x)-f_g(x)|\leq\eps. 
  \end{align*} 
  Due to Lemma~\ref{network_conc}, we have
  $\mathcal{L}(\Psi_{g,\eps}) \leq C((\log(\eps^{-1}))^2+\log(\lceil\tfrac{1}{D}\rceil))$, $\mathcal{W}(\Psi_{g,\eps}) \le 9$, and $\mathcal{B}(\Psi_{g,\eps})\leq 1$.
  We are now ready to proceed to the proof of the statement for general intervals $[a,b]$.
  This will be accomplished by approximating $f$ on intervals of length no more than $2$ and stitching the resulting approximations together according to Lemma~\ref{lem:stitching}.
  We start with the case $b-a\leq2$ and note that here we can simply shift the function by $(a+b)/2$ to center its domain around the origin and then use the result above for approximation on $[-D,D]$ with $D\in(0,1)$ or Lemma~\ref{Sfunctions} if $b-a=2$, both in combination with Corollary~\ref{cor:matrix_mult} to realize the shift through a neural network with weights bounded by $1$.
  Using Lemma~\ref{network_conc} to implement the composition of the network realizing this shift with that realizing $g$, we can conclude the existence of a constant $C'>0$ such that, for all $[a,b]\subseteq\R$ with $b-a\leq 2$, $g\in\Ss_{[a,b]}$, $\eps\in(0,1/2)$, there is a network satisfying $\|g-\Psi_{g,\eps}\|_{L^\infty([a,b])}\leq\eps$ with  $\mathcal{L}(\Psi_{g,\eps}) \leq C'((\log(\eps^{-1}))^2+\log(\lceil\tfrac{1}{b-a}\rceil))$, $\mathcal{W}(\Psi_{g,\eps}) \le 9$, and $\mathcal{B}(\Psi_{g,\eps})\leq 1$.
  Finally, for $b-a>2$, we partition the interval $[a,b]$ and apply Lemma~\ref{lem:stitching} as follows.
  We set $n:=\lceil b-a \rceil$ and define
  \begin{align*}
      a_i&:=a+i\tfrac{b-a}{n}, \quad\, i\in\{0,\dots,n\}.
   \end{align*}
  Next, for $i\in\{1,\dots,n-1\}$, let $g_i\colon[a_{i-1},a_{i+1}]\to\R$ be the restriction of $g$ to the interval $[a_{i-1},a_{i+1}]$, and note that $a_{i+1}-a_{i-1}=\tfrac{2(b-a)}{n}\in(\tfrac{4}{3},2]$.
  Furthermore, for $i\in\{1,\dots,n-1\}$, let $\Psi_{g_i,\eps/3}$ be the network approximating $g_i$ with error $\eps/3$ as constructed above. 
  Then, for every $i\in\{1,\dots,n-1\}$, it holds that $\|g-\Psi_{g_i,\eps/3}\|_{L^\infty([a_{i-1},a_{i+1}])}\leq\tfrac{\eps}{3}$ and application of Lemma~\ref{lem:stitching} yields the desired result.
\end{proof}

We finally record, for technical purposes, slight variations of Lemmas~\ref{network_parallelization} and~\ref{network_linearcombination} to account for parallelizations and linear combinations, respectively, of neural networks with shared input.

\begin{lemma}
\label{lem:shared_input_para}
Let $n,d,L\in\N$ and, for $i\in\{1,2,\dots,n\}$, let $d'_i\in\N$ and $\Phi_i\in\cN_{d,d'_i}$ with $\cL(\Phi_i)=L$.
Then, there exists a network $\Psi \in \cN_{d,\sum_{i=1}^n d'_i}$ with $\cL(\Psi) = L$, $\cM(\Psi) = \sum_{i=1}^n\cM(\Phi_i)$, $\mathcal{W}(\Psi) \le \sum_{i=1}^n \mathcal{W}(\Phi_i)$, $\cB(\Psi)=\max_i\cB(\Phi_i)$, and satisfying
\begin{align*}
    \Psi(x)&=(\Phi_1(x),\Phi_2(x),\dots,\Phi_n(x))\in\R^{\sum_{i=1}^n d'_i},
\end{align*}
for $x\in\R^d$.
\end{lemma}
\begin{proof}
The claim is established by following the construction in the proof of Lemma~\ref{network_parallelization}, but with the matrix $A_1=\diag(A^1_1,A^2_1,\dots,A^n_1)$ replaced by 
\begin{align*}
   A_1=\begin{pmatrix}A^1_1 \\ \vdots \\ A^n_1\end{pmatrix}\in \R^{(\sum_{i=1}^n N_1^i)\times d},
\end{align*}
where $N_1^i$ is the dimension of the first layer of $\Phi_i$.
\end{proof}

\begin{lemma}
\label{lem:shared_input_lc}
Let $n,d,d',L\in\N$ and, for $i\in\{1,2,\dots,n\}$, let $a_i\in\R$ and $\Phi_i\in\cN_{d,d'}$ with $\cL(\Phi_i)=L$.
Then, there exists a network $\Psi \in \cN_{d,d'}$ with $\cL(\Psi) = L$, $\cM(\Psi) \le \sum_{i=1}^n\cM(\Phi_i)$, $\mathcal{W}(\Psi) \leq \sum_{i=1}^n \mathcal{W}(\Phi_i)$, $\cB(\Psi)=\max_i\{|a_i|\cB(\Phi_i)\}$, and satisfying
\begin{align*}
    \Psi(x)&=\sum_{i=1}^n a_i \Phi_i(x)\in\R^{d'},
\end{align*}
for $x\in\R^d$.
\end{lemma}
\begin{proof}
The proof follows directly from that of Lemma~\ref{lem:shared_input_para} with the same modifications as those needed in the proof of Lemma~\ref{network_linearcombination} relative to that of Lemma~\ref{network_parallelization}.
\end{proof}

\section{Tail compactness for Besov spaces}\label{Besov_tail}
We consider the Besov space $B^m_{p,q}([0,1])$ \cite{mallat_wavelet_tour} given by the set of functions $f\in L^2([0,1])$ satisfying
\begin{align}\label{eq:Besov_wave_rep}
    \|f\|_{m,p,q}:=\|(2^{n(m+\frac{1}{2}-\frac{1}{p})}\|(\langle f,\psi_{n,k}\rangle)_{k=0}^{2^n-1}\|_{\ell^p})_{n\in\N_0}\|_{\ell^q}<\infty,
\end{align}
with $\mathcal{D}=\{\psi_{n,k}\colon n\in\N_0, k=0,\dots,2^n - 1\}$ an orthonormal wavelet basis\footnote{The space does not depend on the particular choice of mother wavelet $\psi$ as long as $\psi$ has at least $r$ vanishing moments and is in $C^r([0,1])$ for some $r>m$. For further details we refer to Section~9.2.3 in \cite{mallat_wavelet_tour}.} for
$L^2([0,1])$ and $\ell^p$ denoting the usual sequence norm
\begin{align*}
    \|(a_i)_{i\in I}\|_{\ell^p}=
    \begin{cases}
     \left(\sum_{i\in I}|a_i|^p\right)^{\frac{1}{p}}, & 1\leq p < \infty\\
     \sup_{i\in I}|a_i|, & p = \infty
     \end{cases}.
\end{align*}
The unit ball in $B^m_{p,q}([0,1])$ is 
\begin{align}\label{eq:Besov_unitball}
\mathcal{U}(B^m_{p,q}([0,1])) = \{f\in L^2([0,1])\colon \|f\|_{m,p,q}\leq  1\}.
\end{align}
For simplicity of notation, we set $a_{n,k}(f):=\langle f,\psi_{n,k}\rangle$ and $A_n(f):=(a_{n,k}(f))_{k=0}^{2^n-1}\in\R^{2^n}$, for $n\in\N_0$.
We now want to verify that for $q \in [1,2]$ tail compactness holds for the pair ($\mathcal{U}(B^m_{p,q}([0,1])),\mathcal{D})$ under the ordering $\mathcal{D}=(\mathcal{D}_0,\mathcal{D}_1,\dots)$, where $\mathcal{D}_n:=\{\psi_{n,k}\colon  k=0,\dots,2^n-1\}$. To this end, we first note that owing to $\sum_{n=0}^N|\mathcal{D}_n|=2^{N+1}-1$, we have tail compactness according to \eqref{eq:tailcompactness} if there exist $C,\beta>0$ such that for all $f\in\mathcal{U}(B^m_{p,q}([0,1]))$, $N\in\N$,
\begin{align}\label{eq:Besov_tail_comp}
\left\|f-\sum_{n=0}^N\sum_{k=0}^{2^n-1}a_{n,k}(f)\psi_{n,k}\right\|_{L^2([0,1])}\leq C (2^{N+1})^{-\beta}.
\end{align}
To see that \eqref{eq:Besov_wave_rep} implies \eqref{eq:Besov_tail_comp}, we note that by orthonormality of $\mathcal{D}$,
\begin{align*}
    \left\|f-\sum_{n=0}^N\sum_{k=0}^{2^n-1}a_{n,k}(f)\psi_{n,k}\right\|_{L^2([0,1])}
    &= \left\|\sum_{n=N+1}^\infty\sum_{k=0}^{2^n-1}a_{n,k}(f)\psi_{n,k}\right\|_{L^2([0,1])}
    =\left(\sum_{n=N+1}^\infty\sum_{k=0}^{2^n-1}|a_{n,k}(f)|^2\right)^{\frac{1}{2}}
    \\
    &=\|(\|A_n(f)\|_{\ell^2})_{n=N+1}^\infty\|_{\ell^2}.
\end{align*}
As the $A_n(f)$ are finite sequences of length $|\mathcal{D}_n|=2^n$, it follows, by application of Hölder's inequality, that $\|A_n(f)\|_{\ell^2} \leq 2^{n(\frac{1}{2}-\frac{1}{p})}\|A_n(f)\|_{\ell^p}$. Together with $\|\cdot\|_{\ell^2}\leq\|\cdot\|_{\ell^q}$, for $q\leq 2$, \eqref{eq:Besov_wave_rep} then ensures, for all $f\in\mathcal{U}(B^m_{p,q}([0,1]))$ and $q \in [1,2]$, that
\begin{align*}
    \|(\|A_n(f)\|_{\ell^2})_{n=N+1}^\infty\|_{\ell^2}
    &\leq\|(2^{n(\frac{1}{2}-\frac{1}{p})}\|A_n(f)\|_{\ell^p})_{n=N+1}^\infty\|_{\ell^q}
    \leq 2^{-(N+1)m}\|(2^{n(m+\frac{1}{2}-\frac{1}{p})}\|A_n(f)\|_{\ell^p})_{n=N+1}^\infty\|_{\ell^q}\\
    &\leq 2^{-(N+1)m}\|f\|_{m,p,q} \leq (2^{N+1})^{-m},
\end{align*}
which establishes (\ref{eq:Besov_tail_comp}) with $C=1$ and $\beta=m$.

\section{Tail compactness for modulation spaces}\label{Modulation_tail}

We consider tail compactness for unit balls in (polynomially) weighted modulation spaces, which, for $p,q\in[1,\infty)$, are defined as follows
\begin{align*}
    M^s_{p,q}(\R):=\{f\colon\|f\|_{M^s_{p,q}(\R)}<\infty\},
\end{align*}
with
\begin{align*}
    \|f\|_{M^s_{p,q}(\R)}:=\left(\int_\R\left(\int_\R|V_w f(x,\xi)|^p (1+|x|+|\xi|)^{sp}\mathrm{d}x\right)^{\frac{q}{p}}\mathrm{d}\xi\right)^{\frac{1}{q}},
\end{align*}
where 
\begin{align*}
    V_w f(x,\xi):=\int_\R f(t)\,\overline{w(t-x)}e^{-2\pi i t\xi}\mathrm{d}t, \quad x,\xi\,\in\,\R,
\end{align*}
is the short-time Fourier transform of $f$ with respect to the window function\footnote{The resulting modulation space does not depend on the specific choice of window function $w$ as long as $w$ is in the Schwartz space $\mathcal{S}(\R)=\{f\in C^\infty(\R)\colon \sup_{x\in\R}|x^\alpha f^{(\beta)}(x)|<\infty, \, \mathrm{for\,\, all}\ \alpha,\beta\in\N_0\}$, where $f^{(n)}$ stands for the $n$-th derivative of $f$.} $w\in\mathcal{S}(\R)$.  

Next, let $g\in L^2(\R)$ with $\|g\|_{L^2(\R)}=1$ and $g(x)=\overline{g(-x)}$ such that the Gabor dictionary $\mathcal{G}(g,\frac{1}{2},1,\R)$ is a tight frame \cite{Morgenshtern-Boelcskei-2012} for $L^2(\R)$. Then, the Wilson dictionary $\mathcal{D}=\{\psi_{k,n}\colon (k,n)\in\Z\times\N_0\}$ with
\begin{align*}
    \psi_{k,0}&=T_kg , &k\in\Z,\\
    \psi_{k,n}&=\tfrac{1}{\sqrt{2}}T_{\frac{k}{2}}(M_n+(-1)^{k+n}M_{-n})g, &(k,n)\in\Z\times\N,
\end{align*}
is an orthonormal basis for $L^2(\R)$ (see \cite[Thm.~8.5.1]{grochenig2013foundations}).
We have, for every $f\in M^s_{p,q}(\R)$, the expansion \cite[Thm.~12.3.4]{grochenig2013foundations}
\begin{align*}
    f=\sum_{(k,n)\in\Z\times\N_0} c_{k,n}(f) \psi_{k,n}, \quad \text{where}\quad c_{k,n}(f)=\langle f,\psi_{k,n} \rangle, \quad c(f)\,\in\,\ell^s_{p,q}(\Z\times\N_0), 
\end{align*}
with $\ell^s_{p,q}(\Z\times\N_0)$ the space of sequences $c\in\R^{\Z\times\N_0}$ satisfying
\begin{align*}
    \|c\|_{\ell^s_{p,q}(\Z\times\N_0)}:=\left(\sum_{n\in\N_0}\left(\sum_{k\in\Z}|c_{k,n}|^p(1+|\tfrac{k}{2}|+|n|)^{sp}
    \right)^{\frac{q}{p}}\right)^{\frac{1}{q}}<\infty.
\end{align*}
Moreover, there exists \cite[Thm.~12.3.1]{grochenig2013foundations} a constant $D\geq 1$ such that, for all $f\in M^s_{p,q}(\R)$,
\begin{align*}
   \tfrac{1}{D} \|f\|_{M^s_{p,q}(\R)}\leq \|c(f)\|_{\ell^s_{p,q}(\Z\times\N_0)} \leq D \|f\|_{M^s_{p,q}(\R)}.
\end{align*}
In particular, we can characterize the unit ball of $M^s_{p,q}(\R)$ according to
\begin{align*}
\mathcal{U}(M^s_{p,q}(\R))=\{f\colon  \|c(f)\|_{\ell^s_{p,q}(\Z\times\N_0)} \le D\}.
\end{align*}
We now order the Wilson basis dictionary as follows.
Define $\mathcal{D}_0:=\{\psi_{0,0}\}$ and 
\begin{align*}
\mathcal{D}_\ell:=\{\psi_{k,n}\colon |k|,n\leq \ell\}\setminus\bigcup_{i=0}^{\ell-1}\mathcal{D}_{i}
\end{align*} for $\ell \ge 1$, and order the overall dictionary according to $\mathcal{D}=(\mathcal{D}_0,\mathcal{D}_1,\dots)$. Owing to 
$\sum_{\ell=0}^N|\mathcal{D}_\ell|=(2N+1)(N+1)$, we have tail compactness for the pair $(\mathcal{U}(M^s_{p,q}(\R)),\mathcal{D})$ 
if there exist $C,\beta>0$ such that, for all $f\in\mathcal{U}(M^s_{p,q}(\R))$, $N\in\N$,
\begin{align}\label{eq:Besov_tail_comp_Gabor}
\left\|f-\sum_{n=0}^N\sum_{k=-N}^N c_{k,n}(f)\psi_{k,n}\right\|_{L^2(\R)}\leq CN^{-\beta}.
\end{align}
We restrict our attention to $p,q\leq 2$ and use orthonormality of $\mathcal{D}$ and
the fact that $\|\cdot\|_{\ell^2}\leq\|\cdot\|_{\ell^p}$, for $p\leq 2$, to obtain, for all $f\in\mathcal{U}(M^s_{p,q}(\R))$,
\begin{align*}
  \left\|f-\sum_{n=0}^N\sum_{k=-N}^N c_{k,n}(f)\psi_{k,n}\right\|_{L^2(\R)} &=  
  \left\|\sum_{n>N}\sum_{|k|>N} c_{k,n}(f)\psi_{k,n}\right\|_{L^2(\R)} = 
  \left(\sum_{n>N}\sum_{|k|>N} |c_{k,n}(f)|^2\right)^{\frac{1}{2}}\\
  &\leq  \left(\sum_{n>N}\left(\sum_{|k|>N} |c_{k,n}(f)|^p\right)^{\frac{q}{p}}\right)^{\frac{1}{q}}\\
  &\leq (1+\tfrac{3}{2}N)^{-s}\left(\sum_{n>N}\left(\sum_{|k|>N} |c_{k,n}(f)|^p(1+|\tfrac{k}{2}|+|n|)^{sp}\right)^{\frac{q}{p}}\right)^{\frac{1}{q}
  }\\
  & \leq (1+\tfrac{3}{2}N)^{-s}\|c(f)\|_{\ell^s_{p,q}(\Z\times\N_0)}\leq (3/2)^{-s}DN^{-s},
\end{align*}
which establishes tail compactness with $C=(3/2)^{-s}D$ and $\beta=s$.

\end{appendices}
\newpage 
\bibliographystyle{IEEEtran}
\bibliography{references}

\end{document}